\documentclass{article}

\PassOptionsToPackage{authoryear}{natbib}

\usepackage[preprint]{neurips_2026}

\usepackage[utf8]{inputenc} 
\usepackage[T1]{fontenc}    
\usepackage{hyperref}       
\usepackage{url}            
\usepackage{booktabs}       
\usepackage{amsfonts}       
\usepackage{nicefrac}       
\usepackage{microtype}      
\usepackage{xcolor}         

\usepackage[textsize=tiny]{todonotes} 
\usepackage{subcaption}

\usepackage{amsmath,amsthm,amssymb}
\usepackage{thmtools,thm-restate}  
\usepackage{xspace}  
\usepackage{wrapfig}
\usepackage{natbib}
\usepackage{multirow} 
\usepackage{xcolor}
\usepackage{paralist}
\usepackage{longtable}
\usepackage{ifthen}

\usepackage{tikz}
\usepackage{pgfplots}
\usetikzlibrary{calc}
\usetikzlibrary{intersections}
\usetikzlibrary{backgrounds}
\usepgfplotslibrary{statistics}
\usepgfplotslibrary{fillbetween}
\usetikzlibrary{matrix}
\usetikzlibrary{positioning}
\usetikzlibrary{fit}
\usepackage{pgfplotstable}  
\usepackage{minted}
\usepackage{bbm}
\usepackage{xstring}

\usepackage{eda}
\usepackage{eda-colors}

\definecolor{color(mmci)}{cmyk}{1,0.4,0,0}
\definecolor{color(cispa)}{rgb}{0.008,0.356,0.584}
\definecolor{color(helmholtzblue)}{RGB}{0, 90, 160}
\definecolor{color(helmholtzgreen)}{RGB}{140, 180, 35}

\definecolor{wong-blue}{RGB}{0, 114, 178}
\definecolor{wong-orange}{RGB}{230, 159, 0}
\definecolor{wong-green}{RGB}{0, 158, 115}
\definecolor{wong-reddishpurple}{RGB}{204, 121, 167}
\definecolor{wong-skyblue}{RGB}{86, 180, 233}
\definecolor{wong-vermillion}{RGB}{213, 94, 0}
\definecolor{wong-yellow}{RGB}{240, 228, 66}

\definecolor{tol-q-bright-blue}{HTML}{4477AA}
\definecolor{tol-q-bright-cyan}{HTML}{66CCEE}
\definecolor{tol-q-bright-green}{HTML}{228833}
\definecolor{tol-q-bright-yellow}{HTML}{CCBB44}
\definecolor{tol-q-bright-red}{HTML}{EE6677}
\definecolor{tol-q-bright-purple}{HTML}{AA3377}
\definecolor{tol-q-bright-grey}{HTML}{BBBBBB}

\definecolor{tol-q-high-contrast-white}{HTML}{FFFFFF}
\definecolor{tol-q-high-contrast-yellow}{HTML}{DDAA33}
\definecolor{tol-q-high-contrast-red}{HTML}{BB5566}
\definecolor{tol-q-high-contrast-blue}{HTML}{004488}
\definecolor{tol-q-high-contrast-black}{HTML}{000000}

\definecolor{tol-q-vibrant-blue}{HTML}{0077BB}
\definecolor{tol-q-vibrant-cyan}{HTML}{33BBEE}
\definecolor{tol-q-vibrant-teal}{HTML}{009988}
\definecolor{tol-q-vibrant-orange}{HTML}{EE7733}
\definecolor{tol-q-vibrant-red}{HTML}{CC3311}
\definecolor{tol-q-vibrant-magenta}{HTML}{EE3377}
\definecolor{tol-q-vibrant-grey}{HTML}{BBBBBB}

\definecolor{tol-q-muted-indigo}{HTML}{332288}
\definecolor{tol-q-muted-cyan}{HTML}{88CCEE}
\definecolor{tol-q-muted-teal}{HTML}{44AA99}
\definecolor{tol-q-muted-green}{HTML}{117733 }
\definecolor{tol-q-muted-olive}{HTML}{999933}
\definecolor{tol-q-muted-sand}{HTML}{DDCC77}
\definecolor{tol-q-muted-rose}{HTML}{CC6677}
\definecolor{tol-q-muted-wine}{HTML}{882255}
\definecolor{tol-q-muted-purple}{HTML}{AA4499}
\definecolor{tol-q-muted-pale-grey}{HTML}{DDDDDD}

\definecolor{tol-mark-pale-blue}{HTML}{BBCCEE}
\definecolor{tol-mark-pale-cyan}{HTML}{CCEEFF}
\definecolor{tol-mark-pale-green}{HTML}{CCDDAA}
\definecolor{tol-mark-pale-yellow}{HTML}{EEEEBB}
\definecolor{tol-mark-pale-red}{HTML}{FFCCCC}
\definecolor{tol-mark-pale-grey}{HTML}{DDDDDD}

\definecolor{tol-mark-dark-blue}{HTML}{222255}
\definecolor{tol-mark-dark-cyan}{HTML}{225555}
\definecolor{tol-mark-dark-green}{HTML}{225522}
\definecolor{tol-mark-dark-yellow}{HTML}{666633}
\definecolor{tol-mark-dark-red}{HTML}{663333}
\definecolor{tol-mark-pale-grey}{HTML}{555555}

\definecolor{tol-q-light-blue}{HTML}{77AADD}
\definecolor{tol-q-light-cyan}{HTML}{99DDFF}
\definecolor{tol-q-light-mint}{HTML}{44BB99}
\definecolor{tol-q-light-pear}{HTML}{BBCC33}
\definecolor{tol-q-light-olive}{HTML}{AAAA00}
\definecolor{tol-q-light-yellow}{HTML}{EEDD88}
\definecolor{tol-q-light-orange}{HTML}{EE8866}
\definecolor{tol-q-light-pink}{HTML}{FFAABB}
\definecolor{tol-q-light-pale-grey}{HTML}{DDDDDD}

\colorlet{color(ourmethod)}{tol-q-muted-indigo} 	
\colorlet{color(ourmethodsemi)}{tol-q-muted-indigo} 	
\colorlet{color(asso)}{tol-q-muted-cyan} 	
\colorlet{color(grecond)}{tol-q-muted-teal} 	
\colorlet{color(nmf)}{tol-q-muted-green} 	
\colorlet{color(nmfr)}{tol-q-muted-olive} 	
\colorlet{color(sofa)}{tol-q-muted-sand} 	
\colorlet{color(binaps)}{tol-q-muted-rose} 	
\colorlet{color(pimp)}{tol-q-muted-wine}

\definecolor{TSneColor(0.0)}{RGB}{136, 204, 238}
\definecolor{TSneColor(1.0)}{RGB}{117, 194, 214}
\definecolor{TSneColor(2.0)}{RGB}{98, 185, 190}
\definecolor{TSneColor(3.0)}{RGB}{79, 175, 166}
\definecolor{TSneColor(4.0)}{RGB}{62, 164, 140}
\definecolor{TSneColor(5.0)}{RGB}{47, 149, 112}
\definecolor{TSneColor(6.0)}{RGB}{33, 135, 83}
\definecolor{TSneColor(7.0)}{RGB}{19, 121, 54}
\definecolor{TSneColor(8.0)}{RGB}{26, 98, 72}
\definecolor{TSneColor(9.0)}{RGB}{35, 74, 96}
\definecolor{TSneColor(10.0)}{RGB}{45, 50, 120}
\definecolor{TSneColor(11.0)}{RGB}{67, 50, 134}
\definecolor{TSneColor(12.0)}{RGB}{115, 98, 130}
\definecolor{TSneColor(13.0)}{RGB}{163, 146, 125}
\definecolor{TSneColor(14.0)}{RGB}{210, 193, 120}
\definecolor{TSneColor(15.0)}{RGB}{206, 193, 104}
\definecolor{TSneColor(16.0)}{RGB}{187, 178, 85}
\definecolor{TSneColor(17.0)}{RGB}{168, 164, 66}
\definecolor{TSneColor(18.0)}{RGB}{156, 150, 55}
\definecolor{TSneColor(19.0)}{RGB}{171, 135, 74}
\definecolor{TSneColor(20.0)}{RGB}{185, 121, 94}
\definecolor{TSneColor(21.0)}{RGB}{199, 107, 113}
\definecolor{TSneColor(22.0)}{RGB}{191, 89, 113}
\definecolor{TSneColor(23.0)}{RGB}{172, 70, 103}
\definecolor{TSneColor(24.0)}{RGB}{153, 51, 94}
\definecolor{TSneColor(25.0)}{RGB}{137, 35, 87}
\definecolor{TSneColor(26.0)}{RGB}{147, 45, 106}
\definecolor{TSneColor(27.0)}{RGB}{156, 54, 125}
\definecolor{TSneColor(28.0)}{RGB}{166, 64, 144}
\definecolor{TSneColor(29.0)}{RGB}{178, 92, 164}
\definecolor{TSneColor(30.0)}{RGB}{192, 135, 183}
\definecolor{TSneColor(31.0)}{RGB}{207, 178, 202}
\definecolor{TSneColor(32.0)}{RGB}{221, 221, 221}

\definecolor{mygreen}{rgb}{0.0, 0.5, 0.0}
\definecolor{mygray}{rgb}{0.3, 0.3, 0.3}

\definecolor{variolila}{RGB}{115,100,137}
\definecolor{lilacgray}{RGB}{152,150,164}
\definecolor{red1}{RGB}{201,59,69}
\definecolor{cc1}{rgb}{0.59, 0.78, 0.64}
\definecolor{niceblue}{RGB}{3, 79, 132}
\definecolor{iceblue}{RGB}{128, 182, 207}
\definecolor{cc1}{rgb}{0.59, 0.78, 0.64}
\definecolor{niceblue}{RGB}{3, 79, 132}
\definecolor{iceblue}{RGB}{128, 182, 207}
\definecolor{ca3}{rgb}{0.83, 0.69, 0.22}
\definecolor{red1}{RGB}{201,59,69}
\definecolor{red2}{RGB}{169,22,48}
\definecolor{red3}{RGB}{249,129,115}
\definecolor{cl1}{rgb}{0.61, 0.77, 0.89} 
\definecolor{cl2}{rgb}{0.64, 0.68, 0.82} 
\definecolor{cl3}{rgb}{0.85, 0.65, 0.13} 
\definecolor{ca1}{rgb}{0.6, 0.81, 0.93} 
\definecolor{ca2}{RGB}{139, 204, 109} 
\definecolor{rosequartz}{RGB}{247,202,201}
\definecolor{serenity}{RGB}{145,168,209}
\definecolor{peachecho}{RGB}{247,120,107}
\definecolor{snorkelblue}{RGB}{3,79,132}
\definecolor{buttercup}{RGB}{250,224,60}
\definecolor{limpetshell}{RGB}{152,221,222}
\definecolor{lilacgray}{RGB}{152,150,164}
\definecolor{fiesta}{RGB}{221,65,50}
\definecolor{icedcoffee}{RGB}{177,143,106}
\definecolor{greenflash}{RGB}{121,199,83}
\definecolor{purple}{RGB}{153,0,153}
\definecolor{turquise}{RGB}{0,153,153}
\definecolor{poop}{RGB}{203,178,52}
\definecolor{blue1}{RGB}{215, 216, 233}
\definecolor{blue2}{RGB}{243, 243, 248}
\definecolor{varioblue}{RGB}{68, 114, 157}
\definecolor{variogreen}{RGB}{125, 171, 113}

\usepackage{prettyplots}  
\pgfplotsset{compat=1.18}  
\usepgfplotslibrary{groupplots}

\usetikzlibrary{external, matrix, positioning}
\usetikzlibrary{arrows.meta,positioning,fit,backgrounds}

\newcommand{\missing}{m}
\newcommand{\missingij}{\missing_{i}^{(r)}}
\newcommand{\missingmask}{M}

\newcommand{\featvar}{X}
\newcommand{\featset}{\mathbf{X}}
\newcommand{\feat}{x}
\newcommand{\featvec}{\mathbf{x}}
\newcommand{\featvecr}{\mathbf{x}^{(r)}}
\newcommand{\cell}[2]{\feat_{#1}^{(#2)}}
\newcommand{\cellij}{\feat_i^{(r)}}

\newcommand{\attention}{A}
\newcommand{\Hord}{H^{\text{ord}}}
\newcommand{\Hpred}{H^{\text{pred}}}

\newcommand{\varinc}{\delta}
\newcommand{\varincij}{\varinc_{i,j}}

\newcommand{\order}{\pi}

\newcommand{\hatorder}{\hat{\pi}}
\newcommand{\orderscore}{s}
\newcommand{\orderscores}{\textbf{s}}

\newcommand{\nfeat}{d}
\newcommand{\nsamples}{n}
\newcommand{\dataset}{\mathcal{D}}

\newcommand{\ourmethod}{\textsc{TabOrder}\xspace}
\newcommand{\avici}{\textsc{AVICI}\xspace}
\newcommand{\nogam}{\textsc{NoGAM}\xspace}
\newcommand{\dasmethod}{\textsc{DAS}\xspace} 
\newcommand{\randsort}{\textsc{RndSrt}\xspace} 
\newcommand{\rtwosort}{\textsc{R2Srt}\xspace}

\newcommand{\tabpfn}{\textsc{TabPFN}\xspace}
\newcommand{\tabimpute}{\textsc{TabImpute}\xspace}
\newcommand{\gain}{\textsc{GAIN}\xspace}
\newcommand{\miwae}{\textsc{MIWAE}\xspace}
\newcommand{\naim}{\textsc{NAIM}\xspace}
\newcommand{\mice}{\textsc{MICE}\xspace}

\newcommand{\meanimp}{\textsc{ColMean}\xspace}
\newcommand{\knnimp}{\textsc{KNN}\xspace}
\newcommand{\iceimp}{\textsc{ICE}\xspace}
\newcommand{\missforest}{\textsc{MissForest}\xspace}
\newcommand{\hyperimpute}{\textsc{HyperImpute}\xspace}
\newcommand{\xgboost}{\textsc{XGBoost}\xspace}
\newcommand{\softimpute}{\textsc{SoftImpute}\xspace}

\newcommand{\topic}{\textsc{Topic}\xspace}
\newcommand{\cam}{\textsc{Cam}\xspace}
\newcommand{\scoremethod}{\textsc{Score}\xspace}

\newcommand{\notears}{\textsc{Notears}\xspace}

\newcommand{\R}{\mathbb{R}}
\newcommand{\E}{\mathbb{E}}

\newcommand{\divtop}{d_\text{TOP}}
\newcommand{\pa}{\mathit{pa}\xspace}

\newcommand{\G}{G}

\colorlet{arch-green}{pr-color1a}   
\colorlet{arch-blue}{pr-color1b} 
\colorlet{arch-orange}{pr-color1c} 

\declaretheorem{theorem}

\DeclareMathOperator*{\argmax}{arg\,max}

\DeclareMathOperator{\Var}{Var}
\newcommand{\indep}{\perp \!\!\! \perp}

\let\emptyset\varnothing

\colorlet{color-topic}{pr-color1n}
\colorlet{color-cam}{pr-color1b}
\colorlet{color-nadd-gp}{pr-color1g}
\colorlet{color-ours}{pr-color1g}
\colorlet{color-avici}{pr-color1h}
\colorlet{color-nogam}{pr-color1o}
\colorlet{color-score}{pr-color1e}
\colorlet{color-das}{pr-color1f}
\colorlet{color-notears}{pr-color1k}
\colorlet{color-rtwosort}{tcss-slate4}

\colorlet{color-add-gp}{pr-color1g}
\colorlet{color-nadd-gp-alt}{pr-color1g}
\colorlet{color-add-sp}{pr-color1g}

\colorlet{nicecolor(violet)}{patriarch}
\colorlet{nicecolor(lviolet)}{mediumred-violet}
\colorlet{nicecolor(terracotta)}{tan}
\colorlet{nicecolor(yellow)}{goldenrod} 
\colorlet{nicecolor(green)}{yellow-green}
\colorlet{nicecolor(greenish)}{pistachio}
\colorlet{nicecolor(greenisher)}{mediumspringbud}
\colorlet{nicecolor(viola)}{wildblueyonder}
\colorlet{nicecolor(lav)}{languidlavender}
\colorlet{nicecolor(turquoise)}{ballblue}
\colorlet{nicecolor(paleblue)}{palecerulean}
\colorlet{nicecolor(starkorange)}{orangepeel}
\colorlet{nicecolor(orange)}{deepsaffron}
\colorlet{nicecolor(blue)}{glaucous}
\colorlet{nicecolor(dblue)}{lapislazuli}
\colorlet{nicecolor(cyan)}{darkcyan}
\colorlet{nicecolor(lblue)}{iceberg}
\colorlet{nicecolor(dvio)}{oldmauve}
\colorlet{color(ours)}{nicecolor(green)}  
\colorlet{color(ours2)}{dollarbill}  
\colorlet{color(cam)}{nicecolor(cyan)}
\colorlet{color(score)}{nicecolor(lblue)}
\colorlet{color(globe)}{nicecolor(violet)} 
\colorlet{color(resit)}{nicecolor(lav)}
\colorlet{color(icalingam)}{nicecolor(yellow)} 
\colorlet{color(dirlingam)}{nicecolor(yellow)} 
\colorlet{color(notears)}{nicecolor(paleblue)}
\colorlet{color(pc)}{nicecolor(lav)}
\colorlet{color(fci)}{nicecolor(greenish)}
\colorlet{color(ges)}{mangotango}
\colorlet{color(pc)}{nicecolor(orange)}
\colorlet{color(fci)}{nicecolor(greenish)}
\colorlet{color(ges)}{mangotango}
\colorlet{color(cdnod)}{nicecolor(cyan)}
\colorlet{color(pcmci)}{nicecolor(viola)}
\colorlet{color(varlingam)}{mangotango}
\colorlet{color(dynotears)}{nicecolor(greenish)}

\pgfplotscreateplotcyclelist{topic-line}{  
	{dollarbill,  mark=diamond, mark options={scale=0.2}, line width = 1.2pt}, 
	{color(cam),  mark=diamond, mark options={scale=0.2}},  
	{color(score),mark=diamond, mark options={scale=0.2}}, 
	{color(resit), mark=diamond, mark options={scale=0.2}}, 
	{color(icalingam), mark=diamond, mark options={scale=0.2}},  
	{color(notears), mark=diamond, mark options={scale=0.2}}, 
	{color(globe),  mark=diamond, mark options={scale=0.2}}, 
	{color(pc), mark=diamond, mark options={scale=0.2}}, 
	{color(fci), mark=diamond, mark options={scale=0.2}}, 
	{color(ges), mark=diamond, mark options={scale=0.2}}, 
	{color(cdnod), mark=diamond, mark options={scale=0.2}}, 
	{color(pcmci), mark=diamond, mark options={scale=0.2}}, 
	{color(varlingam), mark=diamond, mark options={scale=0.2}}, 
	{color(dynotears), mark=diamond, mark options={scale=0.2}}, 
}

\pgfplotscreateplotcyclelist{taborder-box}{%
	{pr-color1n, fill=pr-color1n!40}, 
	{pr-color1b, fill=pr-color1b!40}, 
	{pr-color1g, fill=pr-color1g!40}, 
	{pr-color1h, fill=pr-color1h!40}, 
	{pr-color1o, fill=pr-color1o!40}, 
	{pr-color1e, fill=pr-color1e!40},
	{pr-color1f, fill=pr-color1f!40}, 
	{pr-color1k, fill=pr-color1k!40}, 
	{tcss-slate4, fill=tcss-slate4!40},  
	{tcss-slate6, fill=tcss-slate6!40},  
	{tcss-slate8, fill=tcss-slate8!40},  
}

\colorlet{flipcol}{pr-color1b!80!black}
\colorlet{fliplow}{flipcol!10}
\colorlet{flipmid}{flipcol!50}
\colorlet{fliphigh}{flipcol}
\colorlet{rankposcol}{goldenrod}
\colorlet{ranknegcol}{pr-color1b!80!black} 
\newcommand{\flipnode}[4]{%
	\pgfmathparse{#4 < 50 ? int(2*#4) : int(2*(#4-50))}
	\let\mixval\pgfmathresult
	
	\ifnum#4<50
	\node[ 
	circle,
	minimum size=7mm,
	inner sep=1pt,
	fill=fliplow!\mixval!flipmid
	] (#1) #2 {#3};
	\else
	\node[ 
	circle,
	minimum size=7mm,
	inner sep=1pt,
	fill=flipmid!\mixval!fliphigh
	] (#1) #2 {#3};
	\fi
}
\newcommand{\SachsFlipGraph}[2]{%
	\begin{tikzpicture}[
		font=\scriptsize,
		prot/.style={
			draw=tcss-slate3,
			circle,
			minimum size=8mm,
			inner sep=1pt
		},
		ivwhite/.style={
			draw=white,
			white,
			fill=white,
			circle,
			dashed,
			thick,
			minimum size=7mm,
			inner sep=1pt,
			font=\scriptsize\bfseries
		},
		ivact/.style={
			draw=pr-color1a!60!black,
			fill=pr-color1a!10,
			circle,
			dashed,
			thick,
			minimum size=7mm,
			inner sep=1pt,
			font=\scriptsize\bfseries
		},
		ivinh/.style={
			draw=pr-color1b!70!black,
			fill=pr-color1b!10,
			circle,
			dashed,
			thick,
			minimum size=7mm,
			inner sep=1pt,
			font=\scriptsize\bfseries
		},
		>=Latex
		]
		
		\flipnode{pkc}{at (0,4)}{PKC}{\pkcflip}
		\flipnode{pka}{at (0,2.4)}{PKA}{\pkaflip}
		\flipnode{raf}{at (1.7,2.4)}{Raf}{\rafflip}
		\flipnode{jnk}{at (-1.7,2.4)}{Jnk}{\jnkflip}
		\flipnode{p38}{at (-1.7,1.1)}{P38}{\pflip}
		\flipnode{plcg}{at (-2.4,0.2)}{Plc${}_{\gamma}$}{\plcgflip}
		\flipnode{pip2}{at (-3.3,-1.0)}{PIP2}{\piptwoflip}
		\flipnode{pip3}{at (-1.7,-1.0)}{PIP3}{\pipthreeflip}
		\flipnode{akt}{at (0,-1.0)}{Akt}{\aktflip}
		\flipnode{erk}{at (1.7,-1.0)}{Erk}{\erkflip}
		\flipnode{mek}{at (1.7,1.1)}{Mek}{\mekflip}
		
		\path[-{Latex[length=2mm,width=1mm]}, line width=0.8pt, tcss-slate3]
		(raf) edge[] (mek)
		(mek) edge[] (erk)
		(plcg) edge[bend right=10] (pip2)
		(plcg) edge[bend left=10] (pip3)
		(pip3) edge[] (pip2)
		(pip3) edge[] (akt)
		(pip2) edge[bend left=35] (pkc)
		(pkc) edge[] (raf)
		(pkc) edge[bend left=12] (mek)
		(pkc) edge[] (jnk)
		(pkc) edge[bend right=12] (p38)
		(pka) edge[] (raf)
		(pka) edge[] (mek)
		(pka) edge[bend left=10] (erk)
		(pka) edge[] (akt)
		(pka) edge[] (jnk)
		(pka) edge[] (p38)
		;
		
		\node[\iOneStyle]   (i1) at (-1.45,5.45) {1};
		\node[\iTwoStyle]   (i2) at (1.45,5.45) {2};
		\node[\iThreeStyle] (i3) at (-1.55,4.15) {3};
		\node[\iFourStyle]  (i4) at (1.95,4.00) {4};
		\node[\iFiveStyle]  (i5) at (0,-2.35) {5};
		\node[\iSixStyle]   (i6) at (3.25,1.1) {6};
		\node[\iSevenStyle] (i7) at (-1.7,-2.35) {7};
		\node[\iEightStyle] (i8) at (-3.30,-2.35) {8};
		
		\iOneArrow
		\iTwoArrow
		\iThreeArrow
		\iFourArrow
		\iFiveArrow
		\iSixArrow
		\iSevenArrow
		\iEightArrow
		
	\end{tikzpicture}
} 

 \newcommand{\SachsFlipGraphSmall}[2]{%
 	\begin{tikzpicture}[
 		font=\scriptsize,
 		prot/.style={
 			draw=tcss-slate3,
 			circle,
 			minimum size=8mm,
 			inner sep=1pt
 		},
 		ivwhite/.style={
 			draw=white,
 			white,
 			fill=white,
 			circle,
 			dashed,
 			thick,
 			minimum size=7mm,
 			inner sep=1pt,
 			font=\scriptsize\bfseries
 		},
 		ivact/.style={
 			draw=pr-color1a!60!black,
 			fill=pr-color1a!10,
 			circle,
 			dashed,
 			thick,
 			minimum size=7mm,
 			inner sep=1pt,
 			font=\scriptsize\bfseries
 		},
 		ivinh/.style={
 			draw=pr-color1b!70!black,
 			fill=pr-color1b!10,
 			circle,
 			dashed,
 			thick,
 			minimum size=7mm,
 			inner sep=1pt,
 			font=\scriptsize\bfseries
 		},
 		>=Latex
 		]
 		
 		\flipnode{pkc}{at (0,4)}{PKC}{\pkcflip}
 		\flipnode{pka}{at (0,2.4)}{PKA}{\pkaflip}
 		\flipnode{raf}{at (1.7,2.4)}{Raf}{\rafflip}
 		\flipnode{jnk}{at (-1.7,2.4)}{Jnk}{\jnkflip}
 		\flipnode{p38}{at (-1.7,1.1)}{P38}{\pflip}
 		\flipnode{plcg}{at (-2.4,0.2)}{Plc${}_{\gamma}$}{\plcgflip}
 		\flipnode{pip2}{at (-3.3,-1.0)}{PIP2}{\piptwoflip}
 		\flipnode{pip3}{at (-1.7,-1.0)}{PIP3}{\pipthreeflip}
 		\flipnode{akt}{at (0,-1.0)}{Akt}{\aktflip}
 		\flipnode{erk}{at (1.7,-1.0)}{Erk}{\erkflip}
 		\flipnode{mek}{at (1.7,1.1)}{Mek}{\mekflip}
 		
 		\path[-{Latex[length=2mm,width=1mm]}, line width=0.8pt, tcss-slate3]
 		(raf) edge[] (mek)
 		(mek) edge[] (erk)
 		(plcg) edge[bend right=10] (pip2)
 		(plcg) edge[bend left=10] (pip3)
 		(pip3) edge[] (pip2)
 		(pip3) edge[] (akt)
 		(pip2) edge[bend left=35] (pkc)
 		(pkc) edge[] (raf)
 		(pkc) edge[bend left=12] (mek)
 		(pkc) edge[] (jnk)
 		(pkc) edge[bend right=12] (p38)
 		(pka) edge[] (raf)
 		(pka) edge[] (mek)
 		(pka) edge[bend left=10] (erk)
 		(pka) edge[] (akt)
 		(pka) edge[] (jnk)
 		(pka) edge[] (p38)
 		;
 		
 		\node[\iFourStyle]  (i4) at (1.95,4.00) {4};
 		\node[\iEightStyle] (i8) at (-3.30,-2.35) {8};
 		
 		\iOneArrow
 		\iTwoArrow
 		\iThreeArrow
 		\iFourArrow
 		\iFiveArrow
 		\iSixArrow
 		\iSevenArrow
 		\iEightArrow
 		
 	\end{tikzpicture}
 } 
\newcommand{\ResetSachsInterventions}{%
	\def\iOneStyle{ivwhite}%
	\def\iTwoStyle{ivwhite}%
	\def\iThreeStyle{ivwhite}%
	\def\iFourStyle{ivwhite}%
	\def\iFiveStyle{ivwhite}%
	\def\iSixStyle{ivwhite}%
	\def\iSevenStyle{ivwhite}%
	\def\iEightStyle{ivwhite}%
	\def\iOneArrow{}%
	\def\iTwoArrow{}%
	\def\iThreeArrow{}%
	\def\iFourArrow{}%
	\def\iFiveArrow{}%
	\def\iSixArrow{}%
	\def\iSevenArrow{}%
	\def\iEightArrow{}%
}

\newcommand{\ShowSachsInhibitorFour}{%
	\def\iFourStyle{ivinh}%
	\def\iFourArrow{\path[-{Latex[length=2mm,width=1mm]}, dashed, line width=0.8pt, pr-color1b!70!black] (i4) edge[] (pkc);}%
}

\newcommand{\ShowSachsInhibitorFive}{%
	\def\iFiveStyle{ivinh}%
	\def\iFiveArrow{\path[-{Latex[length=2mm,width=1mm]}, dashed, line width=0.8pt, pr-color1b!70!black] (i5) edge[] (akt);}%
}

\newcommand{\ShowSachsInhibitorSix}{%
	\def\iSixStyle{ivinh}%
	\def\iSixArrow{\path[-{Latex[length=2mm,width=1mm]}, dashed, line width=0.8pt, pr-color1b!70!black] (i6) edge[] (mek);}%
}

\newcommand{\ShowSachsInhibitorSeven}{%
	\def\iSevenStyle{ivinh}%
	\def\iSevenArrow{\path[-{Latex[length=2mm,width=1mm]}, dashed, line width=0.8pt, pr-color1b!70!black] (i7) edge[] (pip3) (i7) edge[bend right=10] (akt);}%
}

\newcommand{\ShowSachsInhibitorEight}{%
	\def\iEightStyle{ivinh}%
	\def\iEightArrow{\path[-{Latex[length=2mm,width=1mm]}, dashed, line width=0.8pt, pr-color1b!70!black] (i8) edge[] (pip2) (i8) edge[bend right=10] (pip3);}%
}
 
\newcommand{\dy}{1}
\newcommand{\xbase}{-2.0}
\newcommand{\xcond}{2.0}

\newcommand{\basenode}[3]{%
	\node[
	circle,
	minimum size=7mm,
	inner sep=1pt,
	fill=fliplow
	] (#1) at #2 {#3};
}

\newcommand{\ordernode}[4]{%
	\pgfmathparse{#4 < 50 ? int(2*#4) : int(2*(#4-50))}
	\let\mixval\pgfmathresult
	\ifnum#4<50
	\node[
	circle,
	minimum size=7mm,
	inner sep=1pt,
	fill=fliplow!\mixval!flipmid
	] (#1) at #2 {#3};
	\else
	\node[
	circle,
	minimum size=7mm,
	inner sep=1pt,
	fill=flipmid!\mixval!fliphigh
	] (#1) at #2 {#3};
	\fi
}

\newcommand{\OrderPanel}[4]{%
	\begin{tikzpicture}[font=\scriptsize]
		\node at (\xbase,0.55) {\texttt{cd3cd28}};
		\node at (\xcond,0.55) {#1};
		\node at (0,1.05) {$\divtop=#2$};
		
		#3
		#4
		
		\path[-{Latex[length=1.6mm,width=0.8mm]}, line width=0.6pt]
		(b1) edge (b2) (b2) edge (b3) (b3) edge (b4) (b4) edge (b5)
		(b5) edge (b6) (b6) edge (b7) (b7) edge (b8) (b8) edge (b9)
		(b9) edge (b10) (b10) edge (b11);
		
		\path[-{Latex[length=1.6mm,width=0.8mm]}, line width=0.6pt]
		(c1) edge (c2) (c2) edge (c3) (c3) edge (c4) (c4) edge (c5)
		(c5) edge (c6) (c6) edge (c7) (c7) edge (c8) (c8) edge (c9)
		(c9) edge (c10) (c10) edge (c11);
	\end{tikzpicture}
}

\newcommand{\BaselineOrder}{%
	\basenode{b1}{(\xbase,0)}{PKA}
	\basenode{b2}{(\xbase,-1*\dy)}{PKC}
	\basenode{b3}{(\xbase,-2*\dy)}{Plc${}_{\gamma}$}
	\basenode{b4}{(\xbase,-3*\dy)}{PIP3}
	\basenode{b5}{(\xbase,-4*\dy)}{Erk}
	\basenode{b6}{(\xbase,-5*\dy)}{Jnk}
	\basenode{b7}{(\xbase,-6*\dy)}{Akt}
	\basenode{b8}{(\xbase,-7*\dy)}{PIP2}
	\basenode{b9}{(\xbase,-8*\dy)}{P38}
	\basenode{b10}{(\xbase,-9*\dy)}{Mek}
	\basenode{b11}{(\xbase,-10*\dy)}{Raf}
}

\newcommand{\OrderUzero}{%
	\ordernode{c1}{(\xcond,0)}{PKC}{10}
	\ordernode{c2}{(\xcond,-1*\dy)}{Erk}{30}
	\ordernode{c3}{(\xcond,-2*\dy)}{PKA}{20}
	\ordernode{c4}{(\xcond,-3*\dy)}{Mek}{60}
	\ordernode{c5}{(\xcond,-4*\dy)}{Raf}{60}
	\ordernode{c6}{(\xcond,-5*\dy)}{Jnk}{40}
	\ordernode{c7}{(\xcond,-6*\dy)}{Plc${}_{\gamma}$}{40}
	\ordernode{c8}{(\xcond,-7*\dy)}{PIP3}{40}
	\ordernode{c9}{(\xcond,-8*\dy)}{P38}{40}
	\ordernode{c10}{(\xcond,-9*\dy)}{PIP2}{40}
	\ordernode{c11}{(\xcond,-10*\dy)}{Akt}{40}
}

\newcommand{\OrderGzero}{%
	\ordernode{c1}{(\xcond,0)}{P38}{80}
	\ordernode{c2}{(\xcond,-1*\dy)}{PIP2}{80}
	\ordernode{c3}{(\xcond,-2*\dy)}{Mek}{70}
	\ordernode{c4}{(\xcond,-3*\dy)}{Jnk}{80}
	\ordernode{c5}{(\xcond,-4*\dy)}{Plc${}_{\gamma}$}{60}
	\ordernode{c6}{(\xcond,-5*\dy)}{Raf}{50}
	\ordernode{c7}{(\xcond,-6*\dy)}{Akt}{80}
	\ordernode{c8}{(\xcond,-7*\dy)}{PIP3}{80}
	\ordernode{c9}{(\xcond,-8*\dy)}{PKA}{80}
	\ordernode{c10}{(\xcond,-9*\dy)}{PKC}{80}
	\ordernode{c11}{(\xcond,-10*\dy)}{Erk}{60}
}

\newcommand{\OrderPsitect}{%
	\ordernode{c1}{(\xcond,0)}{PIP2}{70}
	\ordernode{c2}{(\xcond,-1*\dy)}{Plc${}_{\gamma}$}{30}
	\ordernode{c3}{(\xcond,-2*\dy)}{PKA}{20}
	\ordernode{c4}{(\xcond,-3*\dy)}{PIP3}{20}
	\ordernode{c5}{(\xcond,-4*\dy)}{PKC}{30}
	\ordernode{c6}{(\xcond,-5*\dy)}{Akt}{30}
	\ordernode{c7}{(\xcond,-6*\dy)}{Jnk}{30}
	\ordernode{c8}{(\xcond,-7*\dy)}{P38}{10}
	\ordernode{c9}{(\xcond,-8*\dy)}{Erk}{40}
	\ordernode{c10}{(\xcond,-9*\dy)}{Raf}{10}
	\ordernode{c11}{(\xcond,-10*\dy)}{Mek}{10}
}

\newcommand{\OrderLy}{%
	\ordernode{c1}{(\xcond,0)}{PKA}{0}
	\ordernode{c2}{(\xcond,-1*\dy)}{Plc${}_{\gamma}$}{10}
	\ordernode{c3}{(\xcond,-2*\dy)}{PKC}{10}
	\ordernode{c4}{(\xcond,-3*\dy)}{PIP3}{0}
	\ordernode{c5}{(\xcond,-4*\dy)}{Erk}{0}
	\ordernode{c6}{(\xcond,-5*\dy)}{Jnk}{0}
	\ordernode{c7}{(\xcond,-6*\dy)}{P38}{20}
	\ordernode{c8}{(\xcond,-7*\dy)}{Akt}{10}
	\ordernode{c9}{(\xcond,-8*\dy)}{PIP2}{10}
	\ordernode{c10}{(\xcond,-9*\dy)}{Mek}{0}
	\ordernode{c11}{(\xcond,-10*\dy)}{Raf}{0}
}

\newcommand{\ShowSachsFlipsUzero}{
	\def\pkaflip{20}\def\pkcflip{10}\def\plcgflip{40}\def\pipthreeflip{40}
	\def\erkflip{30}\def\jnkflip{40}\def\aktflip{40}\def\piptwoflip{40}
	\def\pflip{40}\def\mekflip{60}\def\rafflip{60}
}

\newcommand{\ShowSachsFlipsAktInhib}{
	\def\pkaflip{0}\def\pkcflip{20}\def\plcgflip{20}\def\pipthreeflip{30}
	\def\erkflip{30}\def\jnkflip{40}\def\aktflip{20}\def\piptwoflip{40}
	\def\pflip{40}\def\mekflip{0}\def\rafflip{0}
}

\newcommand{\ShowSachsFlipsGzero}{
	\def\pkaflip{80}\def\pkcflip{80}\def\plcgflip{60}\def\pipthreeflip{80}
	\def\erkflip{60}\def\jnkflip{80}\def\aktflip{80}\def\piptwoflip{80}
	\def\pflip{80}\def\mekflip{70}\def\rafflip{50}
}

\newcommand{\ShowSachsFlipsLy}{
	\def\pkaflip{0}\def\pkcflip{10}\def\plcgflip{10}\def\pipthreeflip{0}
	\def\erkflip{0}\def\jnkflip{0}\def\aktflip{10}\def\piptwoflip{10}
	\def\pflip{20}\def\mekflip{0}\def\rafflip{0}
}

\newcommand{\ShowSachsFlipsPsitect}{
	\def\pkaflip{20}\def\pkcflip{30}\def\plcgflip{30}\def\pipthreeflip{20}
	\def\erkflip{40}\def\jnkflip{30}\def\aktflip{30}\def\piptwoflip{70}
	\def\pflip{10}\def\mekflip{10}\def\rafflip{10}
}

 \colorlet{rankneutral}{fliplow}
 \colorlet{rankposlow}{goldenrod!20}
 \colorlet{rankposhigh}{goldenrod}
 \colorlet{rankneglow}{pr-color1b!15}
 \colorlet{rankneghigh}{pr-color1b!80!black}
 
 \newcommand{\ranknode}[4]{%
 	\pgfmathtruncatemacro{\rankmix}{min(100,abs(#4)*12)}
 	\ifdim #4pt > 0pt
 	\node[
 	circle,
 	minimum size=7mm,
 	inner sep=1pt,
 	fill=rankposlow!\rankmix!rankposhigh
 	] (#1) #2 {#3};
 	\else
 	\ifdim #4pt < 0pt
 	\node[
 	circle,
 	minimum size=7mm,
 	inner sep=1pt,
 	fill=rankneglow!\rankmix!rankneghigh
 	] (#1) #2 {#3};
 	\else
 	\node[
 	circle,
 	minimum size=7mm,
 	inner sep=1pt,
 	fill=rankneutral
 	] (#1) #2 {#3};
 	\fi
 	\fi
 }

\newcommand{\SachsRankGraph}[2]{%
	\begin{tikzpicture}[
		font=\scriptsize,
		prot/.style={
			draw=tcss-slate3,
			circle,
			minimum size=8mm,
			inner sep=1pt
		},
		ivwhite/.style={
			draw=white,
			white,
			fill=white,
			circle,
			dashed,
			thick,
			minimum size=7mm,
			inner sep=1pt,
			font=\scriptsize\bfseries
		},
		ivact/.style={
			draw=pr-color1a!60!black,
			fill=pr-color1a!10,
			circle,
			dashed,
			thick,
			minimum size=7mm,
			inner sep=1pt,
			font=\scriptsize\bfseries
		},
		ivinh/.style={
			draw=pr-color1b!70!black,
			fill=pr-color1b!10,
			circle,
			dashed,
			thick,
			minimum size=7mm,
			inner sep=1pt,
			font=\scriptsize\bfseries
		},
		>=Latex
		]
		
		\ranknode{pkc}{at (0,4)}{PKC}{\pkcrank}
		\ranknode{pka}{at (0,2.4)}{PKA}{\pkarank}
		\ranknode{raf}{at (1.7,2.4)}{Raf}{\rafrank}
		\ranknode{jnk}{at (-1.7,2.4)}{Jnk}{\jnkrank}
		\ranknode{p38}{at (-1.7,1.1)}{P38}{\prank}
		\ranknode{plcg}{at (-2.4,0.2)}{Plc${}_{\gamma}$}{\plcgrank}
		\ranknode{pip2}{at (-3.3,-1.0)}{PIP2}{\piptworank}
		\ranknode{pip3}{at (-1.7,-1.0)}{PIP3}{\pipthreerank}
		\ranknode{akt}{at (0,-1.0)}{Akt}{\aktrank}
		\ranknode{erk}{at (1.7,-1.0)}{Erk}{\erkrank}
		\ranknode{mek}{at (1.7,1.1)}{Mek}{\mekrank}
		
		\path[-{Latex[length=2mm,width=1mm]}, line width=0.8pt, tcss-slate3]
		(raf) edge[] (mek)
		(mek) edge[] (erk)
		(plcg) edge[bend right=10] (pip2)
		(plcg) edge[bend left=10] (pip3)
		(pip3) edge[] (pip2)
		(pip3) edge[] (akt)
		(pip2) edge[bend left=35] (pkc)
		(pkc) edge[] (raf)
		(pkc) edge[bend left=12] (mek)
		(pkc) edge[] (jnk)
		(pkc) edge[bend right=12] (p38)
		(pka) edge[] (raf)
		(pka) edge[] (mek)
		(pka) edge[bend left=10] (erk)
		(pka) edge[] (akt)
		(pka) edge[] (jnk)
		(pka) edge[] (p38)
		;
		
		\node[\iOneStyle]   (i1) at (-1.45,5.45) {1};
		\node[\iTwoStyle]   (i2) at (1.45,5.45) {2};
		\node[\iThreeStyle] (i3) at (-1.55,4.15) {3};
		\node[\iFourStyle]  (i4) at (1.95,4.00) {4};
		\node[\iFiveStyle]  (i5) at (0,-2.35) {5};
		\node[\iSixStyle]   (i6) at (3.25,1.1) {6};
		\node[\iSevenStyle] (i7) at (-1.7,-2.35) {7};
		\node[\iEightStyle] (i8) at (-3.30,-2.35) {8};
		
		\iOneArrow
		\iTwoArrow
		\iThreeArrow
		\iFourArrow
		\iFiveArrow
		\iSixArrow
		\iSevenArrow
		\iEightArrow
		
	\end{tikzpicture}
}

\title{Learning Causal Orderings for \\In-Context Tabular Prediction }

\author{%
  Sascha Xu\thanks{Equal contribution.}\\
  CISPA Helmholtz Center\\
  \texttt{sascha.xu@cispa.de} \\
  \And
  Sarah Mameche\footnotemark[1]\\
  CISPA Helmholtz Center\\
  \texttt{sarah.mameche@cispa.de} \\
  \And
  Jilles Vreeken\\
  CISPA Helmholtz Center\\
  \texttt{jv@cispa.de}
}

\begin{document}

\maketitle

\begin{abstract}
In-context learning for tabular data sets  strong predictive standards in observational settings; it however  primarily relies on correlational structure, which becomes unreliable under distribution shift or intervention. While established methods to discover causal structure exist, they are often focused on structure  identifiability and decoupled from the predictive architectures that could benefit from them.  
  
To bridge these perspectives, we study how to simultaneously infer and enforce causal structure in the form of topological variable orderings into tabular prediction. Unlike standard architectures, our model \ourmethod uses causal order-constrained attention, basing predictions only on features that precede a target under a learned causal order. Similar to causal discovery methods, \ourmethod learns  the optimal variable ordering in an unsupervised manner through a likelihood-based objective. We justify this choice under standard functional model classes and also study how sample missingness, a common challenge in tabular data, interacts with causal direction identification. 
Empirically, we  confirm that \ourmethod recovers accurate variable orderings while addressing prediction and imputation tasks, as well as gives insight into real-world biological data under intervention.

\end{abstract}

\section{Introduction} 
In-context learning architectures such as 
Tabular foundation models (TFMs) achieve strong performance in predictive learning from tabular data through large-scale pre-training   
\citep{muller:22:transformers,hollmann:23:tabpfn}.      
However, these models are primarily associative  
and do not explicitly account for causal directionality \citep{pearl:09:causality} among features, such that their predictions could rely on non-causal associations sensitive to distribution shifts. 
We illustrate this effect in Figure~\ref{fig:three-variable-intervention-preds} in a three-variable chain $X \to Y \to Z$,  comparing an observational regime (green) to an intervention on the mechanism generating $Z$ (orange). This affects the predictive performance of a state-of-the art TFM  (Figure~\ref{fig:three-variable-intervention-preds}a), suggesting that it relies on the  non-causal statistical association between $Z$   and $Y$.  In contrast,   predictions with our proposed architecture   remain robust (Figure~\ref{fig:three-variable-intervention-preds}b) as the model relies only on the direct causal relationship between $X$   and $Y$.  
 
Identifying  such directional relationships is the primary aim of causal discovery \citep{pearl:09:causality}, usually with focus on recovering the underlying directed acyclic graph (DAG) from observational data. However, classical methods \citep{spirtes2001causation} are typically dataset-specific and detached from large-scale predictive architectures. Recent amortized approaches to causal discovery \citep{lorch:22:avici,ke2022learning} apply an in-context learning paradigm to causal discovery, relying on supervised training over large collections of known causal graphs.  
  However, these works  focus on a different question of whether causal structure recovery can be improved by supervised learning over large synthetic samples \citep[cf.][]{montagna2025demystifying}, whereas the question of how to systematically use causal structure in predictive architectures   is  relatively underexplored. 
 
In this work, we propose to connect the research perspectives behind predictive tabular models on the one hand, and causal structure discovery on the other hand, within a transformer-based framework.   Unlike standard architectures, our model  \ourmethod uses causal order-constrained attention, modeling the joint distribution as an ordered factorization under a learned causal ordering.  In contrast to classical causal discovery, \ourmethod infers this order efficiently within a single forward pass; in addition, most existing causal discovery methods address a fully observed regime, while we allow for missing samples with direct mechanisms for imputation. Finally, different from amortized causal discovery, we maintain an unsupervised approach without requiring ground-truth causal orders during training. 

Given that  recovering causal structure from purely observational data is impossible without further assumptions, we design our architecture around a likelihood-based objective that we justify in  Additive Noise Model (ANM) regimes under standard assumptions \citep{buhlmann:14:cam}.  
We study how the prediction of arbitrary missing values can be used to learn causal orderings, and show that the presence of missing samples can actually facilitate causal direction identification (Section \ref{sec:causality}).
Experiments across different classes of generating processes (Section \ref{sec:experiments}) confirm that  \ourmethod reliably infers accurate causal orderings, while maintaining  competitive predictive performance downstream tasks such as  missing  value imputation.

 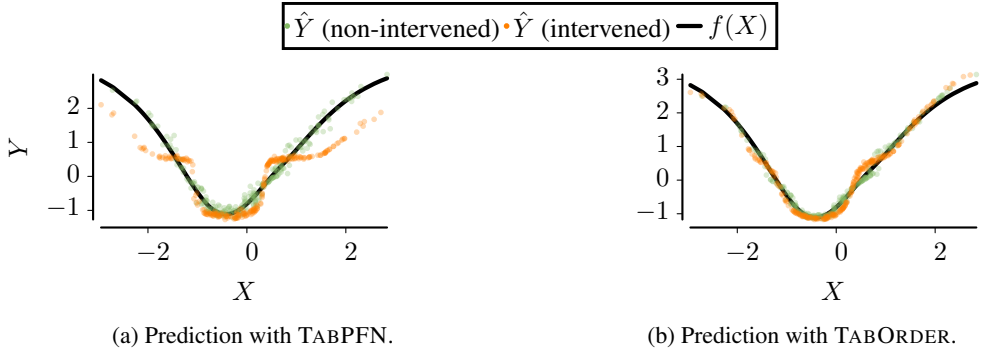
\begin{figure}[t]
  \centering
  \pgfplotstableread[col sep=comma]{expres/three_variable_intervention_predictions.csv}\datatable
  \pgfplotstablesort[sort key=x]{\datatableSorted}{\datatable}
    \pgfplotstableread[col sep=comma]{expres/three_variable_intervention_predictions_intervened.csv}\datatableintervened
    \pgfplotstableread[col sep=comma]{expres/three_variable_intervention_predictions_not_intervened.csv}\datatablenotintervened
  \begin{subfigure}{0.48\linewidth}
    \centering
    \begin{tikzpicture}
      \begin{axis}[
        width=.8\linewidth,
        height=3.5cm,
        pretty line,
        pretty labelshift,
        xlabel={$X$},
        ylabel={$Y$},
        legend style={inner sep=1pt,font=\scriptsize,at={(1.5,1.5)},anchor=north,legend columns=3},
        legend image post style={opacity=1},
      ]

        \addplot+[only marks,opacity=1, opacity=0.3, mark size=0.5pt, color=dollarbill]
          table[
            x=x,
            y=y_pred_tabpfn,
          ]{\datatablenotintervened};
        \addlegendentry{$\hat{Y}$ (non-intervened)}

        \addplot+[only marks,opacity=1,opacity=0.3, mark size=0.5pt, color=orange]
          table[
            x=x,
            y=y_pred_tabpfn,
          ]{\datatableintervened};
        \addlegendentry{$\hat{Y}$ (intervened)}

        \addplot+[no marks, ultra thick, color=black] table[x=x, y=y_noisefree]{\datatableSorted};
        \addlegendentry{$f(X)$}

      \end{axis}
    \end{tikzpicture}
    \subcaption{Prediction with \tabpfn.}
  \end{subfigure}
  \hfill
  \begin{subfigure}{0.48\linewidth}
    \centering
    \begin{tikzpicture}
      \begin{axis}[
        width=.8\linewidth,
        height=3.5cm,
        pretty line,
        pretty labelshift,
        xlabel={$X$},
      ]

        \addplot+[only marks,opacity=0.3, mark size=0.5pt, color=dollarbill]
          table[
            x=x,
            y=y_pred_cfm,
          ]{\datatablenotintervened};

        \addplot+[only marks,opacity=0.3, mark size=0.5pt, color=orange]
          table[
            x=x,
            y=y_pred_cfm,
          ]{\datatableintervened};

          \addplot+[no marks, ultra thick, color=black] table[x=x, y=y_noisefree]{\datatableSorted};

      \end{axis}
    \end{tikzpicture}
    \subcaption{Prediction with \ourmethod.}
  \end{subfigure}

  \caption{\textbf{Prediction under Intervention with and without Order Constrained Attention}. For prediction of a mediator $Y$ in a chain $X \to Y \to Z$, we measure
  test error without resp.~with intervention on $Y \to Z$. The generating mechanism $f(X)$ is shown in black. \tabpfn (a)  accurately models $X \to Y$ when no intervention is present (\textcolor{dollarbill}{green}), but fails under intervention (\textcolor{orange}{orange}). 
  \ourmethod (b) remains accurate in both settings by learning and leveraging the causal order.}
  \label{fig:three-variable-intervention-preds}
\end{figure}

\section{Causal Order-Constrained Prediction}
We begin by introducing the problem setting, then explain the   features of our architecture. 

Let $\featset = (\featvar_1,\dots,\featvar_{\nfeat})$ denote a collection of 
$\nfeat$ real-valued random variables drawn from an unknown joint distribution 
$P(\featset)$.  We observe a dataset 
$\dataset = \{\featvecr\}_{r=1}^{\nsamples}$ over $\nsamples$ samples, where each row
$\featvecr = (\cell{1}{r},\dots,\cell{\nfeat}{r}) \in \R^{\nfeat}$ 
is an independent realization of $\featset$, and where some entries $\cell{i}{r}$ are missing, indicated by a binary mask $\missingmask  = \{\missingij\}_{r=1}^{\nsamples}$ with $\missingij \in \{0,1\}$. Our objective is to learn the structured conditional dependencies between variables in $\featset$ to enable accurate and robust cell-wise predictions.

We base the architecture design on  \textbf{tabular foundation models} (TFMs) \citep{hollmann:23:tabpfn}, which model the posterior predictive distribution of a target variable $\cellij$ in row $r$ 
as $p_{\theta}(\feat_i \mid \featvecr_{i^*}, \dataset)$, where $\featvecr_{i^*}$ denotes all entries in row $r$ except $\cell{i}{r}$.

This approach however allows the use of all available features as predictors for $\cell{i}{r}$ without accounting for causal directionalities.
To address this, we propose modeling the joint distribution 
of a row $\featvecr$ using an \textbf{ordered factorization}  of the form
\begin{equation}
\label{eq:ordered-factorization}
p_\theta(\featvec^{(r)} \mid \dataset, \order)
=
\prod_{k=1}^{\nfeat}
p_\theta\!\left(
\cell{\order(k)}{r}\mid \cell{\order(1:k-1)}{r}\,, \dataset
\right)\;.
\end{equation}

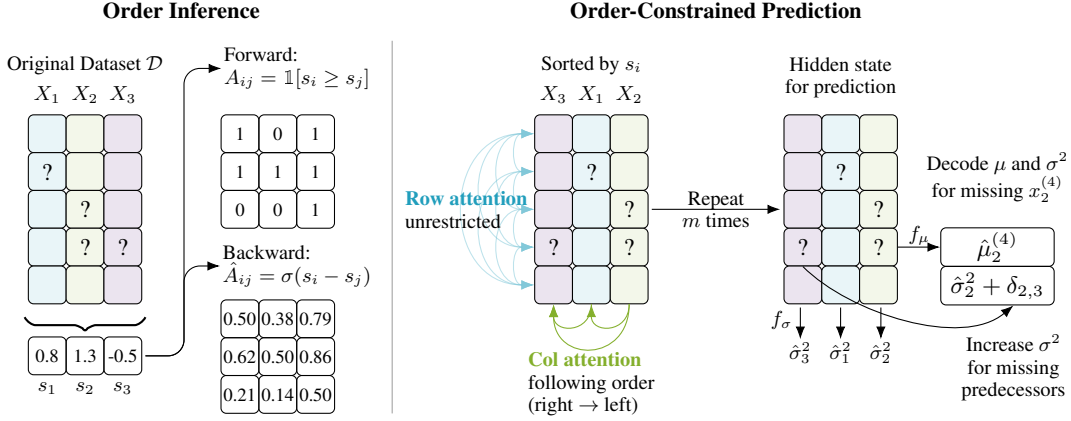
\begin{figure*}[!t]
    \centering 
\begin{tikzpicture}[ 
    font=\small,
    scale=0.85,
    transform shape,
    >=Latex,
    rounded corners=2pt,
    matrixstyle/.style={
        matrix of nodes,
        nodes in empty cells,
        nodes={draw, minimum width=0.5cm, minimum height=0.5cm, anchor=center},
        row sep=-\pgflinewidth,
        column sep=-\pgflinewidth
    },
    vecstyle/.style={
        matrix of nodes,
        nodes in empty cells,
        nodes={draw, minimum width=0.5cm, minimum height=0.5cm, anchor=center},
        row sep=0pt,
        column sep=-\pgflinewidth
    },
    attnstyle/.style={
        matrix of nodes,
        font=\scriptsize,
        nodes in empty cells,
        nodes={draw, minimum width=0.5cm, minimum height=0.5cm, anchor=center},
        row sep=-\pgflinewidth,
        column sep=-\pgflinewidth
    },
    every node/.style={inner sep=1pt}
] 
\matrix (D) [matrixstyle,
    column 1/.style={nodes={fill=arch-blue!10}},
    column 2/.style={nodes={fill=arch-green!10}},
    column 3/.style={nodes={fill=arch-orange!15}}
] {
      &       &       \\
    ? &       &       \\
      & ?     &       \\
      & ?      & ?     \\
       &       &       \\
};

\node[above=4pt] at (D-1-1.north) {$X_1$};
\node[above=4pt] at (D-1-2.north) {$X_2$};
\node[above=4pt] at (D-1-3.north) {$X_3$};

\node[above=15pt of D.north] {Original Dataset $\dataset$};



\matrix (score) [attnstyle, below=0.4cm of D] {
    0.8 & 1.3 & -0.5 \\
};
\node[below=2pt] at (score-1-1.south) {$s_1$};
\node[below=2pt] at (score-1-2.south) {$s_2$};
\node[below=2pt] at (score-1-3.south) {$s_3$};

\draw[
    decorate,
    decoration={brace,mirror,amplitude=6pt},
    thick
] 
    ($(D.south west)+(0,-0.15)$) -- ($(D.south east)+(0,-0.15)$);


\matrix (Ds) [matrixstyle,
    right=6cm of D,
    column 1/.style={nodes={fill=arch-orange!15}}, 
    column 2/.style={nodes={fill=arch-blue!10}},   
    column 3/.style={nodes={fill=arch-green!10}}   
] {
      &       &       \\
     &  ?      &       \\
      &      &  ?     \\
      ? &       & ?     \\
       &       &       \\
};

\node[above=4pt] at (Ds-1-1.north) {$X_3$};
\node[above=4pt] at (Ds-1-2.north) {$X_1$};
\node[above=4pt] at (Ds-1-3.north) {$X_2$};

\node[above=15pt of Ds.north] {Sorted by $s_i$};

\matrix (Dp) [matrixstyle,
    right=2cm of Ds,
    column 1/.style={nodes={fill=arch-orange!15}}, 
    column 2/.style={nodes={fill=arch-blue!10}},   
    column 3/.style={nodes={fill=arch-green!10}}   
] {
      &       &       \\
     &  ?      &       \\
      &      &  ?     \\
      ? &       & ?     \\
       &       &       \\
};
\node[above=15pt of Dp.north,anchor=center, align=center] {Hidden state\\for prediction};

\coordinate (pcol0) at ($(Dp.south west)!0.0!(Dp.south east)$);
\coordinate (pcol1) at ($(Dp.south west)!0.333!(Dp.south east)$);
\coordinate (pcol2) at ($(Dp.south west)!0.666!(Dp.south east)$);
\coordinate (pcol3) at ($(Dp.south west)!1.0!(Dp.south east)$);
\coordinate (pc1) at ($(pcol0)!0.5!(pcol1)$);
\coordinate (pc2) at ($(pcol1)!0.5!(pcol2)$);
\coordinate (pc3) at ($(pcol2)!0.5!(pcol3)$);

\node (sig1) at ($(pc1)+(0,-0.7cm)$) {$\hat{\sigma}^2_3$};
\node (sig2) at ($(pc2)+(0,-0.7cm)$) {$\hat{\sigma}^2_1$};
\node (sig3) at ($(pc3)+(0,-0.7cm)$) {$\hat{\sigma}^2_2$};
\draw[->] (pc1) -- (sig1.north);
\draw[->] (pc2) -- (sig2.north);
\draw[->] (pc3) -- (sig3.north);
\node[anchor=west, align=left] at ($(pc1)+(-0.5cm,-0.2cm)$) {$f_{\sigma}$};

\coordinate (predcell) at (Dp-4-3.east);
\coordinate (missxone) at (Dp-4-1.south);
\coordinate (sigmacol) at ($(Dp.south east)+(0.5cm,0)$);
\matrix (predvec) [matrixstyle, right=1.5cm of Dp,anchor=north, yshift=-0.23cm, nodes={minimum width=1.5cm, minimum height=0.51cm}] {
    $\hat{\mu}_2^{(4)}$ \\
    $\hat{\sigma}^2_2 + \varinc_{2,3}$ \\
};
\node[above=10pt,align=center] at (predvec-1-1.north) {Decode $\mu$ and $\sigma^2$\\for missing $\cell{2}{4}$};
\draw[->] (predcell.east) -- node[above, midway] {$f_{\mu}$} (predvec-1-1.west);
 \draw[->] (missxone.south) to[out=-45, in=220, looseness=0.8]
    ($(predvec-2-1.south)+(0.25cm,0)$); 
\node[anchor=north, align=center] at ($(predvec-2-1.south)+(0.25cm,-0.4cm)$)
    {Increase $\sigma^2$ \\for missing \\predecessors};

\matrix (A) [attnstyle, at={($(D.north)+(3cm,0)$)}, anchor=north] {
    1 & 0 & 1 \\
    1 & 1 & 1 \\
    0 & 0 & 1 \\
};
\node[anchor=south,align=left] (Alabel) at ($ (A.north) + (0.3,0.3cm) $) {Forward:\\$\attention_{ij}= \mathbbm{1}[s_i \ge s_j]$};

\matrix (Ab) [attnstyle, below=1.0cm of A] {
    0.50 & 0.38 & 0.79 \\
    0.62 & 0.50 & 0.86 \\
    0.21 & 0.14 & 0.50 \\
};
\node[above=3pt,align=left] (Ablabel) at ($ (Ab.north) + (0.3,0cm) $) {Backward:\\$\hat{\attention}_{ij}=\sigma(s_i - s_j)$};

\draw[->, rounded corners=6pt]
(score.east) -- ++(0.6cm,0) |- (Alabel.west); 
\draw[->, rounded corners=6pt]
(score.east) -- ++(0.6cm,0) |- (Ablabel.west);

\coordinate (sep) at ($(A.east)!0.28!(Ds.west)$);
\draw[black, opacity=.4, line width=.2pt] ($(sep)+(0,2.2cm)$) -- ($(sep)+(0,-3.6cm)$);

\coordinate (c1) at ($(Ds.south west)!0.166!(Ds.south east)$); 
\coordinate (c2) at ($(Ds.south west)!0.5!(Ds.south east)$);   
\coordinate (c3) at ($(Ds.south west)!0.833!(Ds.south east)$); 

\foreach \i/\j in {1/2,1/3,1/4,1/5,2/3,2/4,2/5,3/4,3/5,4/5} {
    \draw[<->, draw=arch-blue!60, opacity=0.5, line cap=round]
        (Ds-\i-1.west) to[out=180, in=180, looseness=1.6] (Ds-\j-1.west); 
}
\node[anchor=west, align=left] at ($(Ds.west)+(-2cm,0.cm)$) {\textbf{\textcolor{arch-blue}{Row attention}}\\unrestricted};

\node[anchor=north,align=left] at ($(Ds.south)+(0,-0.65cm)$) {\textbf{\textcolor{arch-green!100}{Col attention}}\\following order\\(right $\to$ left)};
\draw[->, line cap=round, draw=arch-green]
    (Ds-5-3.south) to[out=-90, in=-90, looseness=2] (Ds-5-2.south);
\draw[->, line cap=round, draw=arch-green]
    (Ds-5-3.south) to[out=-90, in=-90, looseness=2] (Ds-5-1.south);
\draw[->, line cap=round, draw=arch-green]
    (Ds-5-2.south) to[out=-90, in=-90, looseness=2] (Ds-5-1.south);
\draw[->, rounded corners=6pt]
    (Ds.east) to (Dp.west);
    
\node[align=center] (reptext)
    at ($(Ds.east)!0.5!(Dp.west)$) { Repeat\\$m$ times};

\node[above=40pt, font=\small, font=\bfseries] at ($(D.north)!0.5!(A.north)$) {Order Inference};
\node[above=40pt of Ds.north, font=\small, font=\bfseries] at ($(Ds.north)!0.5!(Dp.north)$) {Order-Constrained Prediction};

\end{tikzpicture}   
\caption{\textbf{Overview over \ourmethod.} For a  dataset $\dataset$ we map each column $i$ to an order score $s_i$. From these scores, we construct a hard/soft attention mask to constrain attention according to  the learned order (left). 
To predict missing entries in $\dataset$, we alternate row-wise attention (unrestricted) and column-wise causal attention (order-constrained).
From the resulting cell embeddings, we decode the conditional mean and a variance term accounting for missing potential causes (right).
}
\label{fig:dataset-to-order}
\end{figure*}
Under this factorization, the prediction for an entry $\cell{\order(k)}{r}$ is conditioned strictly on the variables preceding it in the order $\pi \in S_\nfeat$, where $S_\nfeat$ is the set of all possible variable permutations and  $\order(k)$ denotes the index of the variable at position $k$. 
This induces an asymmetric dependency structure, where each variable is predicted only from its predecessors.
If the learned order aligns with the underlying causal structure, the model is prevented from exploiting anti-causal or spurious correlations that do not generalize under intervention.

We are specifically interested in \textbf{causal orders} where $\pi$ is compatible with the underlying cause-effect relationships of the data-generating process, i.e., corresponds to a  topological sorting of an underlying Directed Acyclic Graph (DAG) \citep{pearl:09:causality}. However, ground-truth causal orders  $\order$ are often unknown or only partially observed in real-world settings \citep[e.g.,][]{sachs:05:sachs}. Therefore, we design the architecture  to learn an optimal ordering $\hat{\pi}$ end-to-end, without requiring ground-truth causal labels during training. We explain this approach in the following, then explain its correspondence with inferring causal orders in Section \ref{sec:causality}.

\subsection{\ourmethod Architecture}
\label{sec:architecture}
\ourmethod is a dual-transformer architecture that infers causal orders and performs order-constrained predictions in a single end-to-end pass. The model operates through two primary stages.
Given a table $\dataset \in \R^{\nsamples \times \nfeat}$ with missing entries indicated by $\missingmask \in \{0,1\}^{\nsamples \times \nfeat}$, each cell $\cellij$ and its missingness indicator $\missingij$ are embedded through a learned linear map
\begin{equation}
    \mathbf{e}_{i}^{(r)}
    = W_{\mathrm{emb}}
      \begin{bmatrix}
      \cellij \\ \missingij
      \end{bmatrix}
    + \mathbf{b}_{\mathrm{emb}}\;,
    \qquad
    W \in \R^{h \times 2}\;.
\end{equation}
These embeddings are processed by alternating row-wise and
column-wise self-attention layers in two specialized transformer branches:
an unconstrained branch that induces a variable ordering $\hatorder$ via scalar scores $\orderscores$ (\emph{Order Inference}), 
and a masked branch that models the conditional distribution of each cell, using a mask constructed from the learned scores $\orderscores$ (\emph{Order-Constrained Prediction}). 
We   provide a high-level overview of \ourmethod in Fig.~\ref{fig:dataset-to-order} 
 and explain its components in detail below.

\paragraph{Order Inference.}
To infer the ordering $\hatorder$, we apply a standard multi-head self-attention transformer \citep{vaswani2017attention} to the initial cell embeddings, alternating between 
row- and column-wise attention for $m$ steps. This results in a hidden state $\Hord \in \R^{\nsamples \times \nfeat \times h}$.

We model each variable's $\featvar_i$ position in the order via a scalar score $\orderscore_i \in \R$, where lower scores indicate an earlier position. 
We decode these scores from the hidden state $\Hord$ by applying a two-layer MLP $f_{\text{ord}}:\R^h\to \R$ to each cell representation, and averaging across rows as per
\begin{equation}
    \orderscore_i(\dataset) = \frac{1}{\nsamples} \sum_{r=1}^{\nsamples} f_{\text{ord}}\left(\Hord_{r,i}\right)\;.
\end{equation}
To enforce the order imposed by the scores $\mathbf{s}$ 
in the \emph{Order-Constrained Prediction} branch, we construct a causal attention mask $\attention^\mathbf{s}$, where $\attention^\mathbf{s}_{i,j} = \mathbbm{1}(s_j \leq s_i)$, meaning that column $i$ may only attend to columns with lower score.
To maintain differentiability, we use a straight-through relaxation in the backward pass as 
\begin{equation}
    \hat{\attention}^\mathbf{s}_{i,j}
    = \operatorname{sigmoid}\!\left(\frac{s_i - s_j}{\tau}\right)\;,
\end{equation}
where the temperature $\tau$ controls the sharpness of the ranking.

The \emph{Order Inference} transformer thus produces a differentiable ordering that controls the information flow in the \emph{Order-Constrained Prediction} branch. 
Since both components are trained jointly, the model is encouraged to place variables earlier in the order when they improve prediction of downstream variables. 
For causal mechanisms with lower prediction error in the causal direction \citep{blobaum:18:reci}, 
incorrect orderings incur systematically higher error and are therefore suboptimal.

\paragraph{Order-Constrained Prediction.}
We now model the structured conditional dependencies using the causal mask $\attention^{\mathbf{s}}$.
We apply a second transformer to the initial cell embeddings, alternating row- and column-wise attention for $m$ steps. 
While row-wise attention remains unconstrained to allow for in-context information sharing across samples, column-wise attention is strictly restricted by $\attention^{\mathbf{s}}$,
ensuring that the hidden state in column $i$ depends only on its predecessors in the learned order $\hatorder$.

This yields predictive hidden states $\Hpred \in \R^{\nsamples \times \nfeat \times h}$ that parameterize the 
conditional factorization of Eq.~\eqref{eq:ordered-factorization}.
We model each variable's conditional distribution using \textbf{additive noise models} (ANMs)
\begin{equation}\label{eq:anm}
    X_i = f(X_{\pa(i)}) + N_i,\; N_i \indep X_{\pa(i)}\;,
\end{equation}
with Gaussian noise $N_i \sim \mathcal{N}(0,\sigma_i^2)$. 
To that end, we decode the conditional mean and the variance, averaged per variable, via two MLPs
$f_{\mu}$ and $f_{\sigma}$ as
\begin{equation}
    \hat{\mu}_i^{(r)} = f_{\mu}\left(\Hpred_{r,i}\right),\quad
    \hat{\sigma}_i^2 = \frac{1}{\nsamples} \sum_{r=1}^{\nsamples} f_{\sigma}\left(\Hpred_{r,i}\right)\;.
\end{equation}
\setlength{\intextsep}{4pt}
 \begin{wrapfigure}{r}{0.45\textwidth} 
\centering
\begin{tikzpicture}[
    font=\small,
    >=Latex,
    matrixstyle/.style={
        matrix of nodes,
        nodes in empty cells,
        nodes={
            draw,
            minimum width=1.3cm,
            minimum height=0.75cm,
            anchor=center
        },
        row sep=-\pgflinewidth,
        column sep=-\pgflinewidth
    }
]


\matrix (tab) [
    matrixstyle,
    below=0.4cm of {(0,1.8)},
    column 1/.style={nodes={fill=arch-blue!10,minimum width=1.4cm,}},     
    column 2/.style={nodes={fill=arch-green!10,minimum width=1.4cm,}},    
    column 3/.style={nodes={minimum width=3.5cm, align=left}} 
] {
   {$X_1$} 
    & {$X_2$} 
    &
       {$X_3 = X_1 + (X_2)^2 + N$} \\
    {1}   & {2}  &
        {$\hat{\mu}_3^{(1)} = 5,\ \hat{\sigma}_3^{2(1)} = 1$} \\
    {0}   & {?}  &
        {$\hat{\mu}_3^{(2)} = 0,\ \hat{\sigma}_3^{2(2)} = 1 + \textcolor{ufogreen}{2}$} \\
    {?}   & {-1} &
        {$\hat{\mu}_3^{(3)} = 1,\ \hat{\sigma}_3^{2(3)} = 1 + \textcolor{blue}{1}$} \\
};
\end{tikzpicture}
\caption{\textbf{Effect of Missingness.} In an additive noise model, $X_3 = X_1 + (X_2)^2 + N$ where $ N, X_1, X_2 \sim \mathcal{N}(0,1)$, when either parent of $X_3$ is missing, the model increments the variance $\hat{\sigma}^2$   (Eq. \eqref{eq:variance-increment}) by a learned amount (blue for $X_1$, green for $X_2$) that reflects the effect of missingness of causal parents. 
}
\label{fig:anm-variance-increment}

\end{wrapfigure}
\paragraph{Prediction under Partial Observation.}
Since the model predicts potentially multiple missing entries per row, missing causal parents affect the estimation of the functional relationships in Eq.~\eqref{eq:anm}, leading to increased uncertainty when relevant inputs are unobserved. 
We capture this effect via a learnable variance increment $\varincij \in \R^+$, decoded from $\Hpred$, such that the point-wise variance for a cell $\cellij$ is given by
\begin{equation}
    \label{eq:variance-increment}
    \hat{\sigma}_i^{2 (r)} = \hat{\sigma}_i^2 + \sum_{j \in [\nfeat]} \varincij \cdot \mathbbm{1}(s_j < s_i)\cdot \missing_{j}^{(r)}\;.
\end{equation}
Intuitively, the model increases predictive uncertainty when variables that precede $X_i$ in the learned order are missing. 
For example, in a collider structure $X_1 \rightarrow X_3 \leftarrow X_2$, missing either parent increases the uncertainty in predicting $X_3$ (Fig. \ref{fig:anm-variance-increment}).
Moreover, we show that this mechanism amplifies the likelihood gap between correct and incorrect orderings, thus facilitating the learning of accurate causal orders (Section \ref{sec:missingness-causality}).

\paragraph{Training.}
To train our model, we optimize all components jointly by maximizing the
log-likelihood of masked entries under the ordered conditional model.
We assume access to a distribution over datasets $\dataset \sim \mathcal{P}$, from which training instances are sampled. 
For each dataset, we generate a prediction task by masking entries completely at random with probability $q = 20\%$.
Given a masked dataset $(\dataset, \missingmask)$, we first infer an attention mask $\attention^\mathbf{s}$ via the order induction
module.  Conditioned on this mask, the predictive transformer produces
cell-wise mean 
$\hat{\mu}_i^{(r)}$ and variance estimates $\hat{\sigma}_i^{2(r)}$.
The training objective is then to maximize the average log-likelihood $\mathcal{L}$ of the
masked entries,
\begin{equation} 
    \mathcal{L}
    = \frac{1}{|\missingmask|}
      \sum_{\missing_{i}^{(r)} = 1}
      \log \left[
      p_{\mathcal{N}}\!\Bigl(
          \feat_{i}^{*(r)} \,\big|\,
          \hat{\mu}_i^{(r)},
          \hat{\sigma}_i^{2(r)}
      \Bigr)\right].
      \label{eq:log-likelihood}
\end{equation}

Maximizing Eq.~\eqref{eq:log-likelihood} couples order learning and prediction: orderings that better explain masked entries receive higher likelihood.
Together, the architecture and objective of \ourmethod provide an in-context learning framework for inferring variable orderings and structured conditional dependencies.

We describe our exact implementation including relevant hyperparameter choices in Appendix~\ref{app:model-details}.
\section{Causal Order Inference}
\label{sec:causality} 
In this section, we provide a theoretical perspective on \ourmethod{} and examine under which conditions its architecture and likelihood-based training objective are aligned with recovering causal structure. 
In particular, we focus on learning variable orderings without requiring ground-truth causal graphs during training.
 
As is standard, we  assume the data to be generated by a Structural Causal Model (SCM)
\citep{pearl:09:causality} with structural equations
    $X_i = f_i(\pa_i, N_i)$,
where $\pa_i$ denotes the set of direct causes (parents) of $X_i$,
and $N_i$ is an exogenous noise variable.
The causal structure of the entire system is captured by a directed acyclic graph (DAG). 
A \textbf{topological ordering} of a DAG is any
order in which all causal parents precede their children; for such an ordering
$\pi$, the joint density factorizes as
\begin{equation}\label{eq:causal-factorization}
    p(\featvec) = \prod_{i=1}^{\nfeat} p\!\left(\feat_{\order(i)} \mid x_{\pi(1:i-1)}\right) = \prod_{i=1}^\nfeat p\!\left(x_i \mid x_{\pa_i}\right)\;.
\end{equation}
Hence we refer to $\order\in S_\nfeat$  as a causal order for $X$ if it corresponds with a topological ordering of the underlying causal DAG. We make standard assumptions of the causal Markov Condition, faithfulness, and causal sufficiency, listed in detail in Appendix \ref{app:background}.

\paragraph{Order Inference via Likelihood-Based Objectives.} Identifying the causal graph and thereby a topological ordering is feasible for certain classes of functional mechanisms \citep{shimizu:06:lingam,hoyer:08:anm,zhang:09:pnl}.
As our learning objective is based on minimizing conditional error variance, we focus on \textbf{causal
	additive models} (CAM) \citep{buhlmann:14:cam}, 
\begin{equation}
	X_i = \sum_{j \in \pa_i} f_{ij}(X_j) + N_i,
	\label{eq:cam-model-a}
\end{equation}
with smooth functions $f_{ij}$ and independent noise $N_i \indep X_i$. 
For this class, the likelihood-optimal order is consistent with the topological orders of the true DAG, such that scoring criteria based on residual variances consistently recover valid causal orders  \citep{buhlmann:14:cam}. 
As \ourmethod jointly predicts a  variable order $\hat{\pi}$ as well as models conditional expectations 
under  the order-constrained SCM (Eq. \ref{eq:anm}) via likelihood maximization, this steers
training toward orderings that are consistent with the underlying causal structure, without requiring ground-truth graphs.

\paragraph{Order Inference  under Partial Observation.}
\label{sec:missingness-causality}
In the presence of missing values, \ourmethod includes an explicit mechanism to capture the effect on the predictive uncertainty in the additive noise model in Eq.~\eqref{eq:anm}  (see Figure~\ref{fig:anm-variance-increment}).
We now show that in certain cases, we can leverage an asymmetry in the effect of missingness to further increase  likelihood terms in the causal direction. 

To illustrate, consider the bivariate cause effect pair $X \to Y$.
Let $M\sim\mathrm{Bernoulli}(q)$ indicate missingness of the parent, and 
for direction $\tau$, define $v_{0,\tau}$ as the residual variance when the parent is observed ($M=0$) and $v_{1,\tau}$ when it is missing ($M=1$).

A Gaussian model with a single variance parameter in direction $\tau$ must match the mask-averaged residual variance
$\bar\sigma_{\tau}^2 = (1-q)\,v_{0,\tau} + q\,v_{1,\tau}\;,$
yielding population log-likelihood
$
\mathcal L_{\tau,\mathrm{fix}}^\star = -\tfrac12 \log \bar\sigma_{\tau}^2 + \text{const}\;.
$
A missingness-aware model fits separate variances for $M=0$ and $M=1$, giving
$
\mathcal L_{\tau,\mathrm{inc}}^\star
= -\tfrac12\bigl[(1-q)\log v_{0,\tau} + q\log v_{1,\tau}\bigr] + \text{const}\;.
$
Since $-\log$ is strictly convex, if $v_{0,\tau}\neq v_{1,\tau}$ then Jensen's inequality implies
$\mathcal L_{\tau,\mathrm{inc}}^\star > \mathcal L_{\tau,\mathrm{fix}}^\star\;,$ 
that is accounting for missingness strictly improves the likelihood in either direction.

The above suggests that the variance increment in Eq. \eqref{eq:variance-increment} helps the model distinguish cause from effect.
We now examine conditions under which this is the case. 
To that end, we compare 
\begin{equation}
\Delta_{\mathrm{fix}} := \mathcal{L}_{\text{fwd},\mathrm{fix}}^\star - \mathcal{L}_{\text{bwd},\mathrm{fix}}^\star \; \text{and} \; \Delta_{\mathrm{inc}} := \mathcal{L}_{\text{fwd},\mathrm{inc}}^\star - \mathcal{L}_{\text{bwd},\mathrm{inc}}^\star\;.
\end{equation}
Then, we can show the following.
\begin{restatable}[Effect of Missingness on Likelihood Asymmetry]{theorem}{theoremmissing}\label{thm:snr-amplification}
Consider an additive noise model $Y=f(X)+N$ with $N\indep X$ and $\Var(N)>0$ and missingness
indicator $M_X, M_Y\sim\mathrm{Bernoulli}(q)$ with $q\in(0,1)$.
Then, the difference in likelihood is amplified
iff the forward signal-to-noise ratio $\Var(f(X))/\Var(N)$ exceeds the backwards ratio $\Var(\E[X\mid Y])/\E[\Var(X\mid Y)]$, i.e.,
\begin{align}
\Delta_{\mathrm{inc}} > \Delta_{\mathrm{fix}}
\quad\Longleftrightarrow\quad
\frac{\Var(f(X))}{\Var(N)} > \frac{\Var(\E[X\mid Y])}{\E[\Var(X\mid Y)]}\;.
\end{align}
\end{restatable}
We defer the proof to Appendix~\ref{app:missingness-proof}.
The result shows that optimizing the likelihood with the variance increment in Eq. \eqref{eq:variance-increment} further increases 
the likelihood gap, facilitating  learning of causally aligned orders in identifiable settings such as CAMs.
More generally, in settings where causal structure and likelihood-optimal factorizations coincide \citep{wendong2025algorithmic}, likelihood maximization provides a principled approach to learning variable orderings without requiring ground-truth graphs. 
\ourmethod{} instantiates this principle in an in-context learning framework. 

\section{Related Work}
\label{sec:related}
\paragraph{Tabular Foundation Models.} 
Recent deep learning approaches for tabular data adopt an in-context learning paradigm using transformer-based models trained on large collections of tables \citep{muller:22:transformers}. 
\tabpfn introduced a transformer pre-trained on synthetic tasks drawn from a prior over structural mechanisms, achieving state-of-the-art performance without finetuning \citep{hollmann:23:tabpfn,hollmann:25:tabpfn2}. Subsequent work improves scalability \citep{qu2025tabicl} and retrieval \citep{ma2025tabdpt}.  
Existing TFMs primarily focus on predicting a single target variable and do not impose structural constraints on feature--feature relationships. Work on causality in TFMs has so far focused on treatment effect estimation \citep{robertson2025dopfn,balazadeh2025causalpfn}, typically relying on amortized learning over a synthetic SCM prior rather than inferring dataset-specific causal structure.
  
\paragraph{Missing Value Imputation.} 
Classical imputation methods include mean imputation, k-nearest neighbors, and Multiple Imputation by Chained Equations (\mice) \citep{van2011mice}, which iteratively fits conditional models. MIRACLE \citep{kyono2021miracle} explicitly models structured missingness by conditioning on estimated causes of missingness, improving robustness under non-random missingness. This is complementary to our approach, which focuses on learning causal relationships between features to improve imputations.
Deep generative models such as \gain \citep{yoon2018gain} and \miwae \citep{mattei2019miwae} learn joint feature distributions, while transformer-based approaches such as \naim \citep{caruso2024not} model conditional distributions directly. \tabimpute \citep{feitelberg2025tabimpute} leverages pretraining on synthetic missingness tasks but does not impose structural constraints.

\paragraph{Causal Discovery.}
Identifying causal structure from observational data requires assumptions on functional forms or noise distributions \citep{shimizu:06:lingam,hoyer:08:anm,zhang:09:pnl}. \ourmethod is conceptually related to order-based discovery methods \citep{buhlmann:14:cam,rolland:22:score,montagna:23:das,montagna:23:nogam,xu:25:topic} and gradient-based DAG learners \citep{zheng:18:notears, Lachapelle2020Gradient}. However, these methods solve dataset-specific optimization problems to infer causal structure from a single dataset. In contrast, amortized approaches like \avici \citep{lorch:22:avici} learn to infer graphs across datasets through supervised training on ground-truth graphs \citep{ke2022learning}, although it has been suggested that  this supervised approach may be sensitive to synthetic priors \citep{montagna2025demystifying}.  \ourmethod uniquely combines amortized learning with unsupervised, likelihood-based order discovery for downstream prediction.

\section{Experiments}\label{sec:experiments}
We evaluate whether \ourmethod achieves a favorable trade-off between I. causal order learning and II. tabular prediction and imputation tasks. 

We consider both synthetic data with known causal structure, and real-world datasets with and without ground-truth graphs. 
Synthetic experiments assess the accuracy of learned orders under a range of nonlinear mechanisms (additive and non-additive), noise levels, and sample sizes (Appendix~\ref{app:synthetic-data}). 
Real-world experiments evaluate predictive performance in missing value imputation and downstream tasks, including robustness under distribution shift and intervention.


\paragraph{Causal Order Inference (Trade Off I).} 
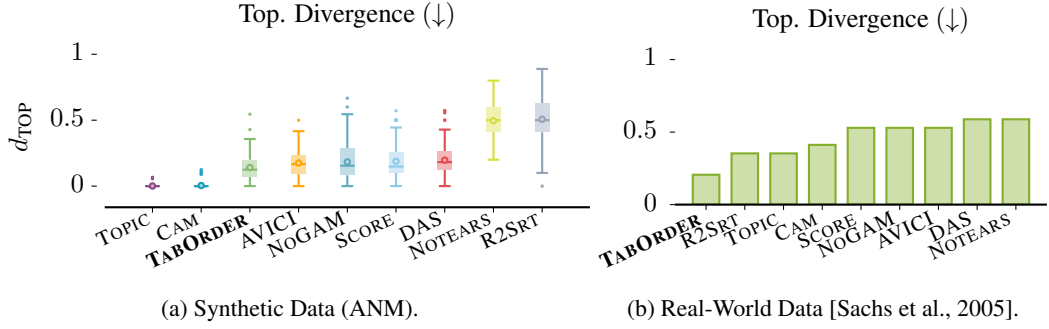
\begin{figure}[t]
	\begin{subfigure}[t]{0.55\linewidth} 
		\begin{tikzpicture}
			\begin{axis}[ pretty boxplot, cycle list name =taborder-box,   height=3.5cm, width=8cm, legend  columns=3, ylabel = {$\divtop$}, legend style={at={(1,1.7)}, font=\scriptsize}, xtick={1,2,3,4,5,6,7,8,9},
			xticklabels={
				\topic, \cam, \textbf{\ourmethod},
				\avici,\nogam, \scoremethod, \dasmethod, \notears, \rtwosort
			},
			major tick length=2pt,
			title ={Top. Divergence  $(\downarrow)$}, ymax=1, pretty labelshift,
			xticklabel style={rotate=25,anchor=east,font=\footnotesize}, xtick style={draw=black,
				yshift=-3pt},
			]
				\foreach \col in {topic,cam,non-additive-unregularized-gp-small,avici-0-5,nogam,score,das,notears,randsort} {
	\addplot table[y=\col] {figs/texdata/order_box_additive_gp_topological_divergence.tsv};
}
			\end{axis}
		\end{tikzpicture}
		\caption{ Synthetic Data (ANM).}
	\end{subfigure}
	\begin{subfigure}[t]{0.45\linewidth}   
\begin{tikzpicture}
	\begin{axis}[
		pretty ybar,
		symbolic x coords={\textbf{\ourmethod},\rtwosort,\topic,\cam,\scoremethod,\nogam,\avici,\dasmethod,\notears},
		xtick=data,
		xticklabel style={rotate=25, anchor=east,font=\footnotesize},
		ylabel={ },
		title ={Top. Divergence  $(\downarrow)$},  
		ymin=0, ymax=1,
		height=3.5cm,
		width=6.5cm,
		ytick={0,0.5,1},
		ytick style={draw=black},major y tick style={draw=black},
		]
		\addplot+[
		fill=arch-green!40,
		draw=arch-green,
		]
	coordinates {
		(\textbf{\ourmethod},0.205882)
		(\rtwosort,0.352941)
		(\topic,0.352941)
		(\cam,0.411765)
		(\scoremethod,0.529412)
		(\nogam,0.529412)
		(\avici,0.529412)
		(\dasmethod,0.588235)
		(\notears,0.588235)
	};
	\end{axis}
\end{tikzpicture}
	\caption{Real-World Data \citep{sachs:05:sachs}.}
	\end{subfigure} 
	\caption{\textbf{Causal Order Inference.} Shown is the topological divergence ($\divtop$, lower is better) of estimated causal orders on synthetic datasets  generated from nonlinear Gaussian process ANMs (\textbf{a}) and on the real-world causal discovery benchmark by \cite{sachs:05:sachs} (\textbf{b}).  }	\label{fig:order}
\end{figure}
We first evaluate the quality of the causal orders discovered by \ourmethod. 
Given a learned order $\hatorder$ and ground-truth DAG $G$, we measure topological divergence
  $\divtop(\hatorder,G) = \sum_{i=1}^{\nfeat} \sum_{j:\hatorder(i)\geq \hatorder(j)} \mathbb I[(i,j)\in G]$,
   where lower values indicate  fewer edge direction disagreements of $\hatorder$ with $\G$. For comparison, we consider a selection of recent topological-ordering based causal structure learning algorithms, \cam{} \citep{buhlmann:14:cam}, \scoremethod{} \citep{rolland:22:score}, \dasmethod{} \citep{montagna:23:das}, \nogam{} \citep{montagna:23:nogam}, \notears{} \citep{zheng:18:notears} and \topic{} \citep{xu:25:topic},   
   as well as the amortized learner \avici \citep{lorch:22:avici}  and the simple sorting baseline
   \rtwosort \citep{reisach:23:sortability}.

Figure \ref{fig:order} shows the topological divergence $\divtop$ across methods. On synthetic data (Figure \ref{fig:order} (a)), \ourmethod discovers causal orderings  of matching quality compared to the outputs of structure learning algorithms tailored to this task. It discovers such orders without relying on explicit regression and graph learning (\cam, \topic) and accompanying scalability constraints (\ref{app:scalability}), but allows order inference in near-instant time; and without relying on ground truth graphs during training (\avici).

\setlength{\intextsep}{4pt}
\vspace{-0.2cm}
\begin{wrapfigure}{r}{0.52\textwidth}
	\centering
	
	\begin{subfigure}[t]{0.48\linewidth}
		\centering
		
		\begin{tikzpicture}
			\begin{axis}[
				pretty ybar,
				symbolic x coords={Correct Order,\ourmethod,Incorrect Orders},
				xtick=data,
				xticklabel style={
					rotate=25,
					anchor=east,
					font=\footnotesize
				},
				ylabel={NLL $(\downarrow)$}, 
				ymin=0,
				ymax=0.30,
				height=3cm,
				width=\linewidth,
				enlarge x limits=0.18
				]
				
				\addplot+[
				fill=arch-green!40,
				draw=arch-green,
				error bars/.cd,
				y dir=both,
				y explicit,
				]
				coordinates {
					(Correct Order,0.071) +- (0,0.000)
					(\ourmethod,0.082) +- (0,0.000)
					(Incorrect Orders,0.257) +- (0,0.017)
				};
				
			\end{axis}
		\end{tikzpicture}
		
		\caption{Replacing the  order with correct vs. incorrect orders.}
		
	\end{subfigure}
	\hfill
	\begin{subfigure}[t]{0.48\linewidth}
		\centering
		
		\begin{tikzpicture}
			\begin{axis}[
				pretty line,
				pretty labelshift,
				cycle list name=pr-colors-conf, 
				xlabel={$\divtop$ $(\downarrow)$},
				ylabel={NLL $(\downarrow)$},
				xmin=0,
				xmax=0.5269,
				ymin=0,
				ymax=0.3 ,
				height=3cm,
				width=\linewidth,
				]
				
				\addplot+
				table[row sep=\\, x=x, y=y] {
					x y ci \\
					0.000000 0.083517 0.000000 \\
					0.227756 0.146868 0.006600 \\
					0.318477 0.177129 0.010920 \\
					0.426564 0.229967 0.006003 \\
					0.465663 0.250829 0.006040 \\
					0.501832 0.248546 0.007044 \\
				};
				
				\addplot[
				forget plot,
				draw=none,
				name path=low0
				]
				table[row sep=\\, x=x, y expr=\thisrow{y} - \thisrow{ci}] {
					x y ci \\
					0.000000 0.083517 0.000000 \\
					0.227756 0.146868 0.006600 \\
					0.318477 0.177129 0.010920 \\
					0.426564 0.229967 0.006003 \\
					0.465663 0.250829 0.006040 \\
					0.501832 0.248546 0.007044 \\
				};
				
				\addplot[
				forget plot,
				draw=none,
				name path=up0
				]
				table[row sep=\\, x=x, y expr=\thisrow{y} + \thisrow{ci}] {
					x y ci \\
					0.000000 0.083517 0.000000 \\
					0.227756 0.146868 0.006600 \\
					0.318477 0.177129 0.010920 \\
					0.426564 0.229967 0.006003 \\
					0.465663 0.250829 0.006040 \\
					0.501832 0.248546 0.007044 \\
				};
				
				\addplot fill between[of=low0 and up0];
				
			\end{axis}
		\end{tikzpicture}
		
		\caption{NLL vs. divergence between  used   and true   order.}
		
	\end{subfigure}
	
	\caption{\textbf{Order Impact.}
		Shown is the effect of replacing learned orders at test time. Incorrect orders degrade prediction (NLL) as the considered order diverges from the true causal order ($\divtop$).}
	\label{fig:orderablation}
\end{wrapfigure}
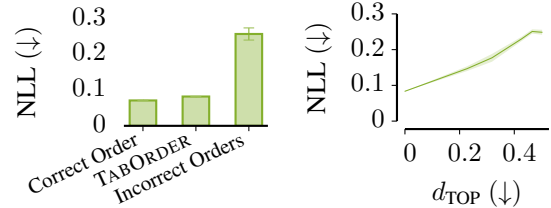
\paragraph{Order Impact.}
Next we assess predictive performance in relation to the quality of the learned causal order. We impose different orders when evaluating \ourmethod,  including a correct causal order,  the learned order, and multiple sampled incorrect orders. To show the effect on missing value imputation, we report the Negative Log-Likelihood (NLL, smaller is better). 
As Figure~\ref{fig:orderablation}a shows, the NLL of the correct and learned order is close and much smaller than for incorrect orders. For these incorrect orders, we  show in Figure~\ref{fig:orderablation}b how the NLL behaves as the divergence $\divtop$ of corrupted orders from the true order increases, showing it worsens monotonically.  

We also evaluate order discovery on a common standard real-world benchmark consisting of multiparameter single-cell data
 \cite{sachs:05:sachs} where a given causal DAG reflects causal cell signaling pathways among 11 variables. The main results (Figure \ref{fig:order} (b)) show that \ourmethod is the only method to discover causal order with $\divtop\approx0.21$ divergence. We show the divergence for the main reference dataset (condition \texttt{cd3cd28}) but besides this, multiple interventional datasets are available, for which \ourmethod consistently achieves topological divergences of $0.12-0.29$, with the exception of one condition where a perturbation on the hub node PKC was applied   (condition \texttt{g0076}). All results are shown in detail in Appendix \ref{app:additional-sachs-results}.

\paragraph{Predictive Performance (Trade Off II).} \pgfplotsset{
  rankplot/.style={
    width=\linewidth,
    height=3.5cm,
    ymin=1,
    ymax=8,
    y dir=reverse,
    xtick={0.1,0.2,0.3,0.4,0.5},
    xticklabel style={/pgf/number format/fixed},
    xlabel={Missing rate},
    smooth,
    ticklabel style={font=\scriptsize},
    label style={font=\small},
  },
}

\begin{figure}
    \centering

    \newcommand{\rankdir}{expres/imputation/rank_by_method_subset}
    
    \begin{subfigure}[t]{0.24\linewidth}
        \centering
        \begin{tikzpicture}
            \begin{axis}[
                rankplot,
                pretty line,
                pretty labelshift,
                cycle list name=topic-line,
                ylabel={Avg.~Rank},
                xlabel={Missing rate},
                legend style={
                    at={(0.8,1.7)},
                    anchor=north west,
                    font=\scriptsize,
                    legend columns=4,
                },
            ]
            \pgfplotsinvokeforeach{finetuned-foundation,hyperimpute,missforest,ice,gain,mice,softimpute,knn}{
                \addplot+
                table[
                    x=missing_rate,
                    y=imputation_rank,
                    col sep=comma,
                ]{\rankdir/CTR23_#1.csv};
            }
            \legend{\ourmethod, \hyperimpute, \missforest, \iceimp, \gain, \mice, \softimpute, \knnimp}
            \end{axis}
        \end{tikzpicture}
        \caption{CTR23: Imputation}
    \end{subfigure} 
    \begin{subfigure}[t]{0.24\linewidth}
        \centering
        \begin{tikzpicture}
            \begin{axis}[
                rankplot,
                pretty line,
                pretty labelshift,
                cycle list name=topic-line,
                xlabel={Missing rate},
            ]

            \pgfplotsinvokeforeach{finetuned-foundation,hyperimpute,missforest,ice,gain,mice,softimpute,knn}{
                \addplot+
                table[
                    x=missing_rate,
                    y=prediction_rank,
                    cycle list name=mamba,
                    col sep=comma,
                ]{\rankdir/CTR23_#1.csv};
            }

            \end{axis}
        \end{tikzpicture}
        \caption{CTR23: Prediction}
    \end{subfigure} 
    \begin{subfigure}[t]{0.24\linewidth}
        \centering
        \begin{tikzpicture}
            \begin{axis}[
                rankplot,
                pretty line,
                pretty labelshift,
                cycle list name=topic-line,
                ylabel={Avg.~Rank},
                xlabel={Missing rate},
                legend pos=north east,
            ]
            \pgfplotsinvokeforeach{finetuned-foundation,hyperimpute,missforest,ice,gain,mice,softimpute,knn}{%
                \addplot+
                table[
                    x=missing_rate,
                    y=imputation_rank,
                    cycle list name=mamba,
                    col sep=comma,
                ]{\rankdir/CC18_#1.csv};
            }
            \end{axis}
        \end{tikzpicture}
        \caption{CC18: Imputation}
    \end{subfigure} 
    \begin{subfigure}[t]{0.24\linewidth}
        \centering
        \begin{tikzpicture}
            \begin{axis}[
                rankplot,
                pretty line,
                pretty labelshift,
                cycle list name=topic-line,
                xlabel={Missing rate},
            ]

            \pgfplotsinvokeforeach{finetuned-foundation,hyperimpute,missforest,ice,gain,mice,softimpute,knn}{%
                \addplot+
                table[
                    x=missing_rate,
                    y=prediction_rank,
                    cycle list name=mamba,
                    col sep=comma,
                ]{\rankdir/CC18_#1.csv};
            }

            \end{axis}
        \end{tikzpicture}
        \caption{CC18: Prediction}
    \end{subfigure}
    \caption{\textbf{Predictive Performance  (Trade Off II)}. Shown is the average rank of imputation (a, c) and downstream prediction (b, d) performance on the CTR23 (left) and CC18 (right) datasets. \ourmethod
    is the most accurate imputation method for $\geq 40\%$ missingness.}
    \label{fig:imputation_results}
\end{figure}
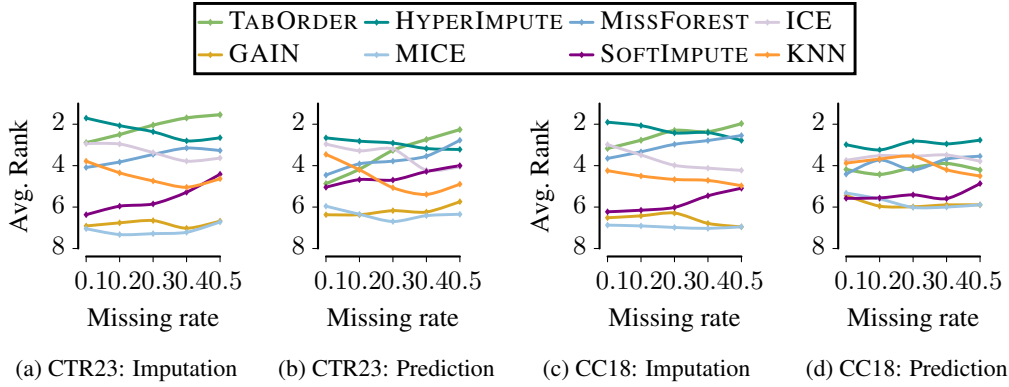
We now evaluate the predictive performance of \ourmethod on real-world data, measuring both imputation accuracy and downstream predictive performance. We use the OpenML CC18 classification and CTR23 regression benchmarks \citep{bischl2021openml,fischer2023openml}. We consider the datasets with at least 80\% real-valued inputs, yielding 33 CC18 and 23 CTR23 datasets, listed in Appendix~\ref{app:datasets}. We apply a split into 80\% train and 20\% test data as well as mask 10–50\% of entries uniformly at random. We compare against standard imputers \knnimp, \iceimp, \mice \citep{van2011mice}, \missforest \citep{stekhoven2012missforest}), neural methods (\gain \citep{yoon2018gain}), and \hyperimpute \citep{jarrett2022hyperimpute}; additional baselines are reported in Appendix~\ref{app:additional-imputation-results}. 
For the trained \ourmethod, we perform fine-tuning on the training data only. For downstream evaluation, we train an \xgboost model \citep{chen:16:xgboost} on imputed data (classification and regression), except that for CTR23 we directly impute missing targets with \ourmethod at test time.

 \setlength{\intextsep}{4pt}
 
\pgfplotstableread[col sep=tab]{figs/texdata/three_variable_intervention_mse_hard.tsv}\datatablehard
\pgfplotstableread[col sep=tab]{figs/texdata/three_variable_intervention_mse_mech_shift.tsv}\datatablemechshift
\begin{wrapfigure}{r}{0.46\textwidth}
	\vspace{-0.6em}
	\centering
	  
	  	\centering
	  	\begin{tikzpicture}
	  		\begin{axis}[
	  			pretty boxplot,
	  			width=\linewidth,
	  			pretty labelshift,
	  			height=3.5cm,
	  			ymax=1,
	  			ylabel={MSE},
	  			xtick={1,2,3,4,5,6},
	  			xticklabels={,i.i.d.,,,{~Intervened},,},
	  			legend columns=3,
	  			legend style={font=\scriptsize, at={(0.5,1.1  )}, anchor=south},
	  			]
	  			\addplot[{pr-color1a, fill=pr-color1a!40}] table[y=cfm_noniv] {\datatablemechshift};
	  			\addlegendentry{\ourmethod}
	  			\addplot[{pr-color1b, fill=pr-color1b!40}] table[y=xgboost_noniv] {\datatablemechshift};        
	  			\addlegendentry{XGBoost}
	  			\addplot[{pr-color1c, fill=pr-color1c!40}] table[y=tabpfn_noniv] {\datatablemechshift};
	  			\addlegendentry{\tabpfn}
	  			\addplot[{pr-color1a, fill=pr-color1a!40}] table[y=cfm_iv] {\datatablemechshift};
	  			\addplot[{pr-color1b, fill=pr-color1b!40}] table[y=xgboost_iv] {\datatablemechshift};
	  			\addplot[{pr-color1c, fill=pr-color1c!40}] table[y=tabpfn_iv] {\datatablemechshift};
	  		\end{axis}
	  	\end{tikzpicture}
	
	\caption{
		\textbf{Intervention-Robust Prediction.}
		Shown is the MSE for predicting $Y$ in the three-variable chain on an i.i.d.~and an intervened sample $Y \to Z$ (mechanism shift). \ourmethod remains accurate under the intervention via the learned causal order, while \tabpfn and \xgboost degrade in performance.
	}
	\label{fig:three-variable-intervention-mse}
	
	\vspace{-1em}
\end{wrapfigure}
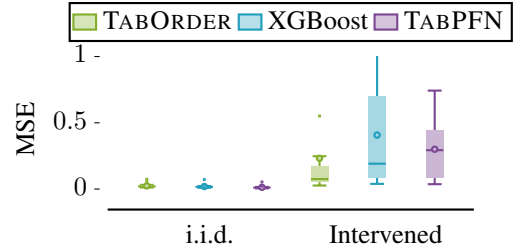
Figure~\ref{fig:imputation_results} reports average ranks across missingness levels. \ourmethod achieves the best imputation performance at higher missingness (CTR23 for $q \geq 30\%$, CC18 for $q \geq 40\%$),   \hyperimpute   at lower missingness. For downstream tasks, \ourmethod is competitive on CC18 and outperforming on CTR23 at $q \geq 40\%$, benefiting from directly imputing   values without retraining.
 
\begin{figure}[t]
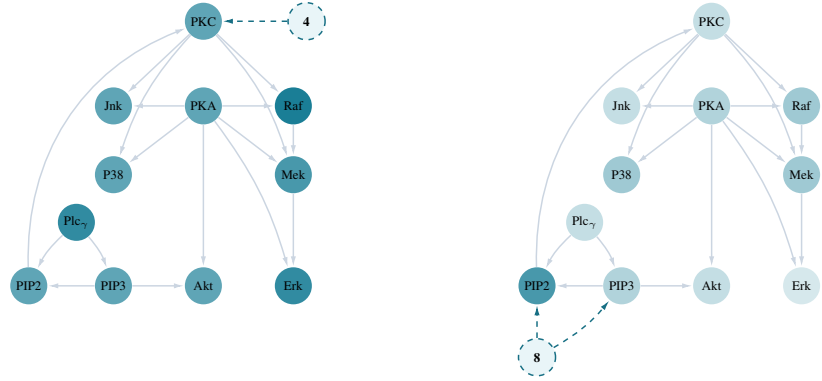
 
	\begin{minipage}[t]{ \linewidth}  
			\begin{subfigure}{0.5\linewidth} 
			\centering
			\ResetSachsInterventions
			\ShowSachsInhibitorFour
			\ShowSachsFlipsGzero
			\scalebox{0.7}{\SachsFlipGraphSmall{\texttt{cd3cd28\_g0076}}{0.765}}
			\caption{Interventional Condition  \textbf{\texttt{cd3cd28\_g0076}.} }
		\end{subfigure} 
		\begin{subfigure}{0.45\linewidth} 
			\centering
			\ResetSachsInterventions
			\ShowSachsInhibitorEight
			\ShowSachsFlipsPsitect
			\scalebox{0.7}{\SachsFlipGraphSmall{\texttt{cd3cd28\_psitect}}{0.235}}
			\caption{Interventional Condition \textbf{\texttt{cd3cd28\_psitect}.} }
		\end{subfigure}
		\caption{\textbf{Changes in Inferred Orders under Intervention.} In the \cite{sachs:05:sachs} data, we compare two interventional conditions (a, b) to the reference condition, and count flips in node positions in the order  by \ourmethod (darker for more flips). An upstream intervention (a) leads to global changes in the order; a downstream one (b) to a local change of one target.  
		}
		\label{fig:sachs-intervention-graph-flips-small}
	\end{minipage} 
\end{figure}

\paragraph{Intervention-Robust Prediction.} We continue investigating the predictive performance but move to an interventional setting. We revisit the chain SCM $X \to Y \to Z$ introduced in Figure~\ref{fig:three-variable-intervention-preds} with focus on predicting the mediator $Y$ given  $X$ and $Z$. We  sample observations   from an SCM with nonlinear mechanisms and on the test set,  
change the mechanism of $Z$ for $50\%$ of the samples.  
Figure~\ref{fig:three-variable-intervention-mse} shows that while all methods perform well on  datapoints drawn from the same distribution (i.i.d.),   
 \tabpfn and \xgboost suffer under the mechanism  change  (Intervened)  to which \ourmethod  remains  reasonably robust.   
 This provides empirical evidence for our motivating idea that using causal order use is beneficial for intervention-robust  predictive performance. 

  \paragraph{Unknown Biological Interventions.} Last, to  study the effects of interventions outside the synthetic setting, we revisit   the real-world data by \cite{sachs:05:sachs}. We take the causal order inferred for the near-observational dataset (\texttt{cd3cd28}) as a reference, and compare it to the orders inferred by \ourmethod for two interventional datasets (\texttt{g0076}, \texttt{psitect}). To allow for a qualitative analysis that shows how the position of each of the nodes in the order changes, we use a simple metric over pairwise position flips of the nodes (Appendix \ref{app:additional-sachs-results}). In Figure \ref{fig:sachs-intervention-graph-flips-small}, we depict the 11 observed variables (solid nodes), the fixed consensus graph (edges), the biological background knowledge on intervention effects (dashed nodes) and the position flips that occurred for each node (color, dark blue indicates more position flips compared to  \texttt{cd3cd28}). The results contrast the effect of an intervention on an upstream hub node PKC (a), which leads  to changes in the order for almost all nodes, against the effect of a downstream  intervention on  PIP2 and PIP3 (b), which leads to a localized change at the  intervention target while leaving the remaining system unaffected. Analyzing such perturbations  in biological systems can be useful given that intervention effects are often not  a priori known.

\section{Conclusion} 

Towards addressing the disconnect between associative in-context learning and causal structure learning, we integrate causal order learning within a transformer-based framework for tabular data. Our architecture \ourmethod jointly infers and uses causal orderings to constrain predictive attention. 
While the model is not limited to specific functional forms during training or inference, causal structure identification is not generally possible from observational data without further assumptions, such that we provided theoretical justification for our approach within the class of additive noise models 
 Our experiments cover both causal order discovery and predictive tasks, where \ourmethod achieved reasonable performance across both dimensions. On  real-world  data, the   inferred topological orderings  align very closely with domain knowledge, and the discovered orders across interventional conditions give insights into the localized effects of the  perturbations.  

Regarding current limitations, the  implementation focuses primarily on learning topological orderings rather than full directed acyclic graphs as we consider this sufficient for the predictive and interventional tasks under consideration. The refinement of these orders into sparse graph structures is well explored using pruning via sparse regression, but future work could address this aspect more directly within the transformer architecture. Furthermore, investigating broader training priors and systematically studying their effect on amortized learners is not the focus of this work, but an important consideration for improving the quality of pretraining. In this context, the use of explicitly interventional data during training is likely particularly promising given the aim to learn causal structure, although this requires access to richer training data where such interventions are present. These research efforts could better connect the areas of structure learning and predictive modeling toward more interpretable and intervention-robust in-context learning.

\bibliographystyle{plainnat}
\bibliography{bib/paper}

\newpage
\appendix
\onecolumn

\appendix

\section{Detailed Theoretical Results}
 
 \subsection{Background and Assumptions} \label{app:background}

 Our analysis builds on a set of standard assumptions commonly used in causal structure learning. We give a shortened overview below and refer to, e.g., \cite{pearl:09:causality} for a detailed exposition. 
 
 \textbf{(i) Causal Markov Condition.}
 The joint distribution factorizes according to the true causal DAG $G^\star$, i.e.,
 each variable $X_i$ is conditionally independent of its non-descendants given its parents $\pa_i$.
 
 \textbf{(ii) Faithfulness.}
 All conditional independencies in the data distribution are entailed by the causal DAG.
 That is, there are no independencies arising from exact cancellations of parameters.
 
 \textbf{(iii) Causal Sufficiency.}
 There are no unobserved confounders: all common causes of observed variables are included in the model.
 
 \textbf{(iv) Additive Noise Model (ANM).}
 Data are generated according to a structural causal model of the form
 \begin{equation}
 X_i = \sum_{j \in \pa_i} f_{ij}(X_j) + N_i,
 \end{equation}
 where the noise variables $N_i$ are jointly independent and the functions $f_{ij}$ are nonlinear and sufficiently smooth.
 
 \medskip
 Under these assumptions, causal directions become identifiable in many settings
 \citep{shimizu:06:lingam,hoyer:08:anm,zhang:09:pnl,buhlmann:14:cam},
 and likelihood- or variance-based scores are consistent for recovering valid
 topological orders of the underlying DAG.
 \subsection{Identifiability of Causal Orders}\label{app:order-identifiability}
Identifying the causal graph and thereby recovering a valid topological
ordering is feasible for certain classes of functional mechanisms \citep{shimizu:06:lingam,hoyer:08:anm,zhang:09:pnl}.
Because our learning objective is based on minimizing conditional error variance, we focus on \textbf{causal
	additive models} \citep{buhlmann:14:cam}, where
\begin{equation}
	X_i = \sum_{j \in \pa_i} f_{ij}(X_j) + N_i,
	\label{eq:cam-model}
\end{equation}
with smooth functions $f_{ij}$ and independent noise.
\citet{buhlmann:14:cam} show that likelihood-/residual-variance–based scores
consistently recover orders that respect the data-generating causal DAG.  
A simplified version of their main result reads as follows.
\begin{theorem}[Order Consistency in Causal Additive Models; \citealp{buhlmann:14:cam}]
	\label{thm:cam-simplified}
	In a causal additive model with non-linear functions and independent, non-zero noise,
	let $\hat{\pi}$ be the order minimizing residual variance.  
	Then,
	\begin{equation}
		\mathbb{P}(\hat{\pi} \in \Pi_0) \to 1 \quad \text{as } n \to \infty,
	\end{equation}
	where $\Pi_0$ denotes the set of all topological orders of the true causal DAG.
	Any order violating a causal edge incurs a strictly larger expected
	residual variance.
\end{theorem}

For this particular class of ANMs,
likelihood-optimal orders are consistent with the topological orders of
the causal DAG.  This property is key for \ourmethod{}: maximizing the likelihood
naturally guides the optimization toward orders that match
the underlying causal structure, because conditioning on the true parents
yields the overall lowest-variance conditionals.  
In this way, the likelihood itself provides a causal inductive bias,
steering the learned order toward compatibility with the generative process,
without requiring any explicit supervision or information about the causal graph.


\subsection{Formal Proof of Theorem \ref{thm:snr-amplification}}
\label{app:missingness-proof}
\theoremmissing*
\begin{proof}
Fix a direction $\tau$ (either forward or backward).
Let $v_{0,\tau}$ be the residual variance when the parent is observed ($M=0$)
and $v_{1,\tau}$ when the parent is missing ($M=1$), with $\Pr(M=1)=q$.

We first consider the log-likelihood under fixed variance versus missingness-aware variance.
An optimal Gaussian conditional model that uses a single variance parameter
must match the mask-averaged residual variance
\begin{equation}
\bar{\sigma}_{\tau}^2 = (1-q)\,v_{0,\tau} + q\,v_{1,\tau},
\end{equation}
hence achieves population log-likelihood
\begin{equation}
\mathcal{L}_{\tau,\mathrm{fix}}^\star = -\tfrac12 \log \bar{\sigma}_{\tau}^2 + C,
\end{equation}
for a constant $C$ independent of $\tau$.
A model that distinguishes the two regimes fits them exactly and achieves
\begin{equation}
\mathcal{L}_{\tau,\mathrm{inc}}^\star
= -\tfrac12\bigl[(1-q)\log v_{0,\tau} + q\log v_{1,\tau}\bigr] + C.
\end{equation}

\paragraph{Likelihood Gain.}
We now characterize the improvement in log-likelihood when using missingness-aware variance.
We define the (nonnegative) gain as
\begin{equation}
G_{\tau} := \mathcal{L}_{\tau,\mathrm{inc}}^\star - \mathcal{L}_{\tau,\mathrm{fix}}^\star.
\end{equation}
By substituting the expressions above and canceling $C$, we obtain for the gain
\begin{align*}
G_{\tau}
&= -\tfrac12\bigl[(1-q)\log v_{0,\tau} + q\log v_{1,\tau}\bigr]
    +\tfrac12 \log\bigl((1-q)\,v_{0,\tau} + q\,v_{1,\tau}\bigr).
\end{align*}
Our objective is now to express the gain in terms of the variance ratio $R_{\tau} := v_{1,\tau}/v_{0,\tau}$ and the missingness rate $q$.
We factor $v_{0,\tau}$ out of the logarithm as 
\begin{equation}
(1-q)\,v_{0,\tau} + q\,v_{1,\tau}
= v_{0,\tau}\bigl((1-q) + q\,R_{\tau}\bigr).
\end{equation}
Therefore, we substitute for the logarithm of the single variance model
\begin{equation}
\log\bigl((1-q)\,v_{0,\tau} + q\,v_{1,\tau}\bigr)
= \log v_{0,\tau} + \log\bigl((1-q)+qR_{\tau}\bigr).
\end{equation}
By expanding $\log v_{1,\tau} = \log(v_{0,\tau}R_{\tau}) = \log v_{0,\tau} + \log R_{\tau}$ and plugging both into
$G_{\tau}$ yields
\begin{align}
G_{\tau}
&= -\tfrac12\Bigl[(1-q)\log v_{0,\tau} + q\bigl(\log v_{0,\tau} + \log R_{\tau}\bigr)\Bigr]
   +\tfrac12\Bigl[\log v_{0,\tau} + \log\bigl((1-q)+qR_{\tau}\bigr)\Bigr] \\
&= -\tfrac12\Bigl[\bigl((1-q)+q\bigr)\log v_{0,\tau} + q\log R_{\tau}\Bigr]
   +\tfrac12\Bigl[\log v_{0,\tau} + \log\bigl((1-q)+qR_{\tau}\bigr)\Bigr] \\
&= -\tfrac12\Bigl[\log v_{0,\tau} + q\log R_{\tau}\Bigr]
   +\tfrac12\Bigl[\log v_{0,\tau} + \log\bigl((1-q)+qR_{\tau}\bigr)\Bigr] \\
&= \tfrac12\Bigl[\log\bigl((1-q)+qR_{\tau}\bigr) - q\log R_{\tau}\Bigr].
\end{align}
Thus $G_{\tau}$ depends on $(v_{0,\tau},v_{1,\tau})$ only through the ratio $R_{\tau}$ and the missingness rate $q$:
\begin{equation}
\label{eq:gain}
G_{\tau} = G(R_{\tau},q)
:= \frac12\Bigl[\log\bigl((1-q)+qR_{\tau}\bigr) - q\log R_{\tau}\Bigr].
\end{equation}

We now differentiate $G(R,q)$ with respect to $R$ to establish monotonicity under masking rates of $q \in (0,1)$
\begin{equation}
\frac{\partial G}{\partial R}
= \frac12\left(\frac{q}{(1-q)+qR} - \frac{q}{R}\right)
= \frac{q(1-q)(R-1)}{2R\bigl((1-q)+qR\bigr)}.
\end{equation}
If the ratio $R>1$, then all terms in the numerator and denominator are positive,
and therefore the derivative is positive, so $G(R,q)$ is strictly increasing in $R$ on $(1,\infty)$, i.e.
\begin{equation}
    \boxed{R_1 > R_2 > 1 \implies G(R_1,q) > G(R_2,q).}
\end{equation}
The case when $R=1$ corresponds to no gain, since then missingness does not affect the residual variance.
This is only achievable in pathological cases (e.g., constant functions). We have now established general 
conditions under which the gain in one direction exceeds that in the other.
We now move to provide explicit conditions in the ANM case.

\paragraph{Gain in Additive Noise Models.}
Let $\Delta_{\mathrm{fix}} := \mathcal{L}^\star_{\mathrm{fwd,fix}} - \mathcal{L}^\star_{\mathrm{bwd,fix}}$
and $\Delta_{\mathrm{inc}} := \mathcal{L}^\star_{\mathrm{fwd,inc}} - \mathcal{L}^\star_{\mathrm{bwd,inc}}$.
Then
\begin{equation}
\Delta_{\mathrm{inc}} - \Delta_{\mathrm{fix}}
= \bigl(\mathcal{L}^\star_{\mathrm{fwd,inc}}-\mathcal{L}^\star_{\mathrm{fwd,fix}}\bigr)
 - \bigl(\mathcal{L}^\star_{\mathrm{bwd,inc}}-\mathcal{L}^\star_{\mathrm{bwd,fix}}\bigr)
= G(R_{\mathrm{fwd}},q)-G(R_{\mathrm{bwd}},q).
\end{equation}
By strict monotonicity of $G(\cdot,q)$ on $R>1$,
\begin{equation}
\Delta_{\mathrm{inc}} > \Delta_{\mathrm{fix}}
\quad\Longleftrightarrow\quad
R_{\mathrm{fwd}} > R_{\mathrm{bwd}}.
\label{eq:delta-ratio}
\end{equation}
It remains to derive specific terms for $R_{\mathrm{fwd}}$ and $R_{\mathrm{bwd}}$ in ANMs.

\textbf{Forward direction ($X\to Y$)}: if $X$ is observed, the optimal residual is $N$, so
$v_{0,\mathrm{fwd}}=\Var(N)$. If $X$ is missing, the optimal predictor is the mean of $Y$,
so $v_{1,\mathrm{fwd}}=\Var(Y)=\Var(f(X))+\Var(N)$ (since $N\indep X$).
Thus
\begin{equation}
R_{\mathrm{fwd}}
= \frac{v_{1,\mathrm{fwd}}}{v_{0,\mathrm{fwd}}}
= \frac{\Var(f(X))+\Var(N)}{\Var(N)}
= 1+\frac{\Var(f(X))}{\Var(N)}.
\end{equation}

\textbf{Backward direction ($Y\to X$)}: if $Y$ is observed, the optimal residual variance is
$v_{0,\mathrm{bwd}}=\E[\Var(X\mid Y)]$. If $Y$ is missing, the best predictor is the mean of $X$,
so $v_{1,\mathrm{bwd}}=\Var(X)=\E[\Var(X\mid Y)]+\Var(\E[X\mid Y])$. Hence
\begin{equation}
R_{\mathrm{bwd}}
= \frac{v_{1,\mathrm{bwd}}}{v_{0,\mathrm{bwd}}}
= \frac{\E[\Var(X\mid Y)]+\Var(\E[X\mid Y])}{\E[\Var(X\mid Y)]} = 1 + \frac{\Var(\E[X\mid Y])}{\E[\Var(X\mid Y)]}.
\end{equation}
Therefore $R_{\mathrm{fwd}} > R_{\mathrm{bwd}}$ is equivalent to
\begin{equation}
1+\frac{\Var(f(X))}{\Var(N)} > 1 + \frac{\Var(\E[X\mid Y])}{\E[\Var(X\mid Y)]}
\quad\Longleftrightarrow\quad
\frac{\Var(f(X))}{\Var(N)} > \frac{\Var(\E[X\mid Y])}{\E[\Var(X\mid Y)]}.
\end{equation}
Using Eq.~\eqref{eq:delta-ratio}, we can thus conclude that the likelihood gap is amplified
iff the forward signal-to-noise ratio exceeds the backward ratio, which proves the claim.
\end{proof}

 \section{Model and Training Details} 
\label{app:model-details}
This section summarizes the model components and hyperparameters used in the experiments.
 
We implement the attention mask as an additive bias to the attention scores before applying the softmax.
Forward, we use $\attention_{ij}\cdot \beta$, where $\beta$ is annealed from $-5.0$ to $-20.0$ during training; backward, we use the sigmoid form $\hat{\attention}_{ij}=\sigma((s_i - s_j)/t)$
where $t$ is annealed from $1.0$ to $0.1$.
\begin{itemize}
    \item We embed every cell to a vector of dimension $d_{\text{embedding}}=128$.
    \item Transformer blocks: $6$ repetitions for inducing order and $4$ repetitions for cell prediction. We use Pytorch's \texttt{nn.TransformerEncoderLayer} with GELU activations,
    feed-forward dimension of $2 \cdot d_{\text{embedding}}$, dropout $0.1$ and prenorm.
    \item Optimization: Adam with base learning rate $2\times10^{-4}$, warmup ratio $0.03$ and weight decay.
    \item Batch size $4$ with 25000 training steps for a total of 100,000 datasets seen. Datasets are generated with $5-10$ variables each and $512-1024$ samples.
    \item Masking during training: random masking with \texttt{masking\_percentage}=0.2.
\end{itemize}

  \section{Detailed Experimental Setups}
  \newcommand{\orderhat}{\hatorder}
  \newcommand{\cond}{c}
  \newcommand{\obs}{{c_0}}
  \newcommand{\orderc}[1]{\orderhat_{#1}}
  \newcommand{\ordercond}{\orderhat_{\cond}}
  \newcommand{\orderobs}{\orderhat_{\obs}} 
  \newcommand{\Gobs}{\mathit{G}_{\obs}} 
  
\subsection{Synthetic Data Generation}
\label{app:synthetic-data}


For each synthetic example, we sample the number of variables
$\nfeat \sim \mathrm{Unif_{Int}}(5,10)$ and construct a directed acyclic graph (DAG) $G=(V,E)$ over $V=\{1,\dots,\nfeat\}$ by first drawing a random topological ordering and then sampling parent sets subject to an in-degree constraint
$
|\pa_i| \in \{1,2,3\}.
$

Given $G$, we generate observations according to a nonlinear structural causal model (SCM). For each node $i=1,\ldots,\nfeat$, variables are generated as
\[
\featvar_i \;=\; f_i(\featvar_{\pa_i}) + \varepsilon_i,
\]
where the noise terms are mutually independent with
$\sigma_i \sim \mathrm{Unif}(0.05,0.25)$ and
$\varepsilon_i \sim \mathcal{N}(0,\sigma_i^2)$.

For root nodes ($\pa_i=\emptyset$), we sample $X_i$ from a mixture of simple base distributions: with probability $1/2$, a scaled Gaussian $X_i = s Z$ where $Z\sim\mathcal{N}(0,1)$ and $s\sim\mathrm{Unif}(0.5,2.0)$; otherwise, a scaled uniform $X_i=(U-\tfrac12)a$ where $U\sim\mathrm{Unif}(0,1)$ and $a\sim\mathrm{Unif}(1.0,5.0)$. Values are lightly clipped to avoid extreme outliers.

In the non-additive setting, each structural function $f_i$ is a smooth nonlinear function of all parents jointly, sampled from an RBF Gaussian process prior approximated by random Fourier features (RFF). Let $x\in\mathbb{R}^{P}$ denote the standardized parent vector, where $P=|\pa_i|$. We sample a lengthscale
 $ \ell \sim \exp\!\big(\mathrm{Unif}(\log 0.3,\log 1.0)\big),$
and define
\[
\phi(x) \;=\; \sqrt{\frac{2}{H}}\cos\!\big(W^\top x + b\big),
\qquad
W\in\mathbb{R}^{P\times H},\; b\in\mathbb{R}^{H},
\]
with $H=256$ features,
$W_{jk}\sim\mathcal{N}(0,\ell^{-2})$, and $b_k\sim\mathrm{Unif}(0,2\pi)$. With weights $a\sim\mathcal{N}(0,I_H)$, we set
$
f_i(x) \;=\; a^\top \phi(x).$
This yields smooth nonlinear mechanisms that are explicitly non-additive across parents.

Finally, for each DAG we draw batches with sequence length $r \sim \mathrm{Unif_{Int}}(512,1024)$ from the resulting SCM.

\subsection{Real-World Datasets}
\label{app:datasets}

\paragraph{Sachs Datasets.} 
We evaluate our method on the protein signaling dataset of \cite{sachs:05:sachs}. The data consist of single-cell measurements of intracellular signaling proteins in primary human CD4$^{+}$ T cells, obtained via multiparameter flow cytometry.

The dataset contains measurements of 11 phosphoproteins and phospholipids (\texttt{Raf}, \texttt{Mek}, \texttt{Erk}, \texttt{Plcg}, \texttt{PIP2}, \texttt{PIP3}, \texttt{PKC}, \texttt{PKA}, \texttt{Akt}, \texttt{Jnk}, \texttt{P38}) under multiple experimental conditions, 
where as is common we disregard the condition \texttt{b2camp} and treat the condition \texttt{cd3cd28} as the main reference condition, shown in the main experiments (Figure \ref{fig:order}b). The conditions correspond to different combinations of external stimulation (e.g., TCR/CD28 activation and ICAM-2 co-stimulation) 
and inhibition (e.g., of MEK, PKC, PI3K, and Akt, as well as perturbations of phosphoinositide signaling).

For the evaluations, we use a reduced Sachs signaling graph over the measured variables    consistent with commonly used benchmark variants in prior causal discovery work \citep[e.g.,][]{zheng:18:notears}. The used network closely follows the validated Sachs benchmark network distributed within the  \texttt{bnlearn} R package, replacing the   edges Erk $\to$ Akt and PKC$\to$ PKA with the biologically established reduced-pathway edges PIP3 $\to$ Akt and PIP2 $\to$PKC, both of which  \cite{sachs:05:sachs} identify as  signaling influences. We show this graph in Figure  \ref{fig:sachs-consensuses}, depicting the consensus graph we use in the experiments (a) and the activating resp. inhibiting perturbations (b). We show the correspondence of these interventions with the available experimental conditions in Table \ref{tab:sachs-consensus-interventions}.

\begin{figure}[h!]
		\begin{subfigure}{.5\textwidth}
		\centering
		\scalebox{0.7}{
			\begin{tikzpicture}[
				prot/.style={draw=tcss-slate3, circle, minimum size=8mm, inner sep=1pt},
				iv/.style={
					draw=white,
					fill=white,white,
					circle,
					dashed,
					thick,
					minimum size=7mm,
					inner sep=1pt,
					font=\bfseries
				},
				>=Latex
				]
				
				\node[prot] (pkc) at (0,4) {PKC};
				
				\node[prot] (pka) at (0,2.4) {PKA};
				\node[prot] (raf) at (1.7,2.4) {Raf};
				
				\node[prot] (jnk) at (-1.7,2.4) {Jnk};
				\node[prot] (p38) at (-1.7,1.1) {P38};
				
				\node[prot] (plcg) at (-2.4,0.2) {Plc${}_{\gamma}$};
				
				\node[prot] (pip2) at (-3.3,-1.0) {PIP2};
				\node[prot] (pip3) at (-1.7,-1.0) {PIP3};
				\node[prot] (akt) at (0,-1.0) {Akt};
				\node[prot] (erk) at (1.7,-1.0) {Erk};
				
				\node[prot] (mek) at (1.7,1.1) {Mek};
				
				\path[-{Latex[length=2mm,width=1mm]}, line width=.9pt, tcss-slate3]
				(raf) edge[] (mek)
				(mek) edge[] (erk)
				(plcg) edge[bend right=10] (pip2)
				(plcg) edge[bend left=10] (pip3)
				(pip3) edge[] (pip2)
				(pip3) edge[] (akt)
				(pip2) edge[bend left=35] (pkc)
				(pkc) edge[bend left=12] (raf)
				(pkc) edge[] (mek)
				(pkc) edge[bend right=12] (jnk)
				(pkc) edge[] (p38)
				(pka) edge[] (raf)
				(pka) edge[] (mek)
				(pka) edge[] (erk)
				(pka) edge[] (akt)
				(pka) edge[] (jnk)
				(pka) edge[] (p38)
				;
				
				\node[iv] (i1) at (-1.45,5.45) {1};
				\node[iv] (i2) at (1.45,5.45) {2};
				\node[iv] (i3) at (-1.55,4.15) {3}; 
				\node[iv] (i4) at (1.95,4.00) {4};
				\node[iv] (i5) at (0,-2.35) {5};
				\node[iv] (i6) at (3.25,1.1) {6};
				\node[iv] (i7) at (-1.7,-2.35) {7};
				\node[iv] (i8) at (-3.30,-2.35) {8};
				
			\end{tikzpicture}
		} 
		\caption{\textbf{Causal DAG.}}
		\label{fig:sachs-consensus-graph}
	\end{subfigure}
\begin{subfigure}{.5\textwidth} 
	\centering
\scalebox{0.7}{
	\begin{tikzpicture}[
		prot/.style={draw=tcss-slate3, circle, minimum size=8mm, inner sep=1pt},
		act/.style={draw=pr-color1a!60!black, fill=pr-color1a!10, circle, dashed, thick,
			minimum size=7mm, inner sep=1pt, font=\bfseries},
		inh/.style={draw=pr-color1b!70!black, fill=pr-color1b!10, circle, dashed, thick,
			minimum size=7mm, inner sep=1pt, font=\bfseries},
		>=Latex
		]
		 
		\node[prot] (pkc) at (0,4) {PKC};
		
		\node[prot] (pka) at (0,2.4) {PKA};
		\node[prot] (raf) at (1.7,2.4) {Raf};
		
		\node[prot] (jnk) at (-1.7,2.4) {Jnk};
		\node[prot] (p38) at (-1.7,1.1) {P38};
		
		\node[prot] (plcg) at (-2.4,0.2) {Plc${}_{\gamma}$};
		
		\node[prot] (pip2) at (-3.3,-1.0) {PIP2};
		\node[prot] (pip3) at (-1.7,-1.0) {PIP3};
		\node[prot] (akt) at (0,-1.0) {Akt};
		\node[prot] (erk) at (1.7,-1.0) {Erk};
		
		\node[prot] (mek) at (1.7,1.1) {Mek};
		 
		\path[-{Latex[length=2mm,width=1mm]}, line width=.9pt, tcss-slate3]
		(raf) edge[] (mek)
		(mek) edge[] (erk)
		(plcg) edge[bend right=10] (pip2)
		(plcg) edge[bend left=10] (pip3)
		(pip3) edge[] (pip2)
		(pip3) edge[] (akt)
		(pip2) edge[bend left=35] (pkc)
		(pkc) edge[bend left=12] (raf)
		(pkc) edge[] (mek)
		(pkc) edge[bend right=12] (jnk)
		(pkc) edge[] (p38)
		(pka) edge[] (raf)
		(pka) edge[] (mek)
		(pka) edge[] (erk)
		(pka) edge[] (akt)
		(pka) edge[] (jnk)
		(pka) edge[] (p38)
		;
		 
		\node[act] (i1) at (-1.45,5.45) {1};
		\node[act] (i2) at (1.45,5.45) {2};
		\node[act] (i3) at (-1.55,4.15) {3};
		\node[inh] (i4) at (1.95,4.00) {4};
		\node[inh] (i5) at (0,-2.35) {5};
		\node[inh] (i6) at (3.25,1.1) {6};
		\node[inh] (i7) at (-1.7,-2.35) {7};
		\node[inh] (i8) at (-3.30,-2.35) {8};
		 
		\path[-{Latex[length=2mm,width=1mm]}, dashed, line width=0.8pt, pr-color1a!60!black]
		(i1) edge[bend right=10] (pkc)
		(i2) edge[bend left=10] (pkc)
		(i3) edge[] (pka);
		
		\path[-{Latex[length=2mm,width=1mm]}, dashed, line width=0.8pt, pr-color1b!70!black]
		(i4) edge[] (pkc)
		(i5) edge[] (akt)
		(i6) edge[] (mek)
		(i7) edge[] (pip3)
		(i7) edge[bend right=10] (akt)
		(i8) edge[] (pip2)
		(i8) edge[bend right=10] (pip3);
		
	\end{tikzpicture}
}
\caption{\textbf{Intervention Targets.}}
\label{fig:sachs-consensus-interventions}
\end{subfigure} 
\caption{\textbf{Consensus Causal Structure for the Data by \cite{sachs:05:sachs}.} We show the causal DAG that we consider in the evaluations (a), as well as the effects of perturbations applied in each experimental condition (b). Dashed colored nodes correspond to activators (green) and inhibitors (blue), cf. Table \ref{tab:sachs-consensus-interventions}. }
\label{fig:sachs-consensuses}
\end{figure}
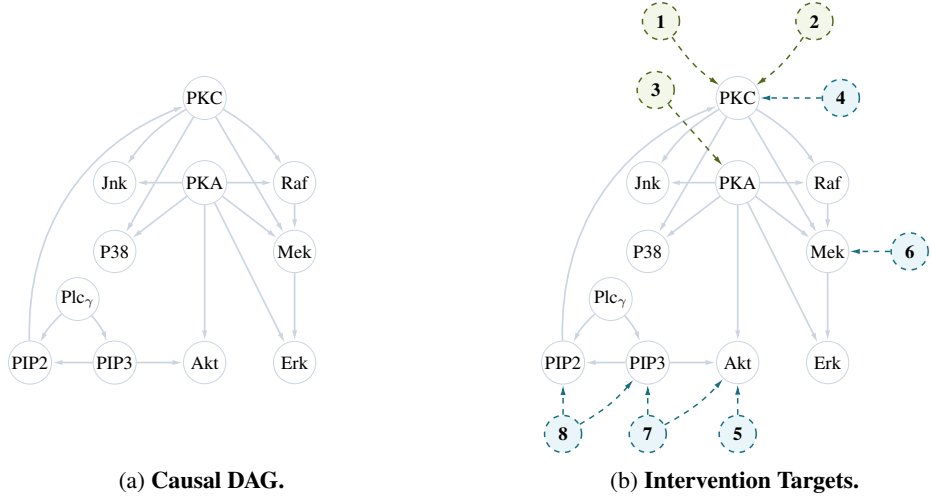
\begin{table}[h!]
	\caption{\textbf{Interventional Conditions in the Data by \cite{sachs:05:sachs}.} Listed is each interventional condition and the type of perturbation applied therein. 
	}
	\label{tab:sachs-consensus-interventions}
	\vspace*{1em}
	\centering
	\normalsize
	\begin{tabular}{llll}
		\toprule
		No. & Type & Condition & Intervention\\
		\midrule
		\textcolor{pr-color1a}{1} & Activator & \texttt{cd3cd28} & TCR/CD28 stimulation \\
		\textcolor{pr-color1a}{2}  & Activator & \texttt{cd3cd28\_icam2} & TCR/CD28 + ICAM-2 \\
		\textcolor{pr-color1a}{3} & Activator & \texttt{pma} & PKA/PKC-side activation proxy \\
		\textcolor{pr-color1b}{4} & Inhibitor & \texttt{cd3cd28\_g0076} & PKC inhibition \\
		\textcolor{pr-color1b}{5} & Inhibitor & \texttt{cd3cd28\_aktinhib} & Akt inhibition \\
		\textcolor{pr-color1b}{6} & Inhibitor & \texttt{cd3cd28\_u0126} & MEK inhibition \\
		\textcolor{pr-color1b}{7} & Inhibitor & \texttt{cd3cd28\_ly} & PI3K/PIP3/Akt-axis inhibition \\
		\textcolor{pr-color1b}{8} & Inhibitor & \texttt{cd3cd28\_psitect} & PIP2/PIP3-axis perturbation \\
		\bottomrule
	\end{tabular}
\end{table} 

\paragraph{CC18 and CTR23.} For reference,   we include statistics on the real-world datasets CC18 and CTR23 used in Figure  \ref{fig:imputation_results} 
in Table \ref{tab:cc18-ctr23-datasets}.
\begin{longtable}{llrrl}
	\caption{\textbf{Real-World Datasets for Predictive Tasks.} CC18 and CTR23 datasets used in our experiments.}
	\label{tab:cc18-ctr23-datasets} \\
	\toprule
	suite & dataset & rows & cols & target\_name \\
	\midrule
	\endfirsthead
	
	\toprule
	suite & dataset & rows & cols & target\_name \\
	\midrule
	\endhead
	
	\midrule
	\multicolumn{5}{r}{Continued on next page} \\
	\midrule
	\endfoot
	
	\bottomrule
	\endlastfoot
CC18 & GesturePhaseSegmentationProcessed & 9873 & 33 & Phase \\
CC18 & MiceProtein & 552 & 78 & class \\
CC18 & analcatdata\_authorship & 841 & 71 & Author \\
CC18 & balance-scale & 625 & 5 & class \\
CC18 & banknote-authentication & 1372 & 5 & Class \\
CC18 & blood-transfusion-service-center & 748 & 5 & Class \\
CC18 & breast-w & 683 & 10 & Class \\
CC18 & churn & 5000 & 17 & class \\
CC18 & climate-model-simulation-crashes & 540 & 19 & outcome \\
CC18 & diabetes & 768 & 9 & class \\
CC18 & electricity & 45312 & 8 & class \\
CC18 & first-order-theorem-proving & 6118 & 52 & Class \\
CC18 & ilpd & 583 & 10 & Class \\
CC18 & jm1 & 10880 & 22 & defects \\
CC18 & jungle\_chess\_2pcs\_raw\_endgame\_complete & 44819 & 7 & class \\
CC18 & kc1 & 2109 & 22 & defects \\
CC18 & kc2 & 522 & 22 & problems \\
CC18 & mfeat-morphological & 2000 & 7 & class \\
CC18 & ozone-level-8hr & 2534 & 73 & Class \\
CC18 & pc1 & 1109 & 22 & defects \\
CC18 & pc3 & 1563 & 38 & c \\
CC18 & pc4 & 1458 & 38 & c \\
CC18 & phoneme & 5404 & 6 & Class \\
CC18 & qsar-biodeg & 1055 & 42 & Class \\
CC18 & satimage & 6430 & 37 & class \\
CC18 & segment & 2310 & 17 & class \\
CC18 & spambase & 4601 & 58 & class \\
CC18 & steel-plates-fault & 1941 & 28 & target \\
CC18 & vehicle & 846 & 19 & Class \\
CC18 & vowel & 990 & 11 & Class \\
CC18 & wall-robot-navigation & 5456 & 25 & Class \\
CC18 & wdbc & 569 & 31 & Class \\
CC18 & wilt & 4839 & 6 & class \\
CTR23 & QSAR\_fish\_toxicity & 908 & 7 & LC50 \\
CTR23 & abalone & 4177 & 8 & rings \\
CTR23 & airfoil\_self\_noise & 1503 & 6 & sound\_pressure \\
CTR23 & california\_housing & 20640 & 9 & medianHouseValue \\
CTR23 & cars & 804 & 18 & Price \\
CTR23 & concrete\_compressive\_strength & 1030 & 9 & strength \\
CTR23 & cpu\_activity & 8192 & 22 & usr \\
CTR23 & energy\_efficiency & 768 & 9 & heating\_load \\
CTR23 & fifa & 19178 & 28 & wage\_eur \\
CTR23 & forest\_fires & 517 & 11 & area \\
CTR23 & geographical\_origin\_of\_music & 1059 & 117 & latitude \\
CTR23 & grid\_stability & 10000 & 13 & stab \\
CTR23 & kin8nm & 8192 & 9 & y \\
CTR23 & kings\_county & 21613 & 18 & price \\
CTR23 & miami\_housing & 13932 & 16 & SALE\_PRC \\
CTR23 & naval\_propulsion\_plant & 11934 & 15 & gt\_compressor\_decay\_  \\
 &  &   &   & state\_coefficient \\
CTR23 & physiochemical\_protein & 45730 & 10 & RMSD \\
CTR23 & pumadyn32nh & 8192 & 33 & thetadd6 \\
CTR23 & red\_wine & 1599 & 12 & quality \\
CTR23 & sarcos & 48933 & 22 & V22 \\
CTR23 & space\_ga & 3107 & 7 & ln\_votes\_pop \\
CTR23 & superconductivity & 21263 & 82 & critical\_temp \\
CTR23 & white\_wine & 4898 & 12 & quality \\ 
\end{longtable}

\section{Extended Experiments} 
\subsection{Causal Order Inference on Synthetic Data (Fig. \ref{fig:order}a)}
Our baselines for causal order discovery include topological sorting methods that use regression together with regularized likelihood or AMC-based scores, \cam{} \citep{buhlmann:14:cam} and  \topic{} \citep{xu:25:topic}, or perform sorting based on the score of the data distribution using score matching, namely   \scoremethod{} \citep{rolland:22:score} and extensions  \dasmethod{} \citep{montagna:23:das} and \nogam{} \citep{montagna:23:nogam}, and finally the neural method \notears{} \citep{zheng:18:notears}. We use their publicly available implementations, among others included in the \texttt{causal-learn} resp. \texttt{do-discover} Python packages, using  default  parameter choices as suggested in the original implementations.   A suite of methods, such as the classical PC algorithm \citep{spirtes2001causation}, return only partially oriented structures (CPDAGs or PDAGs) that do not correspond to  a unique causal order. Given this and the fact that they are often  outperformed by the methods we consider here in terms of other structural metrics  \citep[e.g.,][]{xu:25:topic}, hence we omit them from   comparison.
In the evaluation, we use batch sizes of $|B|=4$ for $n_B=30$ batches, resulting in 120 runs per method.   All methods return a fully oriented causal DAG with the exception of \avici, which outputs edge probabilities from which we extract a causal order by thresholding.
 
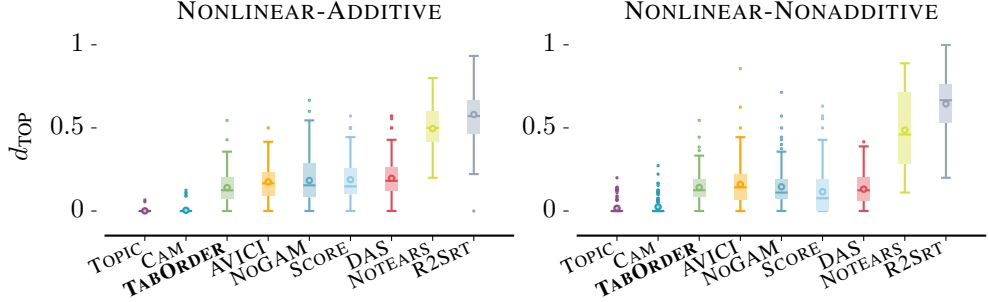
\begin{figure}
	\centering 
	
	\begin{minipage}[t]{0.48\textwidth} 
		\begin{tikzpicture}
			\begin{axis}[
				pretty boxplot,
				cycle list name=taborder-box,
				height=4cm,
				width=7cm,
				ylabel={$\divtop$},
				title={\textsc{Nonlinear-Additive}},
				ymax=1,
				pretty labelshift,
				xtick={1,2,3,4,5,6,7,8,9},
				xticklabels={\topic,\cam,\textbf{\ourmethod},\avici,\nogam,\scoremethod,\dasmethod,\notears,\rtwosort},
				major tick length=2pt,
				xticklabel style={rotate=25,anchor=east,font=\footnotesize},
				xtick style={draw=black,yshift=-3pt},
				]
				\foreach \col in {topic,cam,non-additive-unregularized-gp-small,avici-0-5,nogam,score,das,notears,rtwosort}{
					\addplot table[y=\col] {figs/texdata/order_box_additive_gp_topological_divergence.tsv};
				}
			\end{axis}
		\end{tikzpicture} 
	\end{minipage}
	\begin{minipage}[t]{0.48\textwidth} 
		\begin{tikzpicture}
			\begin{axis}[
				pretty boxplot,
				cycle list name=taborder-box,
				height=4cm,
				width=7cm,
				ylabel={ },
				title={\textsc{Nonlinear-Nonadditive}},
				ymax=1,
				pretty labelshift,
				xtick={1,2,3,4,5,6,7,8,9},
				xticklabels={\topic,\cam,\textbf{\ourmethod},\avici,\nogam,\scoremethod,\dasmethod,\notears,\rtwosort},
				major tick length=2pt,
				xticklabel style={rotate=25,anchor=east,font=\footnotesize},
				xtick style={draw=black,yshift=-3pt},
				]
				\foreach \col in {topic,cam,non-additive-unregularized-gp-small,avici-0-5,nogam,score,das,notears,rtwosort}{
					\addplot table[y=\col] {figs/texdata/order_box_non_additive_gp_topological_divergence.tsv};
				}
			\end{axis}
		\end{tikzpicture} 
	\end{minipage}
	
	\caption{\textbf{Causal Order Inference  (extends Fig. \ref{fig:order}a).}   Shown is the average topological divergence ($\divtop$, lower is better) between ground truth and estimated causal orders on synthetic datasets, generated from a nonlinear additive noise model with functions drawn from a Gaussian process, with both a variant where functional mechanisms are additive (left), respectively non-additive (right), in the causal parents.}
	\label{fig:full_causal_plot} 
\end{figure}
We assess the causal order learning using the topological divergence $\divtop$,  which measures the discrepancy between the predicted causal order $\hat \pi$ and the ground-truth order implied by the true DAG $G^\star$  by counting edge disagreements 
$\divtop(\hat \pi,G^\star) = \sum_{i=1}^{\nfeat} \sum_{j:\pi(i)\geq \pi(j)} \mathbb I[(i,j)\in G^\star]$,
with lower values indicating better alignment with the true causal ordering. 

We plot the topological divergence in Figure~\ref{fig:full_causal_plot}, including both additive mechanisms as well as non-additive mechanisms.

 \subsection{Causal Order Inference on Real-World  Data (Fig. \ref{fig:order}b)} 
 \label{app:additional-sachs-results}
  
 We evaluate causal order discovery on the Sachs protein-signaling benchmark \citep{sachs:05:sachs}, using the   consensus graph in Figure \ref{fig:sachs-consensuses}. We focus on the \texttt{cd3cd28} condition as the primary reference condition, since it corresponds to TCR/CD28 stimulation without additional inhibitors and is commonly used as the closest proxy to an observational setting. We compare \ourmethod{} again with the order-based and graph-based causal discovery baselines  \cam, \topic, \scoremethod, \dasmethod, \nogam, \notears, \rtwosort, and \avici.

 In addition to the reference condition reported in  Figure~\ref{fig:order}b, the Sachs dataset contains multiple interventional conditions corresponding to targeted perturbations of signaling pathways. We therefore also evaluate \ourmethod{} across all  Sachs conditions, besides \texttt{b2camp} which we exclude from the analysis as in previous works \citep{mooij2013joint}. The full results are shown 
  in the heatmap in Figure \ref{fig:order-heatmap-sachs-appendix}. Across most conditions, \ourmethod{} obtains topological divergences in the range $0.12$--$0.29$, indicating that the inferred orders remain broadly compatible with the reference signaling graph under moderate interventions. The main exception is the PKC inhibition condition \texttt{cd3cd28\_g0076}, where the divergence is substantially larger. This is biologically plausible, since \texttt{PKC} acts as a hub in the reduced graph and its perturbation can induce widespread changes across downstream MAPK and stress-signaling branches.
 
 Overall, these results suggest that \ourmethod{} recovers meaningful causal order information not only on synthetic data, but also on a real biological benchmark with known intervention structure. In Section \ref{app:additional-sachs-results-extended}, we further analyze how individual interventions affect the inferred orders through rank shifts and pairwise order flips. 
 \begin{figure}  

\begin{tikzpicture}
	\begin{axis}[
		width=12cm,
		height=8cm,
		title={Top. divergence ($\divtop,  \downarrow$)  across experimental conditions \citep{sachs:05:sachs}},
		xlabel={Interventional Condition},
		ylabel={Causal (Order) Discovery Method},
		xtick={0,1,2,3,4,5,6,7,8,9,10,11,12,13},
		xticklabels={\texttt{cd3cd28},\texttt{cd3cd28\_aktinhib},\texttt{cd3cd28\_g0076},\texttt{cd3cd28\_icam2},\texttt{cd3cd28\_ly},\texttt{cd3cd28\_psitect},\texttt{cd3cd28\_u0126},\texttt{cd3cd28icam2\_aktinhib},\texttt{cd3cd28icam2\_g0076},\texttt{cd3cd28icam2\_ly},\texttt{cd3cd28icam2\_psit},\texttt{cd3cd28icam2\_u0126},\texttt{pma},mean},
		xticklabel style={rotate=45, anchor=east, font=\scriptsize},
		ytick={0,1,2,3,4,5,6,7,8},
		yticklabels={\ourmethod,\notears,\topic,\rtwosort,\cam,\scoremethod,\nogam,\dasmethod,\avici},
		yticklabel style={font=\scriptsize},
		enlargelimits=false,
		axis on top,
		axis lines=left,
		axis line style={draw=none},
		colorbar,
		point meta min=0,
		point meta max=0.823529,
		colormap={blueseq}{rgb255(0cm)=(245,248,252);rgb255(2cm)=(200,220,232);rgb255(4cm)=(0,125,118);rgb255(6cm)=(14,77,104)},
		colorbar style={ylabel={$\divtop$}},
		]
		
		\addplot[
		matrix plot*,
		mesh/cols=14,
		point meta=explicit,
		] table[x=x, y=y, meta=z] {figs/texdata/sachs_method_context_topological_divergence_heatmap.tsv};
		
		\addplot[
		only marks,
		mark=none,
		nodes near coords,
		point meta=explicit symbolic,
		every node near coord/.append style={font=\tiny, text=black, anchor=center},
		] table[x=x, y=y, meta=label] {figs/texdata/sachs_method_context_topological_divergence_heatmap.tsv};
		
	\end{axis}
\end{tikzpicture}
\caption{\textbf{Causal Order Inference in Real-World Data (extends Fig. \ref{fig:order}b).} Shown is the average topological divergence ($\divtop$, lower is better) between ground truth and estimated causal orders on the real-world causal discovery benchmark by \cite{sachs:05:sachs}. For each available dataset $\cond$ (vertical) corresponding to one experimental condition, we show each method (horizontal) and how the causal order $\ordercond$ it discovers in  dataset $\cond$ differs from the consensus graph $\Gobs$. The consensus graph  $\Gobs$ represents a summary of causal relationships that likely best represents the near-observational condition \texttt{cd3cd28}.  Figure  \ref{fig:order}b reports the corresponding bar-plot results for this condition.
}
\label{fig:order-heatmap-sachs-appendix}
\end{figure}
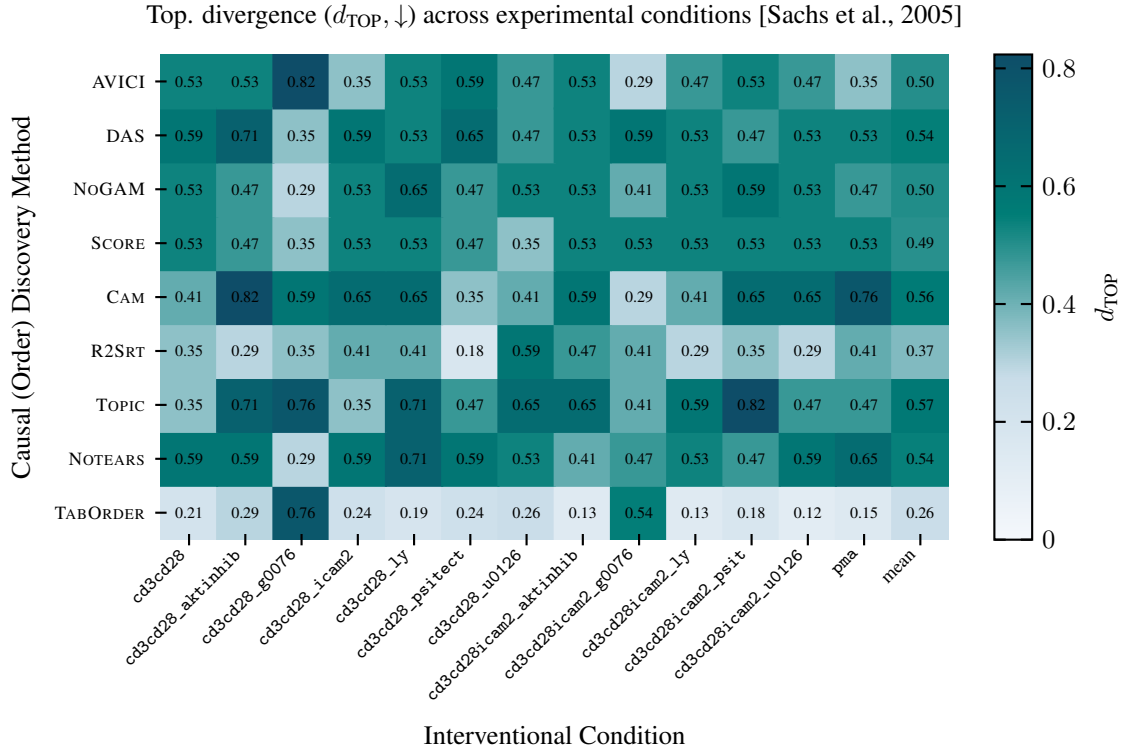
 

\subsection{Imputation Experiments  (Fig. \ref{fig:imputation_results})} 
\label{app:additional-imputation-results}
We use a variety of baselines for missing value imputation:
conventional methods such as column mean (\meanimp), \knnimp, and \softimpute \citep{hastie2015matrix} the iterative approaches \iceimp, \mice \citep{van2011mice}, and \missforest \citep{stekhoven2012missforest},
neural baselines \gain \citep{yoon2018gain} and \miwae \citep{mattei2019miwae}, and \hyperimpute \citep{jarrett2022hyperimpute}, which automatically selects the best imputation model per feature.
Lastly, we compare to \tabimpute \citep{feitelberg2025tabimpute}, a tabular foundation model specifically designed for missing value imputation.
Due to memory constraints we use maximum batch sizes of $512$ data points for \tabimpute and limit the maximum number of variables to $15$,
as larger tables caused out-of-memory errors on our $40$GB GPU hardware.
We evaluate \ourmethod,
which is fine-tuned on the training set, as well as the pretrained \ourmethod{} without fine-tuning (denoted \ourmethod-Pre).

\pgfplotsset{
  rankplot/.style={
    width=\linewidth,
    height=3.5cm,
    ymin=1,
    ymax=11,
    y dir=reverse,
    xtick={0.1,0.2,0.3,0.4,0.5},
    xticklabel style={/pgf/number format/fixed},
    xlabel={Missing rate},
    smooth,
    ticklabel style={font=\scriptsize},
    label style={font=\small},
  },
}
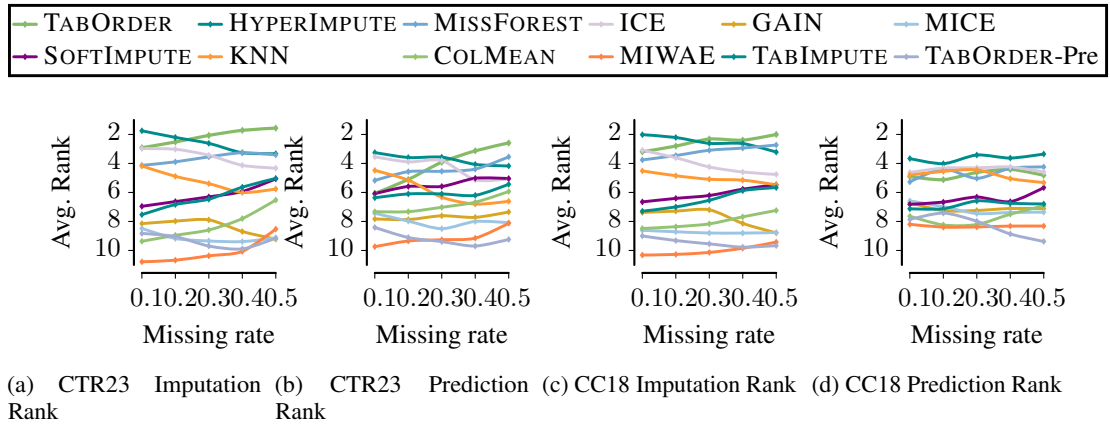
\begin{figure}
    \centering

    \newcommand{\rankdir}{expres/imputation/rank_by_method_complete}
    
    \begin{subfigure}[t]{0.24\linewidth}
        \centering
        \begin{tikzpicture}
            \begin{axis}[
                rankplot,
                pretty line,
                pretty labelshift,
                cycle list name=topic-line,
                ylabel={Avg.~Rank},
                xlabel={Missing rate},
                legend style={
                    at={(-1,1.8)},
                    anchor=north west,
                    font=\scriptsize,
                    legend columns=6,
                },
            ]
            \pgfplotsinvokeforeach{finetuned-foundation,hyperimpute,missforest,ice,gain,mice,softimpute,knn,mean,miwae,tabimpute, pretrained-foundation}{
                \addplot+
                table[
                    x=missing_rate,
                    y=imputation_rank,
                    col sep=comma,
                ]{\rankdir/CTR23_#1.csv};
            }
            \legend{\ourmethod, \hyperimpute, \missforest, \iceimp, \gain, \mice, \softimpute,\knnimp, \meanimp, \miwae, \tabimpute, \ourmethod-Pre}
            \end{axis}
        \end{tikzpicture}
        \caption{CTR23 Imputation Rank}
    \end{subfigure}
    \hfill
    \begin{subfigure}[t]{0.24\linewidth}
        \centering
        \begin{tikzpicture}
            \begin{axis}[
                rankplot,
                pretty line,
                pretty labelshift,
                cycle list name=topic-line,
                ylabel={Avg.~Rank},
                xlabel={Missing rate},
            ]

            \pgfplotsinvokeforeach{finetuned-foundation,hyperimpute,missforest,ice,gain,mice,softimpute,knn,mean,miwae,tabimpute, pretrained-foundation}{
                \addplot+
                table[
                    x=missing_rate,
                    y=prediction_rank,
                    cycle list name=color list,
                    col sep=comma,
                ]{\rankdir/CTR23_#1.csv};
            }

            \end{axis}
        \end{tikzpicture}
        \caption{CTR23 Prediction Rank}
    \end{subfigure}
    \hfill
    \begin{subfigure}[t]{0.24\linewidth}
        \centering
        \begin{tikzpicture}
            \begin{axis}[
                rankplot,
                pretty line,
                pretty labelshift,
                cycle list name=topic-line,
                ylabel={Avg.~Rank},
                xlabel={Missing rate},
                legend pos=north east,
            ]
            \pgfplotsinvokeforeach{finetuned-foundation,hyperimpute,missforest,ice,gain,mice,softimpute,knn,mean,miwae,tabimpute, pretrained-foundation}{%
                \addplot+
                table[
                    x=missing_rate,
                    y=imputation_rank,
                    cycle list name=color list,
                    col sep=comma,
                ]{\rankdir/CC18_#1.csv};
            }
            \end{axis}
        \end{tikzpicture}
        \caption{CC18 Imputation Rank}
    \end{subfigure}
    \hfill
    \begin{subfigure}[t]{0.24\linewidth}
        \centering
        \begin{tikzpicture}
            \begin{axis}[
                rankplot,
                pretty line,
                pretty labelshift,
                cycle list name=topic-line,
                ylabel={Avg.~Rank},
                xlabel={Missing rate},
            ]

            \pgfplotsinvokeforeach{finetuned-foundation,hyperimpute,missforest,ice,gain,mice,softimpute,knn,mean,miwae,tabimpute, pretrained-foundation}{%
                \addplot+
                table[
                    x=missing_rate,
                    y=prediction_rank,
                    cycle list name=color list,
                    col sep=comma,
                ]{\rankdir/CC18_#1.csv};
            }

            \end{axis}
        \end{tikzpicture}
        \caption{CC18 Prediction Rank}
    \end{subfigure}
    \caption{\textbf{Predictive Performance  (extends Fig. \ref{fig:imputation_results}).} Shown is the average rank of imputation (a, c) and downstream prediction (b, d) performance on the CTR23  and CC18  datasets of all tested methods. }
    \label{fig:full_imputation_plot}
\end{figure}
We plot the average rank of imputation and downstream prediction performance across all datasets and missingness rates in Figure~\ref{fig:full_imputation_plot}.
We observe that \ourmethod{} is the most accurate imputation method for $\geq 40\%$ missingness, outperforming all baselines including \tabimpute.
\ourmethod-Pre does not perform competitively, indicating that fine-tuning or alternatively training on different synthetic data is necessary to achieve strong imputation performance.

\subsection{Intervention-Robust Prediction (Fig. \ref{fig:three-variable-intervention-mse})}
For this experiment, we generate data from a three-variable chain $X \to Y \to Z$.
$X$ is sampled from a standard normal distribution.
$Y$ is generated from an additive noise model $Y = f(X) + N_Y$ where $f$ is a GP sample with RBF kernel or spline function and $N_Y\sim\mathcal{N}(0,0.1^2)$.
$Z$ is generated similarly as $Z = g(Y) + N_Z$ where $g$ is an independent GP/spline and $N_Z\sim\mathcal{N}(0,0.1^2)$.
We generate training datasets with $5000$ samples and test datasets with $2500$ samples. 
We use the default configuration of \tabpfn obtained via huggingface as well as the default XGBoost regressor with 100 estimators. 
We provide the full intervention MSE plots in Figure~\ref{fig:three-variable-intervention-mse-appendix}. 

\begin{figure*}[t]
  \centering
  \pgfplotstableread[col sep=tab]{figs/texdata/three_variable_intervention_mse_gp_hard.tsv}\datatablegphard
  \pgfplotstableread[col sep=tab]{figs/texdata/three_variable_intervention_mse_gp_mech_shift.tsv}\datatablegpmech
  \pgfplotstableread[col sep=tab]{figs/texdata/three_variable_intervention_mse_spline_hard.tsv}\datatablesplinehard
  \pgfplotstableread[col sep=tab]{figs/texdata/three_variable_intervention_mse_spline_mech_shift.tsv}\datatablesplinemech

  \begin{subfigure}{0.24\linewidth}
    \centering
    \begin{tikzpicture}
      \begin{axis}[
        pretty boxplot,
        width=\linewidth,
        height=3.1cm,
        title={GP, hard},
        ylabel={MSE},
        ymax=1,
        xtick={1,2,3,4,5,6},
        xticklabels={,i.i.d.,,,,Intervened,},
        legend columns=3,
        legend style={font=\scriptsize, at={(2.5,1.7)}, anchor=south},
      ]
      \addplot[{pr-color1a, fill=pr-color1a!40}] table[y=cfm_noniv] {\datatablegphard};
        \addlegendentry{\ourmethod}
        \addplot[{pr-color1b, fill=pr-color1b!40}] table[y=xgboost_noniv] {\datatablegphard};
        \addlegendentry{\xgboost}
        \addplot[{pr-color1c, fill=pr-color1c!40}] table[y=tabpfn_noniv] {\datatablegphard};
        \addlegendentry{\tabpfn}
        \addplot[{pr-color1a, fill=pr-color1a!40}] table[y=cfm_iv] {\datatablegphard};
        \addplot[pr-color1b, fill=pr-color1b!40] table[y=xgboost_iv] {\datatablegphard};
        \addplot[pr-color1c, fill=pr-color1c!40] table[y=tabpfn_iv] {\datatablegphard};

      \end{axis}
    \end{tikzpicture}
  \end{subfigure}
  \hfill
  \begin{subfigure}{0.24\linewidth}
    \centering
    \begin{tikzpicture}
      \begin{axis}[
        pretty boxplot,
        width=\linewidth,
        height=3.1cm,
        title={GP, mech. shift},
        ymax=1,
        xtick={1,2,3,4,5,6},
        xticklabels={,i.i.d.,,,,Intervened,},
      ]
       \addplot[{pr-color1a, fill=pr-color1a!40}] table[y=cfm_noniv] {\datatablegpmech};
        \addplot[{pr-color1b, fill=pr-color1b!40}] table[y=xgboost_noniv] {\datatablegpmech};
        \addplot[{pr-color1c, fill=pr-color1c!40}] table[y=tabpfn_noniv] {\datatablegpmech};
        \addplot[{pr-color1a, fill=pr-color1a!40}] table[y=cfm_iv] {\datatablegpmech};
        \addplot[{pr-color1b, fill=pr-color1b!40}] table[y=xgboost_iv] {\datatablegpmech};
        \addplot[{pr-color1c, fill=pr-color1c!40}] table[y=tabpfn_iv] {\datatablegpmech};
      \end{axis}
    \end{tikzpicture}
  \end{subfigure}
  \hfill
  \begin{subfigure}{0.24\linewidth}
    \centering
    \begin{tikzpicture}
      \begin{axis}[
        pretty boxplot,
        width=\linewidth,
        height=3.1cm,
        title={Spline, hard},
        ymax=1,
        xtick={1,2,3,4,5,6},
        xticklabels={,i.i.d.,,,,Intervened,},
      ]
        \addplot[{pr-color1a, fill=pr-color1a!40}] table[y=cfm_noniv] {\datatablesplinehard};
        \addplot[{pr-color1b, fill=pr-color1b!40}] table[y=xgboost_noniv] {\datatablesplinehard};
        \addplot[{pr-color1c, fill=pr-color1c!40}] table[y=tabpfn_noniv] {\datatablesplinehard};
    
        \addplot[{pr-color1a, fill=pr-color1a!40}] table[y=cfm_iv] {\datatablesplinehard};
        \addplot[{pr-color1b, fill=pr-color1b!40}] table[y=xgboost_iv] {\datatablesplinehard};
        \addplot[{pr-color1c, fill=pr-color1c!40}] table[y=tabpfn_iv] {\datatablesplinehard};
      \end{axis}
    \end{tikzpicture}
  \end{subfigure}
  \hfill
  \begin{subfigure}{0.24\linewidth}
    \centering
    \begin{tikzpicture}
      \begin{axis}[
        pretty boxplot,
        width=\linewidth,
        height=3.1cm,
        title={Spline, mech. shift},
        ymax=1,
        xtick={1,2,3,4,5,6},
        xticklabels={,i.i.d.,,,,Intervened,},
      ]
      \addplot[{pr-color1a, fill=pr-color1a!40}] table[y=cfm_noniv] {\datatablesplinemech};
        \addplot[{pr-color1b, fill=pr-color1b!40}] table[y=xgboost_noniv] {\datatablesplinemech};
        \addplot[{pr-color1c, fill=pr-color1c!40}] table[y=tabpfn_noniv] {\datatablesplinemech};
        \addplot[{pr-color1a, fill=pr-color1a!40}] table[y=cfm_iv] {\datatablesplinemech};
        \addplot[{pr-color1b, fill=pr-color1b!40}] table[y=xgboost_iv] {\datatablesplinemech};
        \addplot[{pr-color1c, fill=pr-color1c!40}] table[y=tabpfn_iv] {\datatablesplinemech};

      \end{axis}
    \end{tikzpicture}
  \end{subfigure}

  \caption{\textbf{Intervention-Robust Prediction (extends Fig.\ref{fig:three-variable-intervention-mse}).}  Test MSE for predicting $Y$ in the three-variable chain, split by mechanism family (GP vs. spline) and intervention type.}
  \label{fig:three-variable-intervention-mse-appendix}
\end{figure*}
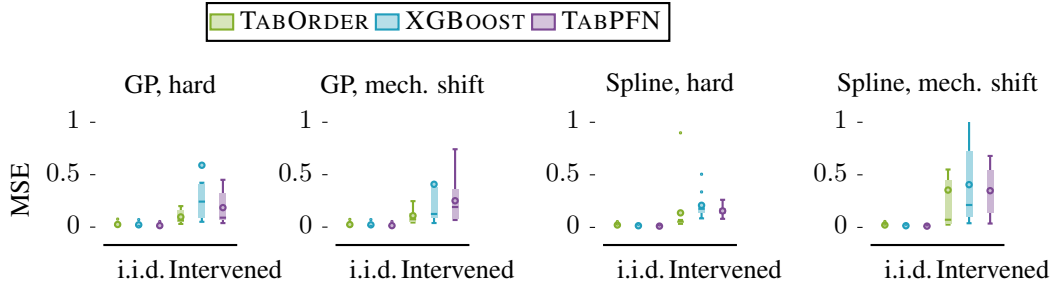

\section{Additional Experiments}

\subsection{Assumptions: Additive Noise Assumption}

We additionally investigate the robustness of \ourmethod{} to violations of the additive noise assumption used in the theoretical analysis. Besides standard additive Gaussian noise, we evaluate heteroskedastic noise, where the noise variance depends on the parent variables, and multiplicative noise, where noise scales the signal directly. 

Table~\ref{tab:noise_comparison} reports the resulting topological divergence $\divtop$. For illustration, we include both models trained on additive (CAM) functional models as well as non-additive ones. While the additive GP variant performs best under additive and heteroskedastic noise, the non-additive GP variant achieves the strongest performance under multiplicative noise and overall remains more stable across noise forms. These results suggest that \ourmethod{} is reasonably robust to moderate deviations from the classical additive noise setting.

\begin{table}[h]
	\caption{\textbf{Non-Additive Noise.} Topological divergence ($\divtop$) under different noise forms in the SCM.}
	\label{tab:noise_comparison} \vspace*{1em}
	\centering
	\begin{tabular}{lccc}
		\toprule
		Method & Additive Noise & Heteroskedastic Noise & Multiplicative Noise \\
		\midrule
		\ourmethod (additive GP)      & 0.202 & 0.190 & 0.232 \\
		\ourmethod (non-additive GP)  & 0.179 & 0.180 & 0.172 \\
		\bottomrule
	\end{tabular}
\end{table}

\subsection{Causal Order Inference: Scalability}
 \label{app:scalability} 

We also demonstrate the scalability of \ourmethod depending on (1) the number of features $\nfeat$  and (2) sequence length of each batch $\nsamples$, otherwise using the standard data generating settings in Section \ref{app:synthetic-data}.

For each configuration, we generate $25$ batches of size $4$ (a total of $100$ graphs) and report averaged metrics.
We vary either the number of variables $\nfeat \in \{10,20,50,100,200\}$ at fixed sequence length $\nsamples=1024$, or the sequence length $\nsamples \in \{128,256,512,1024,2048,4096\}$ at fixed $\nfeat=10$. All experiments were run on a single \texttt{NVIDIA A100} GPU (40GB VRAM) using \texttt{PyTorch} without distributed computation.

Figure~\ref{fig:additive-gp-scalability} compares \ourmethod{} against causal discovery baselines as the number of variables increases. We report topological divergence, evaluation runtime, and peak GPU memory usage. Some classical methods  could not be evaluated beyond moderate dimensionalities due to computational limitations and are therefore omitted for larger $\nfeat$, in particular \cam{} for  $\nfeat=50$ (Figure~\ref{fig:additive-gp-scalability} a).

\input{figs/appendix_scalability}

Table~\ref{tab:additive-gp-scalability-nvars} and Table~\ref{tab:additive-gp-scalability-seqlen} provide detailed scalability statistics for \ourmethod{}, including topological divergence, evaluation wall-clock time, and peak GPU memory usage.

\begin{table}[h] 
	\caption{\textbf{Scalability}. Shown are scalability metrics for \ourmethod on causal order discovery. This experiment uses  the same data generation and evaluation as in Figure \ref{fig:order} and varies the number of features respectively  samples. }
\begin{subtable}{\textwidth}
		\caption{\textbf{Scalability with Feature Size}. Shown are scalability metrics for \ourmethod when varying the number of features $\nfeat$ at fixed sample size $\nsamples=1024$. Values are mean $\pm$ 95\% CI.}
	\label{tab:additive-gp-scalability-nvars} \vspace*{1em}
	\centering
	\begin{tabular}{rrrr}
		\toprule
		$\nfeat$ & Top. div. ($\divtop$) & Eval wall (s) & Peak mem (GB) \\
		\midrule
		10  & $0.204096 \pm 0.0179684$ & $0.0458761 \pm 0.00885146$ & $0.151067 \pm 0.0000251987$ \\
		20  & $0.236118 \pm 0.0147771$ & $0.0772007 \pm 0.0137936$  & $0.293119 \pm 0.0000504157$ \\
		50  & $0.264468 \pm 0.0105586$ & $0.185035 \pm 0.0112221$   & $0.981129 \pm 0.000124748$ \\
		100 & $0.322675 \pm 0.00883721$& $0.321426 \pm 0.00937954$  & $2.92858 \pm 0.000259869$ \\
		200 & $0.421040 \pm 0.00577534$& $0.663079 \pm 0.00304706$  & $8.55309 \pm 0.000527340$ \\
		\bottomrule
	\end{tabular}
\end{subtable}
\begin{subtable}{\textwidth}
	\caption{\textbf{Scalability with Sample Size}. Shown are scalability metrics for \ourmethod when varying the sample size   $\nsamples$ at fixed number of features $\nfeat=10$. Values are mean $\pm$ 95\% CI.}
	\label{tab:additive-gp-scalability-seqlen} \vspace*{1em}
	\centering
	\begin{tabular}{rrrr}
		\toprule
		$\nsamples$ & Top. div. ($\divtop$) & Eval wall (s) & Peak mem (GB) \\
		\midrule
		128  & $0.254131 \pm 0.0207305$ & $0.0205390 \pm 0.00128916$ & $0.0397114 \pm 0.00000329479$ \\
		256  & $0.229872 \pm 0.0218224$ & $0.0196907 \pm 0.000812033$& $0.0560856 \pm 0.00000642391$ \\
		512  & $0.199187 \pm 0.0206693$ & $0.0268431 \pm 0.00181145$ & $0.0891907 \pm 0.0000126822$ \\
		1024 & $0.204096 \pm 0.0179684$ & $0.0458761 \pm 0.00885146$ & $0.151067 \pm 0.0000251987$ \\
		2048 & $0.197760 \pm 0.0187264$ & $0.149261 \pm 0.00514696$  & $0.278104 \pm 0.0000855268$ \\
		4096 & $0.172316 \pm 0.0184299$ & $0.209578 \pm 0.0120166$   & $0.532027 \pm 0.000107620$ \\
		\bottomrule
	\end{tabular}
\end{subtable}
\end{table}

Computationally, runtime and memory usage scale predictably with problem size. Increasing the number of variables $\nfeat$ has a substantially stronger effect on runtime and memory consumption than increasing the sequence length $\nsamples$. Despite being trained primarily on graphs with $5$--$10$ variables, \ourmethod{} generalizes to substantially larger settings up to $\nfeat=200$ variables while maintaining moderate topological divergence.

Order quality degrades gradually as $\nfeat$ increases, reflecting the increasing difficulty of predicting globally consistent orders in high-dimensional settings. In contrast, the topological divergence remains comparatively stable across sequence lengths and slightly improves for larger $\nsamples$, suggesting that the model is not strongly data-limited in this regime and can extract sufficient causal signal from relatively short sequences.

Inference of the causal order itself remains near-instant once the forward pass is completed.

\subsection{Real-World Data: Analysis of Biological Intervention Effects}\label{app:additional-sachs-results-extended}

We additionally analyze how the causal orders inferred by \ourmethod{} change across the interventional conditions of the Sachs signaling benchmark \citep{sachs:05:sachs}. Since the dataset contains multiple experimentally perturbed signaling conditions with known intervention targets, it provides a useful setting to study whether learned orderings reflect biologically meaningful system reorganization under intervention.
 
\paragraph{Setup.} For each experimental condition $\cond$, we infer a topological ordering $\hatorder_\cond$ over the 11 signaling variables and compare it to the reference condition
$
c_0 = \texttt{cd3cd28},$
which is commonly treated as the closest approximation to an observational baseline. The discovered orders and their topological divergences relative to the consensus graph are summarized in Table~\ref{tab:sachs-orders} and Table~\ref{tab:sachs-orders-topdiv}. The evaluated conditions and intervention targets are summarized in Table~\ref{tab:sachs-consensus-interventions}, while the underlying consensus signaling graph is shown in Figure~\ref{fig:sachs-consensuses}.
 
\begin{figure} 
\begin{tikzpicture}
	\begin{axis}[
		width=10cm,
		height=8cm,
		title={Rank shifts vs. \texttt{cd3cd28}},
		xlabel={Node},
		ylabel={Condition},
		xtick={0,1,2,3,4,5,6,7,8,9,10},
		xticklabels={PKA,PKC,Plcg,PIP3,Erk,Jnk,Akt,PIP2,P38,Mek,Raf},
		xticklabel style={rotate=45, anchor=east, font=\scriptsize},
		ytick={0,1,2,3,4,5,6,7,8},
		yticklabels={\texttt{cd3cd28},\texttt{cd3cd28\_icam2},\texttt{cd3cd28\_u0126},\texttt{cd3cd28\_aktinhib},\texttt{cd3cd28icam2\_aktinhib},\texttt{cd3cd28\_g0076},\texttt{cd3cd28icam2\_g0076},\texttt{cd3cd28\_ly},\texttt{cd3cd28\_psitect}},
		yticklabel style={font=\scriptsize},
		enlargelimits=false,
		axis on top,
		axis lines=left,
		axis line style={draw=none},
		colorbar,
		point meta min=-9,
		point meta max=9, 
colormap={rankmap}{rgb255(0cm)=(255,166,0);rgb255(1.2cm)=(255,166,0);rgb255(2.4cm)=(255,220,80);rgb255(3cm)=(245,248,252);rgb255(3.6cm)=(0,125,118);rgb255(4.8cm)=(14,77,104);rgb255(6cm)=(14,77,104)},
		colorbar style={ylabel={Rank shift}},
			]
			
			\addplot[
			matrix plot*,
			mesh/cols=11,
			point meta=explicit,
			] table[x=x, y=y, meta=rank_shift] {figs/texdata/sachs_rank_shift_heatmap.tsv};
			
			\addplot[
			only marks,
			mark=none,
			nodes near coords,
			point meta=explicit symbolic,
			every node near coord/.append style={font=\tiny, text=black, anchor=center},
			] table[x=x, y=y, meta=label] {figs/texdata/sachs_rank_shift_heatmap.tsv};
			
		\end{axis}
	\end{tikzpicture}

	\begin{tikzpicture}
		\begin{axis}[
			width=10cm,
			height=8cm,
			title={Pairwise order flips vs. \texttt{cd3cd28}},
			xlabel={Node},
			ylabel={Condition},
			xtick={0,1,2,3,4,5,6,7,8,9,10},
			xticklabels={PKA,PKC,Plcg,PIP3,Erk,Jnk,Akt,PIP2,P38,Mek,Raf},
			xticklabel style={rotate=45, anchor=east, font=\scriptsize},
			ytick={0,1,2,3,4,5,6,7,8},
			yticklabels={\texttt{cd3cd28},\texttt{cd3cd28\_icam2},\texttt{cd3cd28\_u0126},\texttt{cd3cd28\_aktinhib},\texttt{cd3cd28icam2\_aktinhib},\texttt{cd3cd28\_g0076},\texttt{cd3cd28icam2\_g0076},\texttt{cd3cd28\_ly},\texttt{cd3cd28\_psitect}},
			yticklabel style={font=\scriptsize},
			enlargelimits=false,
			axis on top,
			axis lines=left,
			axis line style={draw=none},
			colorbar,
			point meta min=0,
			point meta max=1,colormap={blueseq}{rgb255(0cm)=(245,248,252);rgb255(2cm)=(200,220,232);rgb255(4cm)=(0,125,118);rgb255(6cm)=(14,77,104)},
			colorbar style={ylabel={Fraction of pairwise order flips}},
				]
				
				\addplot[
				matrix plot*,
				mesh/cols=11,
				point meta=explicit,
				] table[x=x, y=y, meta=flip_fraction] {figs/texdata/sachs_pairwise_flip_heatmap.tsv};
				
				\addplot[
				only marks,
				mark=none,
				nodes near coords,
				point meta=explicit symbolic,
				every node near coord/.append style={font=\tiny, text=black, anchor=center},
				] table[x=x, y=y, meta=label] {figs/texdata/sachs_pairwise_flip_heatmap.tsv};
				
			\end{axis}
		\end{tikzpicture}
\caption{\textbf{Changes in Inferred Orders under Intervention.}
	The heatmaps compare the topological orders inferred by \ourmethod{} across all Sachs experimental conditions relative to the reference condition \texttt{cd3cd28}. 
	Top: node-wise rank shifts, where positive (negative) values indicate downstream (upstream) movement in the inferred order compared to the baseline condition. 
	Bottom: fraction of pairwise order flips involving each node, measuring how often its relative ordering with other variables changes compared to the baseline. 
	Rows correspond to experimental conditions and columns to signaling variables. Strong localized changes are observed for targeted pathway perturbations such as MEK inhibition (\texttt{cd3cd28\_u0126}) and phosphoinositide perturbation (\texttt{cd3cd28\_psitect}), while PKC inhibition (\texttt{cd3cd28\_g0076}) induces broader global reorganization across the signaling network.} 
		\label{fig:order-heat-sachs}
	\end{figure}
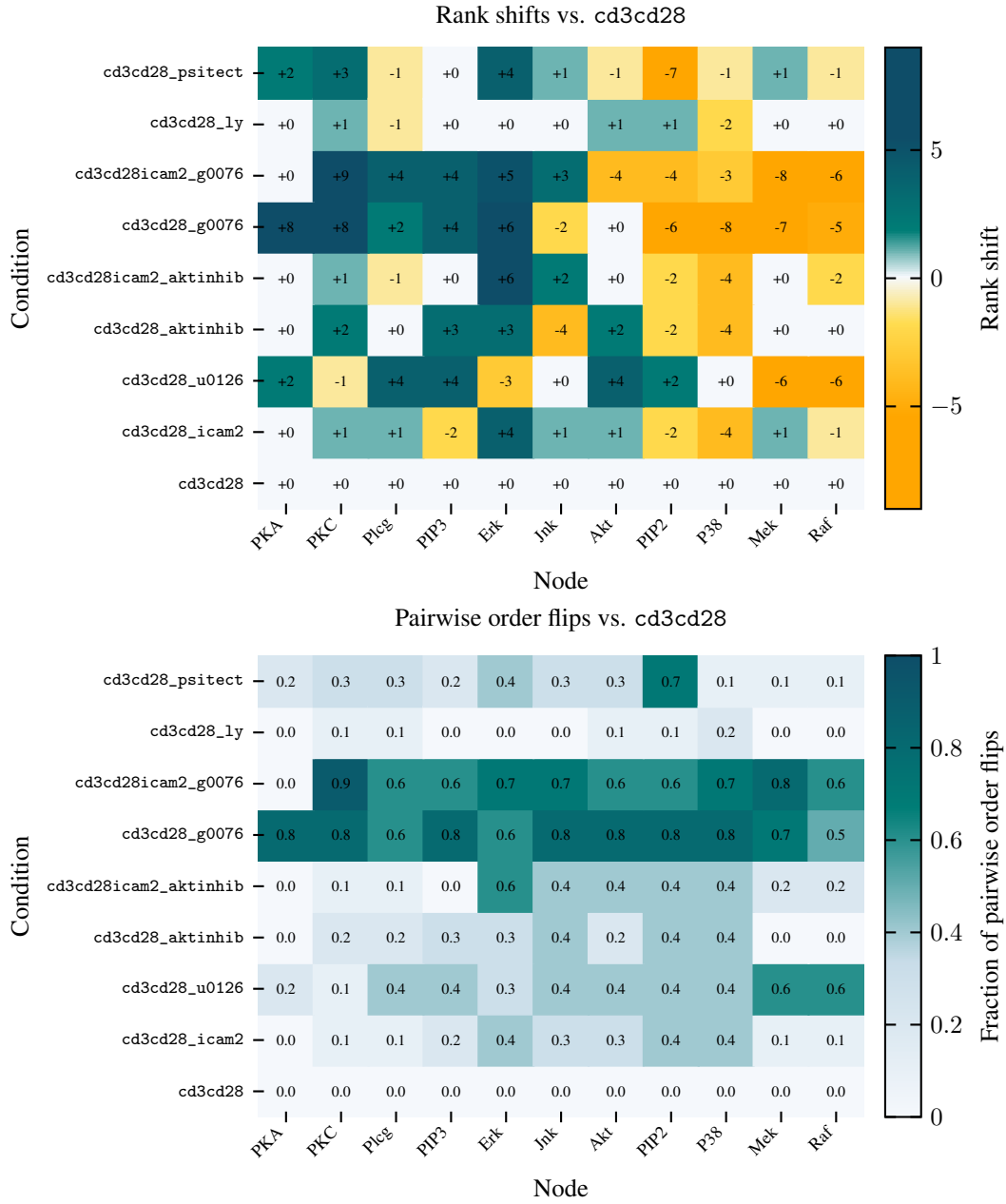
	 
\paragraph{Metrics.}
Beyond the global topological divergence $\divtop$, we analyze two local measures of intervention-dependent reorganization.

For each node $X$, we compute the  rank shift as 
\[
\Delta_{\texttt{c}}(X)
=
\mathrm{rank}(\hatorder_{\texttt{c}}, X)
-
\mathrm{rank}(\hatorder_{\texttt{cd3cd28}}, X),
\]
where positive (negative) values indicate downstream (upstream) movement in the inferred order.

We additionally compute pairwise order flips, i.e., changes in the relative ordering between node pairs across conditions. For each node, we aggregate the number of flipped pairwise relations relative to the baseline and normalize by $(p-1)$ to obtain a fraction in $[0,1]$.

Figure~\ref{fig:order-heat-sachs} summarizes rank shifts and pairwise flips across all experimental conditions. Figure~\ref{fig:sachs-all-intervention-graph-flips} visualizes pairwise flips on the consensus signaling graph, while Figure~\ref{fig:sachs-intervention-order-lineups} shows some representative examples of how the inferred orders themselves change relative to the baseline condition.

\paragraph{Analysis.} We observe clear intervention-dependent reorganizations in the inferred orders. MEK inhibition (\texttt{cd3cd28\_u0126}) induces strong changes involving \texttt{Mek}, \texttt{Erk}, and \texttt{Raf}, consistent with perturbation of the MAPK pathway and known feedback within the Raf--MEK--ERK cascade \citep{lake2016erk}. PKC inhibition (\texttt{cd3cd28\_g0076}) produces the strongest global reorganization across the signaling network, which is biologically plausible given the central role of \texttt{PKC} as a signaling hub \citep{isakov2002pkc}. In contrast, phosphoinositide perturbation (\texttt{cd3cd28\_psitect}) yields comparatively localized changes centered around \texttt{PIP2} and nearby signaling components.

Other interventions induce weaker or more diffuse effects. PI3K inhibition (\texttt{cd3cd28\_ly}) produces only moderate changes involving \texttt{PIP3} and \texttt{Akt}, while Akt inhibition (\texttt{cd3cd28\_aktinhib}) affects a broader set of MAPK-related nodes rather than remaining localized to the Akt pathway itself \citep{vivanco2002pi3k}.

Overall, the inferred orders reflect biologically meaningful intervention-dependent reorganizations of the signaling system, suggesting that \ourmethod{} captures pathway-specific changes beyond purely observational dependencies.

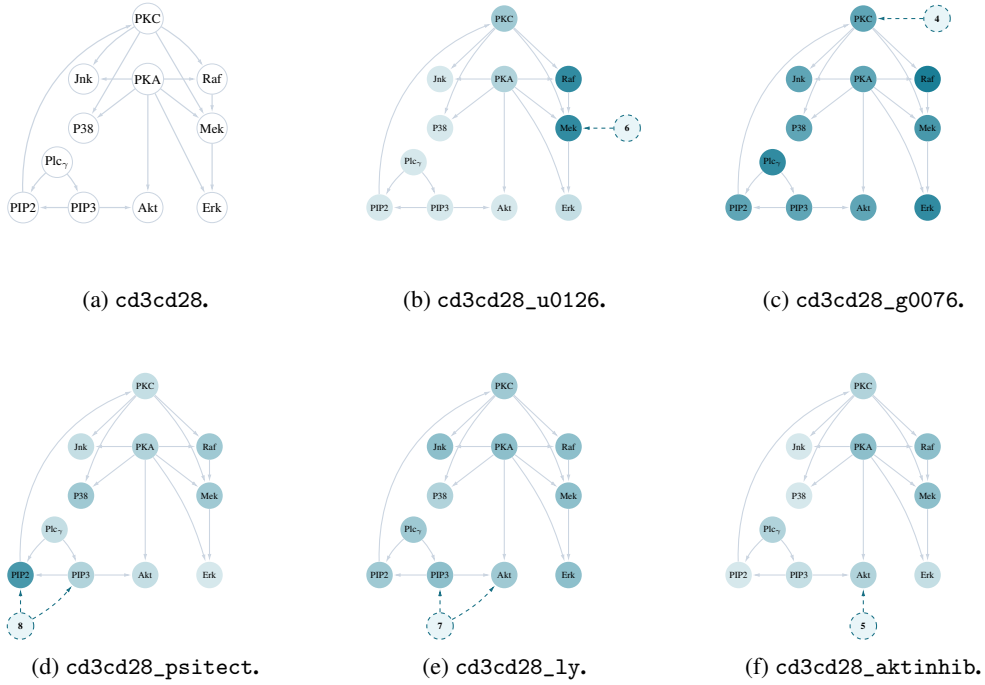
\begin{figure}[t]
	\centering
			\begin{subfigure}{.32\linewidth}
		\centering
		\scalebox{0.5}{
			\begin{tikzpicture}[
				prot/.style={draw=tcss-slate3, circle, minimum size=8mm, inner sep=1pt},
				iv/.style={
					draw=white,
					fill=white,white,
					circle,
					dashed,
					thick,
					minimum size=7mm,
					inner sep=1pt,
					font=\bfseries
				},
				>=Latex
				]
				
				\node[prot] (pkc) at (0,4) {PKC};
				
				\node[prot] (pka) at (0,2.4) {PKA};
				\node[prot] (raf) at (1.7,2.4) {Raf};
				
				\node[prot] (jnk) at (-1.7,2.4) {Jnk};
				\node[prot] (p38) at (-1.7,1.1) {P38};
				
				\node[prot] (plcg) at (-2.4,0.2) {Plc${}_{\gamma}$};
				
				\node[prot] (pip2) at (-3.3,-1.0) {PIP2};
				\node[prot] (pip3) at (-1.7,-1.0) {PIP3};
				\node[prot] (akt) at (0,-1.0) {Akt};
				\node[prot] (erk) at (1.7,-1.0) {Erk};
				
				\node[prot] (mek) at (1.7,1.1) {Mek};
				
				\path[-{Latex[length=2mm,width=1mm]}, line width=.9pt, tcss-slate3]
				(raf) edge[] (mek)
				(mek) edge[] (erk)
				(plcg) edge[bend right=10] (pip2)
				(plcg) edge[bend left=10] (pip3)
				(pip3) edge[] (pip2)
				(pip3) edge[] (akt)
				(pip2) edge[bend left=35] (pkc)
				(pkc) edge[bend left=12] (raf)
				(pkc) edge[] (mek)
				(pkc) edge[bend right=12] (jnk)
				(pkc) edge[] (p38)
				(pka) edge[] (raf)
				(pka) edge[] (mek)
				(pka) edge[] (erk)
				(pka) edge[] (akt)
				(pka) edge[] (jnk)
				(pka) edge[] (p38)
				;
				
				\node[iv] (i1) at (-1.45,5.45) {1};
				\node[iv] (i2) at (1.45,5.45) {2};
				\node[iv] (i3) at (-1.55,4.15) {3}; 
				\node[iv] (i4) at (1.95,4.00) {4};
				\node[iv] (i5) at (0,-2.35) {5};
				\node[iv] (i6) at (3.25,1.1) {6};
				\node[iv] (i7) at (-1.7,-2.35) {7};
				\node[iv] (i8) at (-3.30,-2.35) {8};
				
			\end{tikzpicture}
		} 
		\caption{\textbf{\texttt{cd3cd28}.}}
	\end{subfigure} 
	\hfill 
	\begin{subfigure}{0.32\linewidth}
		\centering
		\ResetSachsInterventions
		\ShowSachsInhibitorSix
		\ShowSachsFlipsUzero
		\scalebox{0.5}{\SachsFlipGraph{\texttt{cd3cd28\_u0126}}{0.265}}
		\caption{\textbf{\texttt{cd3cd28\_u0126}.} }
		\end{subfigure}	\hfill 
			\begin{subfigure}{0.32\linewidth}
			\centering
			\ResetSachsInterventions
			\ShowSachsInhibitorFour
			\ShowSachsFlipsGzero
			\scalebox{0.5}{\SachsFlipGraph{\texttt{cd3cd28\_g0076}}{0.765}}
			\caption{\textbf{\texttt{cd3cd28\_g0076}.} }
		\end{subfigure}
	 \vspace{0.7em}
			\begin{subfigure}{0.32\linewidth}
			\centering
			\ResetSachsInterventions
			\ShowSachsInhibitorEight
			\ShowSachsFlipsPsitect
			\scalebox{0.5}{\SachsFlipGraph{\texttt{cd3cd28\_psitect}}{0.235}}
			\caption{\textbf{\texttt{cd3cd28\_psitect}.} }
		\end{subfigure}
		\hfill
			\begin{subfigure}{0.32\linewidth}
			\centering
			\ResetSachsInterventions
			\ShowSachsInhibitorSeven
			\ShowSachsFlipsLy
			\scalebox{0.5}{\SachsFlipGraph{\texttt{cd3cd28\_ly}}{0.191}}
			\caption{\textbf{\texttt{cd3cd28\_ly}.} }
		\end{subfigure}  
			\hfill
		\begin{subfigure}{0.32\linewidth}
			\centering
			\ResetSachsInterventions
			\ShowSachsInhibitorFive
			\ShowSachsFlipsAktInhib
			\scalebox{0.5}{\SachsFlipGraph{\texttt{cd3cd28\_aktinhib}}{0.294}}
			\caption{\textbf{\texttt{cd3cd28\_aktinhib}.} }
		\end{subfigure} 
	\caption{\textbf{Changes in Inferred Orders under Intervention (extends Fig. \ref{fig:sachs-intervention-graph-flips-small}).} 
		Nodes are colored by pairwise flip fraction relative to \texttt{cd3cd28}; edges show the fixed consensus graph. Activator and inhibitor markers indicate the corresponding experimental perturbations. Panel titles report the global topological divergence $\divtop$ for \ourmethod where available.}
	\label{fig:sachs-all-intervention-graph-flips}
\end{figure}  
\begin{figure}[h!]
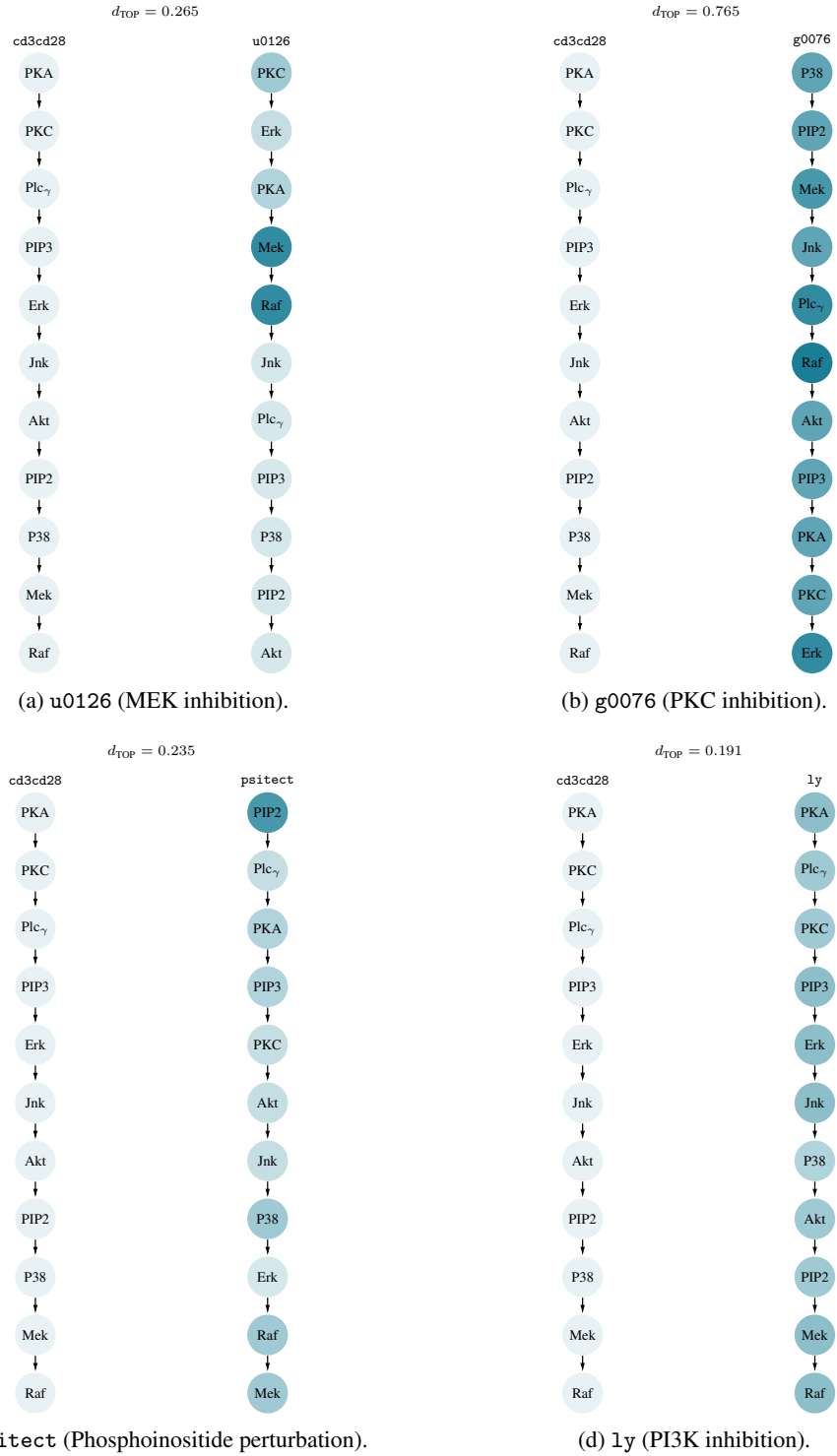

	\centering 
	\begin{subfigure}{0.48\linewidth}
		\centering
		\scalebox{0.78}{\OrderPanel{\texttt{u0126}}{0.265}{\BaselineOrder}{\OrderUzero}}
		\caption{\texttt{u0126} (MEK inhibition).}
	\end{subfigure}
	\hfill
	\begin{subfigure}{0.48\linewidth}
		\centering
		\scalebox{0.78}{\OrderPanel{\texttt{g0076}}{0.765}{\BaselineOrder}{\OrderGzero}}
		\caption{\texttt{g0076} (PKC inhibition).}
	\end{subfigure}
	
	\vspace{0.8em}
	
	\begin{subfigure}{0.48\linewidth}
		\centering
		\scalebox{0.78}{\OrderPanel{\texttt{psitect}}{0.235}{\BaselineOrder}{\OrderPsitect}}
		\caption{\texttt{psitect} (Phosphoinositide perturbation).}
	\end{subfigure}
	\hfill
	\begin{subfigure}{0.48\linewidth}
		\centering
		\scalebox{0.78}{\OrderPanel{\texttt{ly}}{0.191}{\BaselineOrder}{\OrderLy}}
		\caption{\texttt{ly} (PI3K inhibition).}
	\end{subfigure} 
	\caption{\textbf{Example Changes in Inferred Orders.}
		Each panel compares the baseline order inferred by \ourmethod on \texttt{cd3cd28} to the order inferred under the intervention. Nodes in the baseline order are shown in light blue;   nodes in the interventional condition are colored in dark blue depending on their pairwise flip fraction relative to the baseline.}
	\label{fig:sachs-intervention-order-lineups}
\end{figure}

\begin{table}[t]
	\centering 
	\caption{\textbf{Real-World Causal Order Inference.} Listed are the topological orders  discovered by \ourmethod in each experimental condition in the data by \cite{sachs:05:sachs}, as well as how they compare to the consensus network.}
	\begin{subtable}{\textwidth} 
		\caption{\textbf{Discovered Topological Orders per Condition.} Listed are the topological orders $\ordercond$  discovered in $\cond$ by \ourmethod in each experimental condition $\cond$.}
		\label{tab:sachs-orders}
		\vspace*{.5em}
		\centering
		\small
		\begin{tabular}{ll}
			\toprule
			Condition $\cond$ & Topological order $\ordercond$\\
			\midrule
			\texttt{cd3cd28}
			& PKA $\prec$ PKC $\prec$ Plcg $\prec$ PIP3 $\prec$ Erk $\prec$ Jnk $\prec$ Akt $\prec$ PIP2 $\prec$ P38 $\prec$ Mek $\prec$ Raf \\
			
			\texttt{cd3cd28\_icam2}
			& PKA $\prec$ PIP3 $\prec$ PKC $\prec$ Plcg $\prec$ P38 $\prec$ PIP2 $\prec$ Jnk $\prec$ Akt $\prec$ Erk $\prec$ Raf $\prec$ Mek \\
			
			\texttt{cd3cd28\_u0126}
			& PKC $\prec$ Erk $\prec$ PKA $\prec$ Mek $\prec$ Raf $\prec$ Jnk $\prec$ Plcg $\prec$ PIP3 $\prec$ P38 $\prec$ PIP2 $\prec$ Akt \\
			
			\texttt{cd3cd28\_aktinhib}
			& PKA $\prec$ Jnk $\prec$ Plcg $\prec$ PKC $\prec$ P38 $\prec$ PIP2 $\prec$ PIP3 $\prec$ Erk $\prec$ Akt $\prec$ Mek $\prec$ Raf \\
			
			\texttt{cd3cd28icam2\_aktinhib}
			& PKA $\prec$ Plcg $\prec$ PKC $\prec$ PIP3 $\prec$ P38 $\prec$ PIP2 $\prec$ Akt $\prec$ Jnk $\prec$ Raf $\prec$ Mek $\prec$ Erk \\
			
			\texttt{cd3cd28\_g0076}
			& P38 $\prec$ PIP2 $\prec$ Mek $\prec$ Jnk $\prec$ Plcg $\prec$ Raf $\prec$ Akt $\prec$ PIP3 $\prec$ PKA $\prec$ PKC $\prec$ Erk \\
			
			\texttt{cd3cd28icam2\_g0076}
			& PKA $\prec$ Mek $\prec$ Akt $\prec$ PIP2 $\prec$ Raf $\prec$ P38 $\prec$ Plcg $\prec$ PIP3 $\prec$ Jnk $\prec$ Erk $\prec$ PKC \\
			
			\texttt{cd3cd28\_ly}
			& PKA $\prec$ Plcg $\prec$ PKC $\prec$ PIP3 $\prec$ Erk $\prec$ Jnk $\prec$ P38 $\prec$ Akt $\prec$ PIP2 $\prec$ Mek $\prec$ Raf \\
			
			\texttt{cd3cd28\_psitect}
			& PIP2 $\prec$ Plcg $\prec$ PKA $\prec$ PIP3 $\prec$ PKC $\prec$ Akt $\prec$ Jnk $\prec$ P38 $\prec$ Erk $\prec$ Raf $\prec$ Mek \\\bottomrule
		\end{tabular}
	\end{subtable} 
	\begin{subtable}{\textwidth}  
		\caption{\textbf{Discovered Topological Orders per Condition vs. Consensus Graph. } Edge Viol. and Top. Div. are edge violation counts, respectiveky the  topological divergence $\divtop$,   for the order $\ordercond$ discovered  \ourmethod, computed with respect to the consensus graph $\Gobs$.}
		\label{tab:sachs-orders-topdiv}
		\vspace*{.5em}
		\centering
		\small
		\begin{tabular}{lll}
			\toprule
			Condition $\cond$
			& Edge viol. 
			& Div. Top. \\
			\midrule
			\texttt{cd3cd28}
			& 3
			& 0.206 \\
			
			\texttt{cd3cd28\_icam2}
			& 3
			& 0.235 \\
			
			\texttt{cd3cd28\_u0126}
			& 4
			& 0.265 \\
			
			\texttt{cd3cd28\_aktinhib}
			& 5
			& 0.294 \\
			
			\texttt{cd3cd28icam2\_aktinhib}
			& 1
			& 0.132 \\
			
			\texttt{cd3cd28\_g0076}
			& 13
			& 0.765 \\
			
			\texttt{cd3cd28icam2\_g0076}
			& 8
			& 0.544 \\
			
			\texttt{cd3cd28\_ly}
			& 3
			& 0.191 \\
			
			\texttt{cd3cd28\_psitect}
			& 3
			& 0.235 \\\bottomrule
		\end{tabular}
	\end{subtable}
\end{table}

  \clearpage

\section{Discussion}
\subsection{Limitations}
Further limitations include that \ourmethod currently supports only real-valued features, and our theoretical guarantees apply primarily to settings such as additive noise models. 
Extending the approach to mixed-type data, and establishing corresponding theoretical properties, presents an important direction for future work. 
Our method also  focuses on learning causal orderings rather than full causal graphs. 
 Extending the framework to learn sparser or more fine-grained causal structures is an interesting future direction but may require additional inductive biases or assumptions.

In practice, we observe that \ourmethod pretrained solely on synthetic data does not yet match the performance of classical imputation methods, and fine-tuning on real-world datasets is required to achieve competitive results. 
Our current synthetic priors are best suited for generating mechanisms with known identifiability properties. 
Developing more realistic priors to improve pretraining quality is a particularly promising avenue for future research.

\subsection{Broader impacts}
\ourmethod advances the state of the art in tabular foundation models by enabling them to infer and integrate causal structure directly from data without structural supervision. 
This principled approach can improve robustness to distribution shifts and interventions, 
but it also introduces risks when variable orderings are misidentified or causal assumptions are incomplete. 
The structural inferences produced by \ourmethod are hypotheses rather than ground truth and 
therefore require further validation, especially in high-stakes applications. 
As with other causal discovery methods, we recommend that users carefully assess the validity of the 
inferred structure, incorporate domain expertise, 
and perform appropriate sensitivity analyses before deploying \ourmethod in critical settings.




\end{document}